\documentclass[11pt]{amsart}
\usepackage{amsmath, amssymb,mathabx,mathrsfs}
\usepackage{color}
\usepackage{tcolorbox}
\usepackage{tikz}
\usepackage{hyperref}

\usetikzlibrary{matrix, positioning, patterns}
\usepackage[toc,page]{appendix}
\usepackage{stackengine}
\usepackage{longtable}
\usepackage{comment}
\usepackage{float}
\usepackage{epsfig}
\usepackage{graphicx}
\usepackage[T1]{fontenc}
\usepackage[utf8]{inputenc}
\usepackage[english]{babel}
\usepackage{amsmath}
\usepackage{subcaption}
\usepackage{mathtools, nccmath}
\usepackage{amsthm}
\usepackage{multicol}
\usepackage{tikz,pgfplots}
\usepackage[normalem]{ulem}

\theoremstyle{plain}
\newtheorem{thm}{Theorem}[section]

\newtheorem{lem}[thm]{Lemma}
\newtheorem{prop}[thm]{Proposition}
\newtheorem{defn}[thm]{Definition}
\theoremstyle{remark}

\newtheorem{ex}[thm]{Example}

\theoremstyle{definition}
\theoremstyle{remark}
\newtheorem{remark}{Remark}[section]

\pgfplotsset{compat=1.18}

\newcommand{\eps}{\epsilon}

\renewcommand{\epsilon}{\varepsilon}

\def\N{{\mathbb N}}

\def\R{{\mathbb R}}

\begin{document}

\title[Stability of accuracy for the training of DNNs]{Stability of accuracy for the training of DNNs Via the Uniform Doubling Condition}

\author{Yitzchak Shmalo }
\thanks{Corresponding author: Yitzchak Shmalo, Penn State. Email -- yitzchak.shmalo@gmail.com.   The code in this work is available and will be provided upon request}
 
 \begin{abstract}
 We study the stability of accuracy during the training of deep neural networks (DNNs). In this context, the training of a DNN is performed via the minimization of a cross-entropy loss function, and the performance metric is accuracy (the proportion of objects that are classified correctly). While training results in a decrease of loss, the accuracy does not necessarily increase during the process and may sometimes even decrease. The goal of achieving stability of accuracy is to ensure that if accuracy is high at some initial time, it remains high throughout training.

A recent result by Berlyand, Jabin, and Safsten introduces a doubling condition on the training data, which ensures the stability of accuracy during training for DNNs using the absolute value activation function. For training data in $\mathbb{R}^n$, this doubling condition is formulated using slabs in $\mathbb{R}^n$ and depends on the choice of the slabs. The goal of this paper is twofold. First, to make the doubling condition uniform, that is, independent of the choice of slabs. This leads to sufficient conditions for stability in terms of training data only. In other words, for a training set $T$ that satisfies the uniform doubling condition, there exists a family of DNNs such that a DNN from this family with high accuracy on the training set at some training time $t_0$ will have high accuracy for all time $t>t_0$. Moreover, establishing uniformity is necessary for the numerical implementation of the doubling condition. We demonstrate how to numerically implement a simplified version of this uniform doubling condition on a dataset and apply it to achieve stability of accuracy using a few model examples.

The second goal is to extend the original stability results from the absolute value activation function to a broader class of piecewise linear activation functions with finitely many critical points, such as the popular Leaky ReLU.
 
 \end{abstract}
 
       \maketitle

       \tableofcontents
\newpage  
      
\section{Introduction}

Deep neural networks (DNNs) are used in the classification problem to determine what class a set of objects $S \subset \mathbb{R}^n$ belongs to. As will be described in greater detail shortly, we start with a set of objects in a training set $T \subset \R^n$, for which the classes are known, and train the DNN using a loss function (see \eqref{loss_function}) with the hope of increasing accuracy as loss decreases.   Here, accuracy refers to the proportion of elements in the training set that the DNN classifies correctly (see \eqref{2_class_prob}). DNNs have been shown to be very capable of solving many real-world classification problems, such as handwriting recognition \cite{LBD}, image classification \cite{krizhevsky2017imagenet}, speech recognition \cite{hinton2012deep} and natural language processing \cite{sutskever2014sequence} to name a few.

The goal of the paper is to provide a condition on the training set of a DNN which establishes the stability of accuracy of the DNN during training. Specifically, DNNs training on a training set satisfying the condition will have that their accuracy is bounded from below as loss decreases. Ideally, this means that the accuracy of the DNN should go up or remain constant as training progresses. This is not always the case during training, sometimes loss can decrease drastically while accuracy also decreases. Before we provide such an example, a quick introduction to DNNs is needed.

We start with a set of objects $S \subset \R^n$ which we wish to classify into $K$ classes. A set $T \subset S$ exists for which we have an exact classifier $\phi^*(s)$ such that 

\begin{equation}
    \phi^*: s \in T \mapsto \phi^*(s)=(p_1(s), \dots, p_K(s)),
\end{equation}
with $p_{i(s)}=1$ and  $p_{j}=0$ when $j \neq i(s)$, and $i(s)$ denoting the correct class of $s$.  This exact classifier is only known for $T$, and so we use $T$ to train a classifier $\phi$ thereby approximating the exact classifier $\phi^*$ via $\phi$. Thus, we construct a $\phi$ such that 
\begin{equation}
\phi: s \in T \mapsto \phi(s)=(p_1(s), \dots, p_K(s)), \quad p_1+ \dots + p_K=1.
\end{equation}

We can think of $p_i(s)$ as the probability of $s$ belonging to the class $i$, and while we know the correct classes of the objects $s \in T$, the goal is to extend $\phi^*$ from $T$ to all of $S$ via $\phi$. Thus, we wish to obtain a $\phi$ which maps $s$ to the same class as $\phi^*$ when $s \in T$, with the hope that $\phi$ will also map elements $s \in S$ to their correct class. This is done by forming a family of parameterized functions $\phi(s,\alpha)$ and finding the correct parameters $\alpha$ so that $p_{i(s)}=1$ and  $p_{j}=0$ when $j \neq i(s)$ for $s \in T$.

In this paper, only $\phi$ which are DNNs (with different activation functions) are studied. Thus, we consider DNNs which are a composition $\rho \circ X(\cdot,\alpha)$, with $\rho$ the softmax (given in \eqref{soft_max}) and $X(\cdot,\alpha)$ defined as follows. Take 
 \begin{equation}
X(\cdot,\alpha)=\lambda\circ M_L(\cdot,\alpha_L)\circ\cdots\circ\lambda\circ M_1(\cdot,\alpha_1)\end{equation} with:
 \begin{itemize}
    	\item $M_k(\cdot,\alpha_k)$  an affine function $\mathbb R^{N_{k-1}}\to\mathbb R^{N_k}$ depending on an $\mathbb{N}_{k} \times \mathbb{ N}_{k-1}$ parameter matrix $W_k$ and bias vectors $\beta_k$  (i.e. $M_k(x)=W_k(x)+\beta_k$).
		\item $\lambda: \R^m \mapsto \R^m$ is a nonlinear \emph{activation function}. For now, we  assume $\lambda$ is the absolute value activation function.  More precisely, for $\lambda$ the absolute value activation function we have that $\lambda:\R^m \to \R^m$ is defined as $\lambda(x_1, \dots, x_m)=(|x_1|, \dots, |x_m|)$. Later we will also consider other activation functions with finite critical points.    
		\item $\rho$: is the  \emph{softmax function} and is used to normalize the outputs of $X(\cdot,\alpha)$ to probabilities. The components of $\rho$ are defined as: \begin{equation}
		\label{soft_max}	
\rho_i(s,\alpha)=\frac{\exp(X_i(s,\alpha))}{\sum_{j=1}^{K}\exp(X_j(s,\alpha))}.
\end{equation}	
 \end{itemize}

Thus the output of the DNN $\phi$ is a vector and provides the probability that an object $s \in T$ belongs to some class $i$ while $X(s, \alpha)$ is the final output  before softmax. We train a classifier $\phi$ (in our case a DNN) to try and mimic the exact classifier. Thus, taking $\phi$ to be a DNN, we have that $\phi=\phi(s,\alpha)$ with $\alpha \in \R^\nu$ the parameters in the affine maps of $\phi$. Here $\nu \gg 1$ is a very large number and is the dimension of the parameter space of our DNN. For $s \in T$, we know what $i(s)$ is and so we use a loss function to train $\phi$. Specifically, in this paper, we will consider the cross-entropy loss function 
  	\begin{equation}
	loss(t)=\bar L(\alpha(t))=-\sum_{s\in T}\mu(s) \log\left(p_{i(s)}(s,\alpha)\right).
	\label{loss_function}
	\end{equation}
	
	 Here $\mu$ is a measure given on the data set $T$ which determines how important objects $s \in T$ are during training. The hope is that by minimizing the loss function (for example through gradient descent or stochastic gradient descent), we can obtain a $\phi$ which classifies the elements of $T$ accurately. The loss function is a smooth function while accuracy is discontinuous, and so training via gradient descent is done on the loss function while we wish to improve (increase) accuracy. However, as we will see, the accuracy of $\phi$ (the proportion of $s \in T$ that is mapped by $\phi$ to their correct class) and the value of the loss function is not completely correlated. Thus, there are situations where even as we decrease the loss via training, the accuracy of $\phi$ also decreases at the same time. The main goal of this paper is to provide conditions on the training set $T$ which ensure that this does not happen so that accuracy is bounded from below as the loss decreases. The results in this paper build on similar stability of accuracy results for classifiers given in \cite{LJS}.  

  The relationship between the loss function and the accuracy of a model is foundational in the domain of neural networks. The loss function is one way to quantify the deviation of a model's predictions from the true labels in the dataset, with the fundamental goal of training being to adjust model parameters in order to minimize this loss \cite{goodfellow2016deep}. Studies have shown relationships between minimizing loss and increasing accuracy on the training and test sets for cases with separable data  \cite{soudry2018implicit}. Other studies provide estimates for the solution accuracy of support vector machines (SVMs) based on the number of steps taken during training \cite{shalev2007pegasos}. Especially for separable data, this solution accuracy will be closely related to the accuracy of the SVM, thus providing a relationship between SMV training and the accuracy of the classifier.  Furthermore, issues like overfitting, where a model excels on the training data but underperforms on unseen data, further complicate this relationship \cite{zhang2021understanding}. Regularization techniques have been developed to mitigate such overfitting \cite{goodfellow2016deep}. Despite these advancements, a comprehensive mathematical study explicitly detailing the correlation between loss minimization and accuracy enhancement is notably absent in the literature, to the best of the author's knowledge. The current work endeavors to bridge this gap, seeking to establish a robust mathematical foundation elucidating the interplay between minimizing loss and maximizing accuracy on the training set. This remains a significant, albeit mostly underexplored, mathematical challenge warranting thorough investigation. For more on this mathematical question, see \cite{LJS}.

  Understanding the intricate relationship between training loss and training accuracy is not just a matter of academic intrigue; it is a pivotal theoretical question with significant practical implications. Empirically, minimizing loss and improving accuracy have shown to be strongly correlated, especially in scenarios involving extensive training datasets. This observed relationship demands a rigorous mathematical elucidation to solidify our understanding of why this correlation persists and under what conditions it is most pronounced. By delving into this theoretical framework, our work sheds light on specific structures inherent within the training data, particularly focusing on their "doubling property." This property, we posit, can be leveraged to optimize DNNs to achieve superior accuracy on both training and testing datasets,  see Example \ref{new_example2}, Example \ref{new_example3} and Subsection \ref{example1_new}. This "doubling property" also correlates with the loss of DNNs on the training set, a phenomenon that our theory helps explain, see Examples \ref{DC_vs_loss_example_1}, \ref{DC_vs_loss_example_2}, \ref{DC_vs_loss_example_3}, Theorem \ref{weakenedDC1} and Section \ref{are_truncations_important}.  Such an understanding does not merely serve theoretical purposes; it offers tangible guidance on how to better structure and utilize training data for improved model performance. Although the overarching theme of the interplay between loss functions and model accuracy has been touched upon in foundational machine learning literature, our focus on the "doubling property" and its influence on this relationship is a important contribution that bridges theoretical inquiry with practical application.

	 We now introduce a tool called classification confidence, which expresses how well an object $s$ is classified. The classification confidence is a function that maps elements in $T$ to $\mathbb{R}$, and provides information regarding how well $s \in T$ is classified by our DNN. Recall that we defined $X(s,\alpha)$ to be the final output of our DNN before the softmax function $\rho$. The softmax function preserves order in the sense that if $X_i(s,\alpha)>X_j(s,\alpha)$ then $\rho_i \circ X_i(s,\alpha)> \rho_j \circ X_j(s,\alpha)$.   Thus, we define the classification confidence via $X(s,\alpha)$ alone. 
  	
  	\begin{defn}The classification confidence is defined as follows:
  	
\begin{equation}\label{eq:defdeltaX}
\delta X(s,\alpha):=X_{i(s)}(s,\alpha)-\max_{j\neq i(s)} X_j(s,\alpha).
\end{equation}
  	\end{defn}
  	
  	In other words, 
  	\begin{itemize}
		\item $\delta X(s,\alpha(t))>0\Rightarrow s$ is well-classified by $\phi$.
		\item $\delta X(s,\alpha(t))< 0\Rightarrow s$ is misclassified by $\phi$.
			\end{itemize}
			
			Intuitively, $\delta X$ is a measure of how much the DNN prefers the correct class compared to all other classes. If $\delta X$ is positive, it means that the DNN is more confident in its classification of the object $s$ with respect to the correct class. Conversely, if $\delta X$ is negative, it means that the DNN is more confident in classifying $s$ with respect to another class, and thus, it is misclassifying $s$.

We can now define accuracy using the classification confidence,
		\begin{equation}\label{2_class_prob}
		 \text{acc}(t)=\mu\left(\{s\in T:\delta X(s,\alpha(t))>0\}\right).
		\end{equation}

  \begin{figure}[h!]
				\includegraphics[width=4.5in]{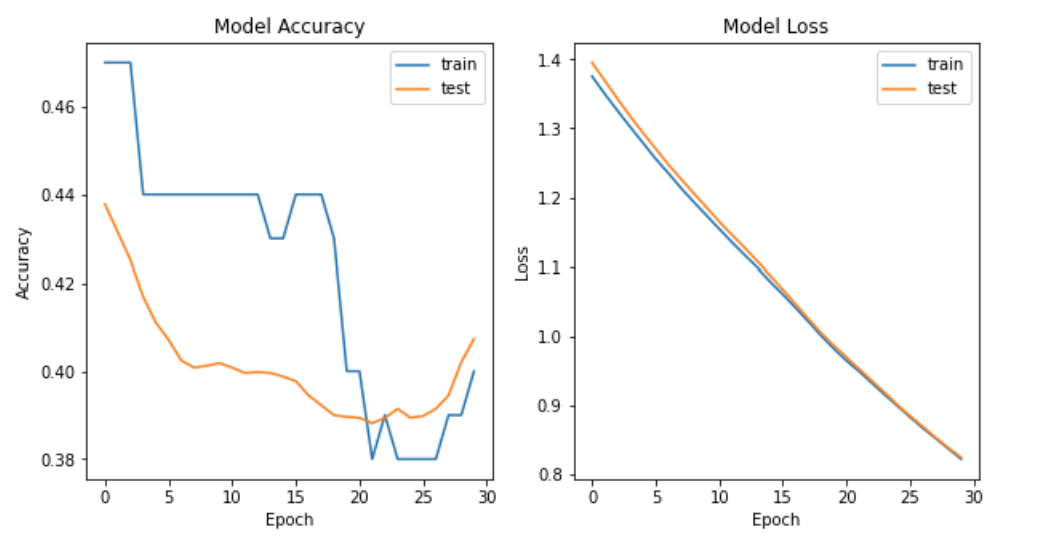}
				
				\caption{Despite the progress in training and reduction in loss, the accuracy has experienced a significant decline, falling from 48\% at epoch 0 to 38\% at epoch 25.}	\label{toy_example_acc_loss}
		\end{figure}

An example of when accuracy decreases even though loss decreases is given in Fig. \ref{toy_example_acc_loss}. Here a DNN had been trained to find a decision boundary between two data sets in $\mathbb{R}^2$, see Example \ref{example_1} for more details. Although training has progressed and loss has decreased, the training accuracy has dropped drastically (from 48\% at epoch $0$ to 38\% at epoch $25$). We would like to avoid the type of situations like in Fig. \ref{toy_example_acc_loss}, at least for the cases when the DNN initially has high accuracy. In Section \ref{motivation_example} we will provide more numerical examples of this behavior as well as a more detailed explanation for why it happens and how such behavior might be avoided.

    	    One can show that instability happens because of small clusters in the training set $T$ (see Def. \ref{def:no_small_islated_data_cluster_2} and Theorem \ref{weakenedDC1}). In Fig. \ref{toy_example1_small_cluster}, these small clusters in $T$ correspond with the small clusters in the histogram of $\delta X(T,\alpha_1)$ contained near $0$, with $\alpha_k$ the parameters of the DNN at epoch $k$. As training progresses objects $s$ corresponding to these small clusters might become misclassified (see Fig. \ref{toy_example1_after_training}) while the objects $s'$ which were badly misclassified, i.e. $\delta X(s',\alpha_1)<-2$, obtain larger $\delta X$ but still remain misclassified (see Fig. \ref{delta_X_toy_example1}). Because of this, loss decreases given that many objects have larger $\delta X$. However accuracy  decreases, given that more elements have negative $\delta X$. Thus, preventing small clusters in the data leads to stability during training.

         \begin{figure}[h!]
         \begin{subfigure}[b]{0.3\textwidth}
         \centering
\includegraphics[width=1.5\textwidth]{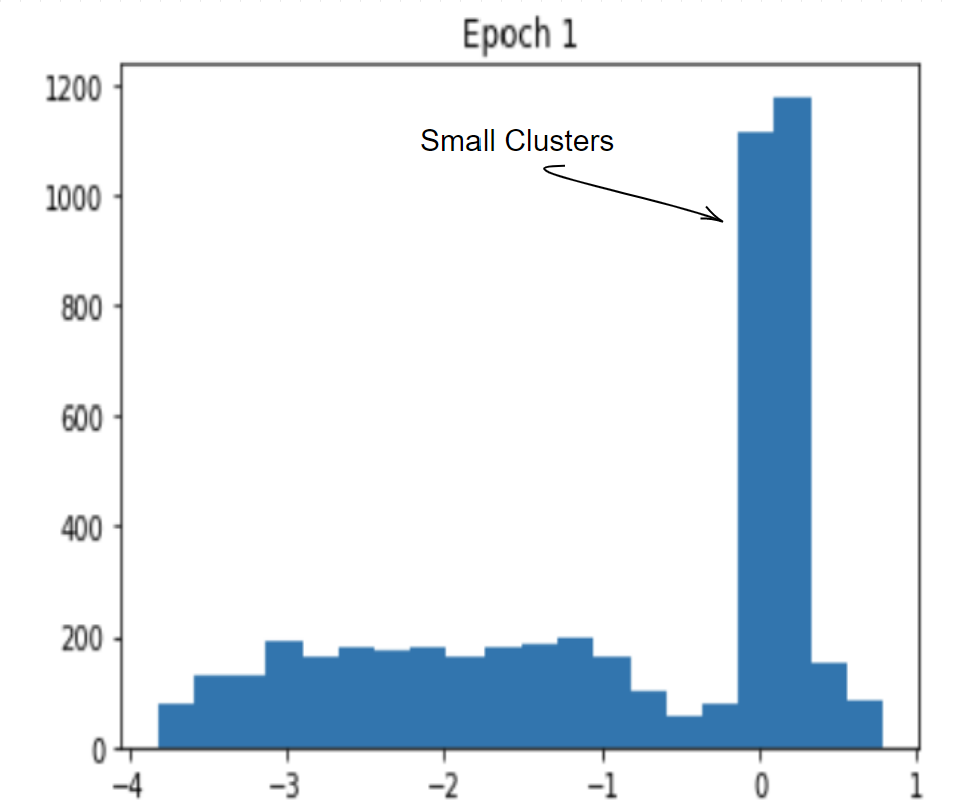}
         \caption{$\delta X(T,\alpha_1)$ at epoch 1. A cluster of the histogram is centered around 0. The x-axis represents $\delta X$ while the y-axis represents $s$. }    \label{toy_example1_small_cluster}
     \end{subfigure}
     \hspace{0.2\textwidth}
     \begin{subfigure}[b]{0.3\textwidth}
         \centering    \includegraphics[width=1.65\textwidth]{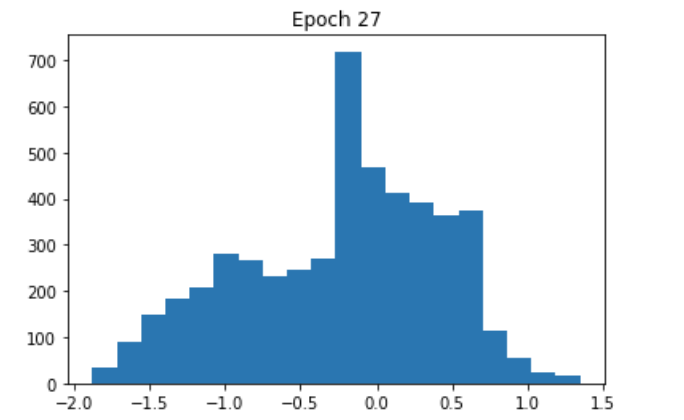}
         \caption{$\delta X(T,\alpha_{27})$ at epoch 27. Many data points $s$ which were correctly classified but for which $\delta X(s,\alpha_1) \approx 0$ are now misclassified. }
         \label{toy_example1_after_training}
     \end{subfigure}

     \caption{Histogram of $\delta X$ at epoch 1 and 27}
     \label{delta_X_toy_example1}
				
		\end{figure}

As mentioned, we would like to show that high accuracy at some initial time implies high accuracy for future time. That is, for $t_0$ some time during training we would like that 

\begin{equation}
    \text{high acc}(t_0) \rightarrow \text{high acc}(t), \forall t>t_0.
\end{equation}

 In this paper, this will be achieved with the following steps:

\begin{equation}
    \text{high acc}(t_0) \rightarrow \text{small loss}(t_0) \rightarrow \text{small loss}(t) \rightarrow \text{high acc}(t). 
\end{equation}

The step $\text{high acc}(t_0) \rightarrow \text{small loss}(t_0)$ is the most difficult part and is achieved by introducing the \textit{uniform doubling condition} on the training set. The step $\text{small loss}(t_0) \rightarrow \text{small loss}(t) $ is fundamental to training because training is essentially    the minimization of loss, and the step $\text{small loss}(t) \rightarrow \text{high acc}(t)$ is achieved using inequality \eqref{acc_loss} from \cite{LJS}. Throughout the paper, we assume that loss always decreases as training progresses. This however is not always the case. Nevertheless, we can reformulate the notion of stability of accuracy in the following way for the situations in which loss does not always decrease. \textbf{Stability of accuracy:} if accuracy is high at $t_0$ and for some $t>t_0$ we have that $\text{loss}(t)\leq \text{loss}(t_0)$ then accuracy at $t$ will be high.  

   To obtain stability of accuracy, it is essential to demonstrate two key points. First, a low loss should result in high accuracy, which can be derived from the following inequality:

\begin{equation}
\label{acc_loss}
1-\text{acc}(t)\leq \frac{\text{loss(t)}}{\log(2)},
\end{equation}
as presented in Lemma \ref{Upper_Bound_on_Loss}. However, this bound is only useful if the loss is sufficiently small. So we wish to obtain conditions on the training set $T$ which ensure that loss does not only decrease but also becomes sufficiently small during training. In general, it might not be possible to find such a condition on the training set. However, stability of accuracy, which is the main result of \cite{LJS} and this paper (under different conditions), means that high accuracy at the initial time guarantees high accuracy for future time. So we already assume that accuracy is high at $t_0$. Assuming that the uniform doubling condition holds on the training set $T$ and that the accuracy of the DNN is high at $t_0$, we will show that loss at $t_0$ is small, see \eqref{eq:cond_B_loss_estimateWDC}. Then because loss decreases during training, we can use \eqref{acc_loss} to obtain that accuracy stays high for the rest of the training (i.e. for $t>t_0$).       

So, it is important to understand why high accuracy does not always imply a low loss. Sometimes we can have high accuracy and high loss simultaneously. This occurs when some objects are significantly misclassified and there is a small group of elements near zero that are correctly classified, yet still have a high loss. Therefore, to achieve stability of accuracy, it is crucial to ensure that elements are not grossly misclassified and that no small clusters of $\delta X$ exist near zero. Refer to Def. \ref{eq:uniform_doubling_condtion_deltaX} and Prop \ref{lemmadoublingconditionondeltaX} for further information.

In previous work \cite{LJS}, a doubling condition was introduced on the training set $T$ which ensured no small clusters of objects in the training set. This doubling condition on $T$ also ensures no small cluster in $\delta X(T,\alpha)$ and therefore guarantees the stability of accuracy. That is, if a DNN trained on $T$  achieves sufficiently high accuracy at $t_0$ it will have that its accuracy will remain high for $t \geq t_0$, see Theorems \ref{thm:stabilityB} and \ref{Pro_DC}. While leading to novel theoretical results, this doubling condition has two drawbacks. First, for $T \subset \R^n$ the doubling condition  defined on slabs in $\R^n$ depends on the choice of the slabs. This means that the resulting stability of accuracy theorem depends on the choice of slabs used to verify the doubling condition and not only on the training set $T$. Second, this doubling condition cannot be implemented numerically as it requires checking it for infinitely many slabs with arbitrarily small widths. Finally, the constants which are given by the stability theorem in \cite{LJS} are also very large and not practically useful, see Theorem \ref{Pro_DC} and \eqref{constantC_1}.    

In this work, we present a new uniform doubling condition that resolves these two difficulties. First, this doubling condition does not depend on the choice of slabs and so depends only on the training set $T$. This leads to stability of accuracy in terms of the training set $T$ only. That is, assuming the training set $T$ satisfies the uniform doubling condition, there is a subset of parameters $A \subset \R^\nu$ such that for a DNN with parameters $\alpha \in A$ and certain architecture if it has high accuracy on the training set at some time $t_0$ then it will have high accuracy for all time $t>t_0$. Second, our uniform doubling condition can be verified numerically. In practice, for high dimensional data (i.e. $T \subset \R^n$ for $n$ very large), verification is subject to the curse of dimensionality and becomes harder to check as dimension $n$ grows. Even in small dimensions, this uniform doubling condition probably cannot be fully verified, however, we show numerically that a simplified version of the uniform doubling condition can be checked for a data set and is an indication of whether loss becomes small during training. That is, we construct simple training sets for a decision boundary problem and check the simplified uniform doubling condition (SUDC) on them.  We show that DNNs with similar accuracy trained on these training sets will have lower loss based on the SUDC. Using \eqref{acc_loss}, we conclude that the SUCD is a good indication of whether the DNN has stability of accuracy in these model examples. For more, see Examples \ref{DC_vs_loss_example_1}, \ref{DC_vs_loss_example_2} and \ref{DC_vs_loss_example_3}. In Subsection \ref{are_truncations_important}, we explain why it is reasonable to assume that SUDC provides important information for the stability of accuracy of a DNN. Finally, we also show a correlation between the SUDC and the accuracy and loss of the DNN on the test set, see Example \ref{new_example2}, Example \ref{new_example3} and Subsection \ref{example1_new}. That is, we show numerically that the training sets that better satisfy the SUDC have that:

\begin{enumerate}
    \item DNNs trained on them have higher stability of accuracy.
    \item DNNs trained on them have higher accuracy on the test set. 
\end{enumerate}

       \section{Motivation and some numerical examples} \label{motivation_example}
        In this section, we motivate our results using some numerical examples on a toy problem as well as a brief  explanation of overfitting and how the accuracy of the training set is related to the accuracy of the test set.
        
        \begin{ex}
\label{example_1}
       
       In this example, we are provided with a collection of two-dimensional points, and our goal is to determine the boundary separating the two regions. The dataset is generated by constructing a random polynomial function of a specific degree and then sampling points uniformly across a range of x-values. For each x-value, we calculate the polynomial function to obtain the corresponding y-value and then shift the points up or down to create points above and below the decision boundary. This results in two distinct regions of points, which we label as red or blue. The region above the polynomial function is labeled red, while the region below it is labeled blue. We also introduce noise to displace the y-values up or down, causing some red points to fall below the decision boundary and some blue points to rise above it. Refer to Fig. \ref{fig:toy_example_1} for an illustration.

\begin{figure}[h!]
\includegraphics[width=.5\textwidth]{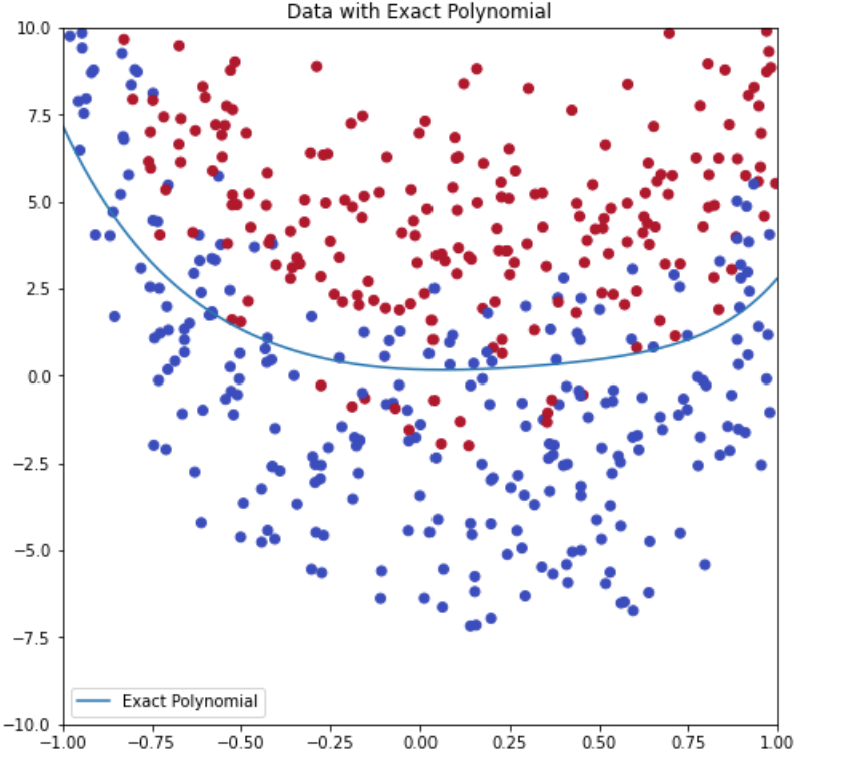}
\caption{Decision boundary plot for a two-class classification problem generated by a random polynomial function of degree 6, with added noise. The blue and red points represent the two classes, and the blue line shows the original polynomial.}
\label{fig:toy_example_1}
\end{figure}

The objective is to train a DNN to learn the decision boundary between these regions. Specifically, the DNN should learn a function that takes a two-dimensional point as an input and outputs a binary label indicating whether the point belongs to the region above or below the polynomial. Since the dataset is synthetically generated, we know the true decision boundary and can use it to assess our DNN's performance.

We create a neural network model for this classification task on the two-dimensional dataset. The model has two hidden layers: the first is a fully connected layer with five neurons and a ReLU activation function, while the second is a fully connected layer with two neurons and a softmax activation function. The model is trained using the cross-entropy loss function and the Adam optimizer. Although the main results in this work focus on the absolute value activation function or activation functions with only finitely many critical points, the ReLU activation function is used in this example for convenience. The relationship between $\delta X$ and the DNN's training accuracy does not depend on the activation function.

After training the DNN, we plot its accuracy, as shown in Fig. \ref{toy_example_acc.png}. In this example, the loss decreases throughout training, but for the first ten epochs, the accuracy on both the training and test set decreases before beginning to increase. The goal of this work is to explain why the DNN's training accuracy sometimes decreases as the loss decreases and to provide conditions that ensure that this does not happen when accuracy is high. First, we examine this phenomenon in the simple case by analyzing $\delta X$, as defined in \eqref{eq:defdeltaX}, for this DNN during training. 

\begin{figure}[h!]
\includegraphics[width=.5\textwidth]{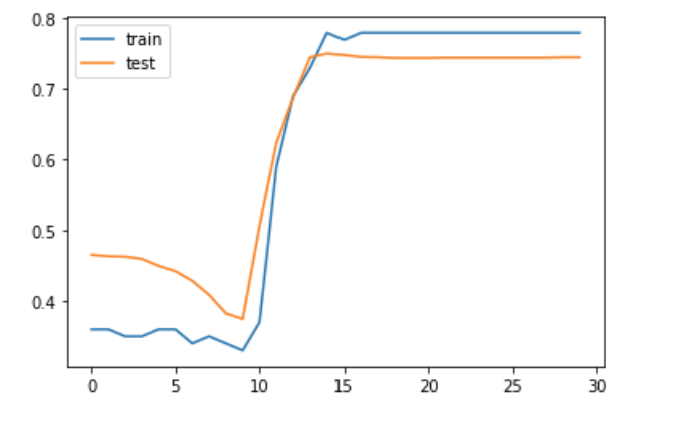}
\caption{The graph of the training and test accuracy of our DNN. The x-axis represents the epochs of the DNN, and the y-axis represents its accuracy.}
\label{toy_example_acc.png}
\end{figure}

In Fig. \ref{CC_toy_example1}, we display the classification confidence ($\delta X(T,\alpha)$) of the DNN at various stages of training for the training set $T$. During the initial ten epochs, the accuracy of the training set declines due to the existence of two small clusters of $\delta X(T,\alpha)$ near $0$. One of these clusters, which we call the correctly classified cluster, is located to the right of $0$, while the other, called the misclassified cluster, is to the left of $0$. As the training progresses, the clusters merge into a single large cluster on the left side of $0$. The loss decreases as the misclassified cluster moves closer to 0 during training; however, the accuracy drops because the correctly classified small cluster also shifts to the left of $0$, resulting in incorrect classification of its elements. 

As mentioned, stability of accuracy occurs when the accuracy of a DNN is high, and the loss is sufficiently small. This relationship between high accuracy and low loss is essential for maintaining consistent performance during the training process. One of the primary motivations of this work is to achieve sufficiently small loss, which in turn is expected to ensure high accuracy of the DNN. By focusing on finding conditions to obtain sufficiently small loss, we aim to create a direct connection between the decrease in loss and the increase in accuracy, thus promoting the stability of accuracy throughout the training process. This approach enables the development of DNNs that can maintain their performance as they train.

    \end{ex}

    \begin{ex}
        It is important to note that avoiding small clusters in $\delta X(T,\alpha)$ is a sufficient but not necessary condition for stability of accuracy. In this example, we train the same model as in Example \ref{example_1} and show the accuracy of the model over time and its $\delta X$ at epoch $1$, see Fig. \ref{toy_example_3}. Even though there is a small cluster in $\delta X$, we still have stability of accuracy.  

             \begin{figure}[h!]
         \begin{subfigure}[b]{0.3\textwidth}
         \centering
\includegraphics[width=1.5\textwidth]{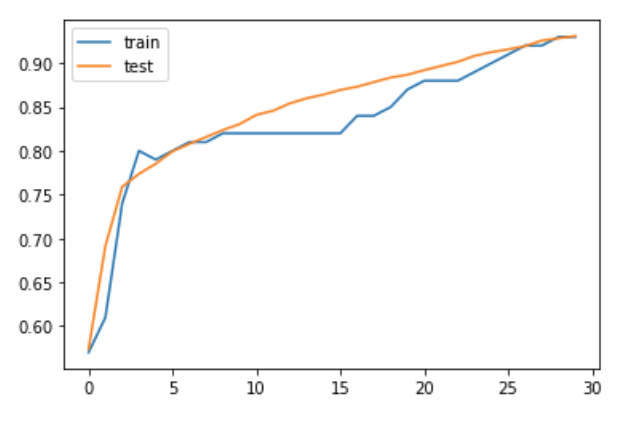}
         \caption{Accuracy of DNN during training.}    \label{toy_example_small_cluster}
     \end{subfigure}
     \hspace{0.2\textwidth}
     \begin{subfigure}[b]{0.3\textwidth}
         \centering    \includegraphics[width=1.65\textwidth]{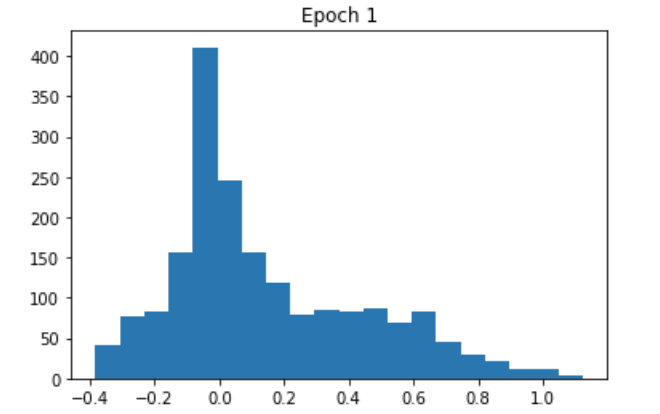}
         \caption{$\delta X(\alpha_{1},T)$. We see a small cluster near $0$.}
         \label{toy_example_after_training}
     \end{subfigure}

     \caption{Avoiding small clusters is a sufficient but not necessary condition for stability of accuracy.}
     \label{toy_example_3}
				
		\end{figure}
    \end{ex}

\subsection{Overfitting: the connection between the accuracy on the training set and accuracy on the test set}
\label{Overfitting}

Overfitting is a common issue in machine learning, particularly in deep neural networks. It occurs when a model learns to fit the training data too well, capturing not only the underlying patterns in the data but also the noise and other irrelevant details. As a result, even in cases when the model performed well on the training set, it might perform poorly on unseen or test data because it fails to generalize from the training data effectively. For more on overfitting of DNNs, see \cite{ma2018power,kawaguchi2017generalization,cohen2021learning}.

Monitoring the model's performance on both the training and test datasets during the training process can help identify overfitting. Overfitting is suspected when the model's performance on the training data continues to improve, while its performance on the test data plateaus or even worsens. This indicates that the model has become too specialized in recognizing features and patterns in the training data, making it unable to generalize to new, unseen data.

When training a DNN, it is crucial to ensure that the accuracy of the model on the training set increases while the loss decreases. If a model's accuracy on the training set increases and the DNN does not overfit, it indicates that the model is learning the underlying structure of the data and can generalize well to new, unseen data. Thus, to achieve a well-performing DNN, it is essential to strike a balance between increasing accuracy on the training set and avoiding overfitting. This can be done by using techniques such as regularization, dropout, and careful selection of optimization algorithms and their hyperparameters. By ensuring that the model's accuracy increases on the training set while avoiding overfitting, it is possible to obtain a DNN that can generalize well and perform effectively on new, unseen data.

In Examples \ref{new_example2}, \ref{new_example3} and Subsection \ref{example1_new}, we see that data sets that satisfy the SUDC for more slabs have that DNNs training on them achieve higher accuracy and lower loss not only on the training set but also on the test set. This can be explained as follows: DNNs training on data that satisfy the UDC for more slabs (see Def. \ref{def:no_small_islated_data_cluster_2}) achieve lower loss on the training set (see Theorem \ref{weakenedDC1}). The SUDC is a decent approximation of the UDC (see Subsection \ref{are_truncations_important}). This leads to higher accuracy on the training set and greater stability of accuracy (see, e.g., \eqref{acc_loss}). If the DNN is not overfitting, this also leads to higher accuracy and lower loss on the test set.

    	    \section{Main results and  numerical implementation}
         \label{Main_result}
    	    
    	    The following are some important definitions we must introduce before the main result of this paper is stated. As mentioned, our goal is to ensure the stability of accuracy during training (see Theorem \ref{weakenedDC1}). To this end, we introduce a condition on the training set $T$ called the \emph{uniform doubling condition}.  

We define the uniform doubling condition on the data set $T\subset \R^n$ via \emph{truncated slabs} in $\R^n$.

\begin{defn}[Slab $S'$ in $\R^n$]
   A slab $S'$ with width $\kappa$ is defined by some unit vector $u$ and center $s \in \R^n$ as, 

\begin{equation}
S'_\kappa:=\{x:\ |(x-s)\cdot u|\leq \frac{\kappa}{2}\}.
\label{slab}
\end{equation}

\end{defn} 

 Thus a \textbf{slab $S'$ in $\R^n$} is the domain between parallel $(n-1)$-dim hyperplanes, see Fig. \ref{fig:slab}. The \textbf{width of a slab} is given by the (shortest) distance between the bounding hyperplanes.  Thus, a slab can be defined by a unit   vector, a center, and a width, see Fig. \ref{fig:slab}.

\begin{figure}[h!]
				\includegraphics[width=.5\textwidth]{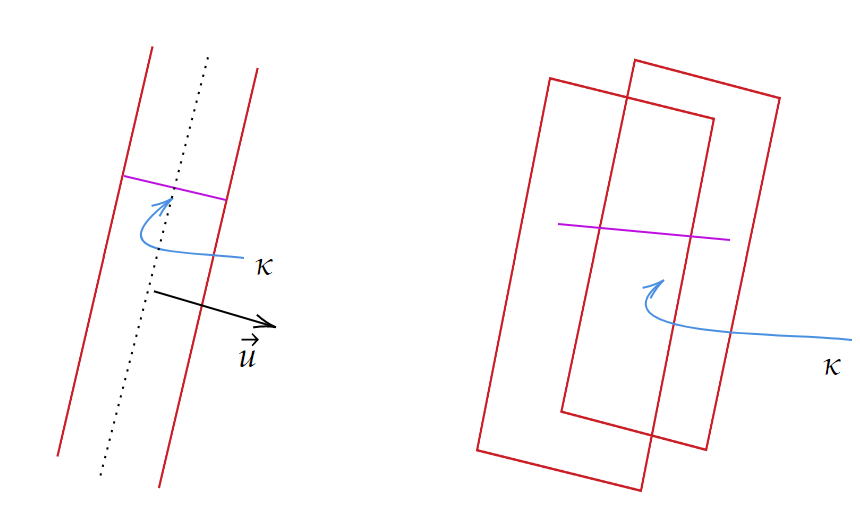}
				\caption{left: slab in $\R^2$. Width is $\kappa$, the center is a dashed line, and normal vector $\vec{u}$ is shown.  right: slab in $\R^3$.}
				\label{fig:slab}
			\end{figure}


	\begin{defn}[Truncated slab $S$ in $\R^n$]
   A truncated slab $S$ is the intersection of a slab defined by some unit vector $u$ and center at $s \in \R^n$ with a finite number $k$ of linear constraints

\begin{equation}
S_\kappa=S_\kappa(u,s,v_1,t_1,\ldots,v_k,t_k):=\left\{x:\ |(x-s)\cdot u|\leq \frac{\kappa}{2}\ \mbox{and}\ v_i\cdot x\leq t_i\ \forall i=1,\ldots, k\right\},\label{truncatedcylinder0}
\end{equation}
where the $v_i \in \R^n$ are the normal vectors of hyperplanes which form the truncation set. We call $k$ the truncation number (the number of elements in the truncation set). See Fig. \ref{fig:slabtruncated} for an example.  
\label{truncated_slab}
\end{defn}

				

			\begin{figure}[h!]
				\includegraphics[width=\textwidth]{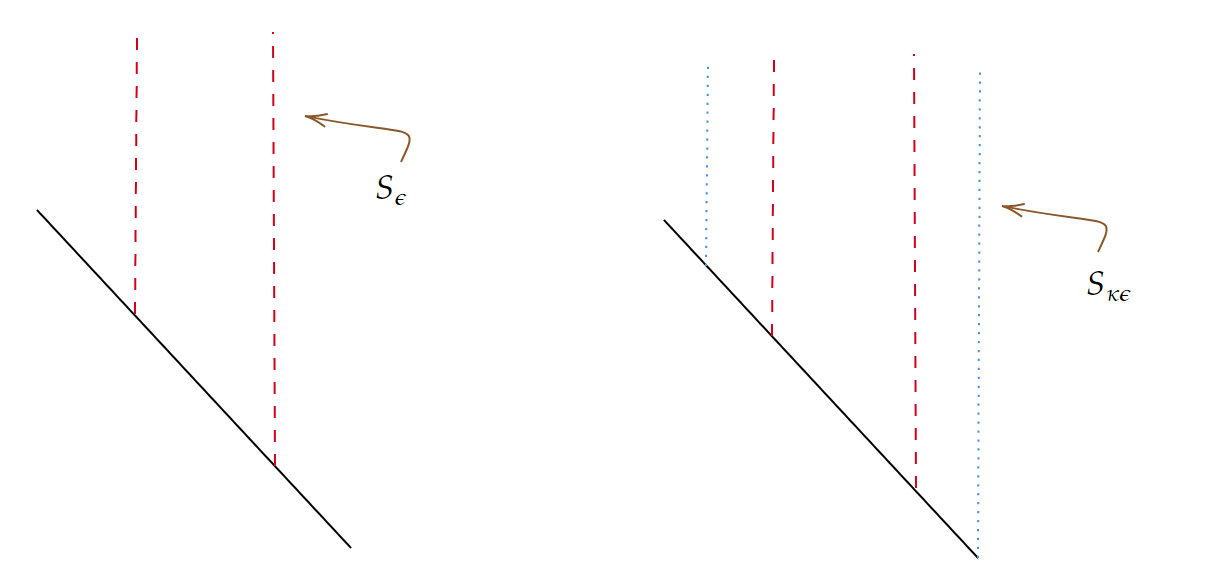}
				\caption{left: truncated slab $S_{\epsilon}$. The truncation is a solid line (in black) and the slab is a dashed line (in red). right: $S_{\epsilon}$ is shown in dashed lines (red) and its scale $S_{\kappa \epsilon}$ is shown as the dotted lines (in blue).}
				\label{fig:slabtruncated}
			\end{figure}
			
 A truncation set of size $k$ is a set of $k$ hyperplanes. We now consider a set $L$ which is a family of size $k$ truncation sets. For a family $L$ of truncation sets with $k$ elements in each truncation set we define the uniform doubling condition on the training set $T$:  
	\begin{defn}\label{def:no_small_islated_data_cluster_2}
		Let $\bar\mu$ be the extension of the measure $\mu$ from $T$ to $\mathbb R^n$ by $\bar\mu(A)=\mu(A\cap T)$ for all $A\subset\mathbb R^n$. Let $L$ be a family of size $k$ truncation sets. The \emph{uniform doubling} condition holds on $T$ if there exists  $\kappa>1$ and $\delta,\sigma, \beta,\ell>0$ so that $\forall$ truncated slabs $S$ s.t. $S$ is generated by a truncated set $h \in L$ there exists  $\ell$, such that for all $\ell\kappa^i<\beta$ with $i \in \mathbb N$,  we have

		\begin{equation}\label{eq:no_small_isolated_data_cluster_2}
		\bar\mu(S_{\ell\kappa^{i+1}})\geq \min\{\delta,\;(1+\sigma)\,\bar\mu(S_{\ell\kappa^i})\}.
		\end{equation}
		
	\end{defn}	

Essentially, this condition requires that as we double the width of a slab $S$ by $\kappa$ to obtain the slab $S_{\kappa}$ we capture more data than was originally in the slabs $S$ (i.e., $S_{\kappa}$ contains more data than the slab $S$). See Fig. \ref{doubling_condition_with_data} as an example. This ensures that the data set $T$ does not have small clusters of data which in turn ensures the stability of accuracy of a DNN during training.

\begin{figure}[h!]
				\includegraphics[width=.5\textwidth]{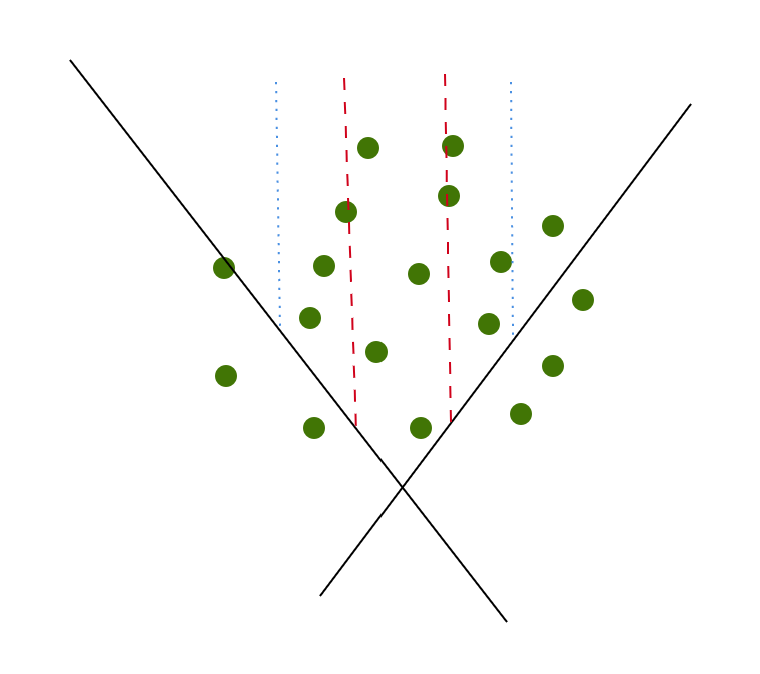}
				\caption{ Truncated slab $S_{\epsilon}$. The truncation is  a solid line (black)  and the original slab is a dashed line (in red). We also see the scaled slab $S_{\kappa \epsilon}$ as a dotted line (in blue). }
				\label{doubling_condition_with_data}
			\end{figure}

For reasons that will be explained shortly, we call $\kappa$ the doubling constant, $\ell$ the scaling factor,  and $\beta$ the scale parameter.
		
		As mentioned, $\beta$ is the scale parameter (of the doubling condition). Put simply, it is the largest width of the slab $S$ for which the doubling condition is satisfied. The larger $\beta$ the more  slabs must satisfy the doubling condition. If $(1+\sigma)\mu(S)\geq \delta$ the doubling condition is automatically satisfied. Thus, we can think of $\delta$ as the largest mass for which the doubling condition must be satisfied. $\kappa$ is the factor by which we "double" the slab $S$, and so it is referred to as the doubling constant. $\ell$ is the scaling factor, and for a slab $S$ we have that $\ell$ is the smallest width for which the doubling condition for $S$ must be satisfied.  $L$ is a family of size $k$ truncations, and so the doubling condition holds for some set $L$ if $\forall h \in L$ and slab $S$ truncated by the $k$ elements in $h$, we have that $S$ (with the truncations) satisfies the doubling condition. 

The truncation set $L$ is closely associated with the non-linear absolute value activation function in the DNN. To clarify this connection, we first introduce the truncation set of a DNN.

			\begin{figure}[h!]
				\includegraphics[scale=.7]{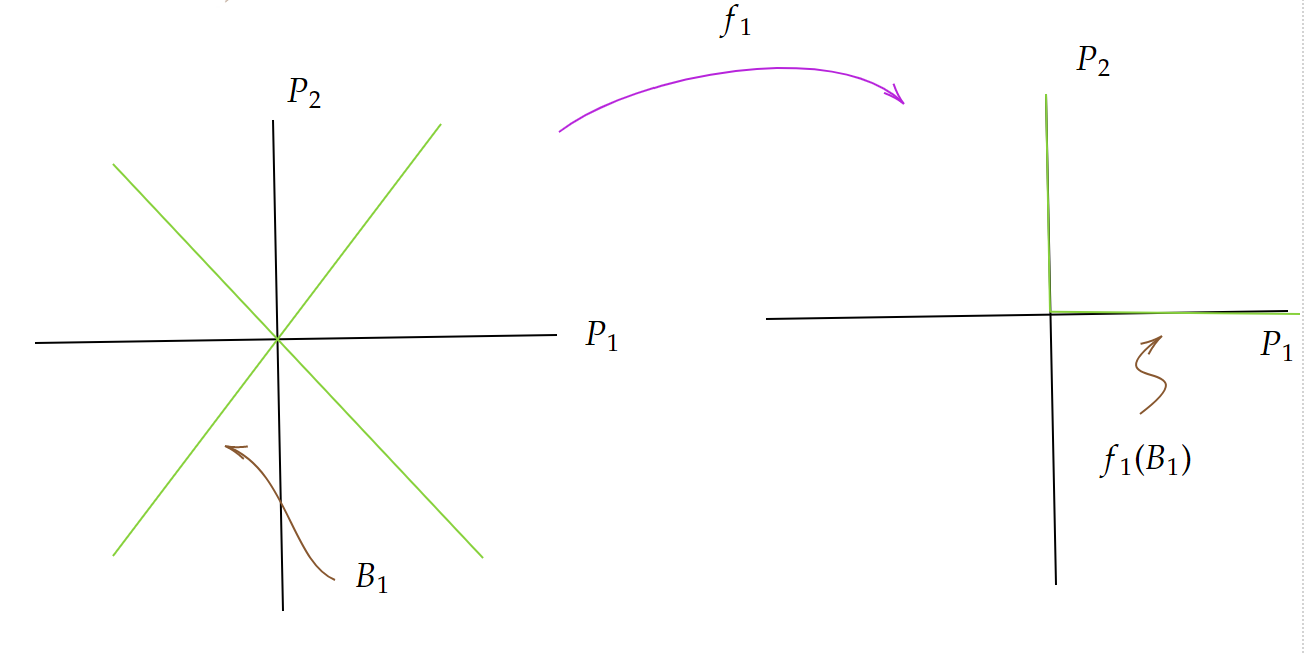}
				\caption{Mapping $f_1:\R^2\to \R^2$. The set $B_1$ is shown in green.}
				\label{B_1}
			\end{figure}

	We  define the truncation set of a DNN as follows:

   \begin{defn} We start with a DNN with only one layer:
   \begin{enumerate}
       \item  Take $f_1=\lambda \circ M_1 :\R^n \to \R^m$ (layer map), with $\lambda(x_1, \dots, x_m)=(|x_1|, \dots, |x_m|)$ the activation function and $M_i$ a affine map, to be a DNN. Consider the set of coordinate hyerplanes of $\R^m$, $P_1=(0,x_2,\dots,x_m), \dots, P_m=(x_1,x_2,\dots,0)$. Define \textit{first layer truncation set} of hyperplanes $B_1 =\{f_1^{-1}(P_1),f_1^{-1}(P_2),\cdots,f_1^{-1}(P_m)\}$ for the $f_1^{-1}(P_{w})$ which are $(n-1)$ dimensional and $1 \geq w \geq m$. See Fig. \ref{B_1}. 
       
       For a DNN with more layers, we define $i$th layer of the truncation set of the DNN in a similar manner: 
       \item Take $F_i=f_i \circ f_{i-1}\dots f_2 \circ f_1: \R^n \to \R^{c(i)}$, with $c(i)$ the number of dimensions of the $i$th layer of the DNN. See Fig. \ref{whatareF_i}. Take the coordinate hyerplanes $P_1=(0,x_2,\dots,x_{c(i)}),  P_2=(x_1,0,\dots,x_{c(i)}), \dots, P_{c(i)}=(x_1,x_2,\dots,0)$. Define the \textit{$i$th layer truncation set} \begin{equation}
       B_i=\{F_i^{-1}(P_1), F_i^{-1}(P_2),\dots,F_i^{-1}(P_{c(i)})\}, \end{equation}
       for the $F_i^{-1}(P_{w})$ which are $(n- 1)$ dimensional and $1 \geq w \geq c(i)$.     
       \item  We take $B= \bigcup_{i=1}^l B_l$ to be the \textit{DNN truncation set}, with $l$ the numbers of layers in the DNN.
   \end{enumerate}
      
\end{defn}
				
			 To ensure stability for $f_1$ (as a $1$ layer DNN), the data set $T$ must satisfy the doubling condition only with truncations $B_1$.  Similarly for DNNs with more than one layer. $B$ is a family of sets of hyperplanes (or half hyperplanes which for simplicity we will refer to as hyperplanes) that are mapped to the coordinate  planes under layer maps of our DNN. As we will see, in order for a DNN to have stability of accuracy we must have that the truncation set of the DNN is an element of the family of truncation sets $L$ for which the training set $T$ satisfies the doubling condition. 
			
			We now introduce two more definitions before the main result is stated. 
			
				\begin{figure}[h!]
	\includegraphics[scale=.7]{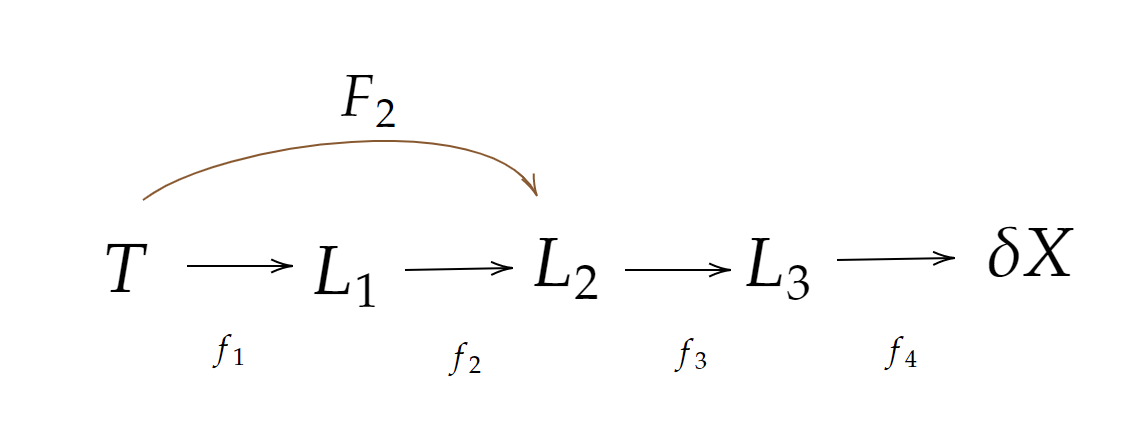}
	\caption{What are the $F_i$ maps?}
		\label{whatareF_i}
			\end{figure}

	\begin{defn}\label{whatisd}
		Take $d_{\min}=|\min_i{(\sigma_i)^2}|$ and $d_{\max}=|\max_i{(\sigma_i)^2}|$ with $\sigma_i$ the nonzero singular values of the linear map  $W_n \circ \dots \circ W_2 \circ W_1$ (in general the $M_i$ are affine maps, but when obtaining $d_{\min}$ and $d_{\max}$ we ignore the bias vectors of $M_i$ and only consider their linear part $W_i$). 
		
	\end{defn}
	
	Finally, define the good sets and bad sets of $T$ as follows:
	for $\eta\geq 0$, the \emph{good set} of margin $\eta$ at time $t$ is
	\begin{equation}\label{eq:good_set}
	G_\eta(t)=\{s\in T:\delta X(s,\alpha(t))>\eta\}
	\end{equation}
	and the  \emph{bad set} of margin $\eta$  at time $t$ is
	\begin{equation}
	B_{-\eta}(t)=\{s\in T:\delta X(s,\alpha(t))<-\eta\}.
\end{equation}

The main result of this paper is the following Theorem:

		\begin{thm}\label{weakenedDC1}

			Assume that a data set $T \subset \R^n$ satisfies the  doubling condition (\ref{def:no_small_islated_data_cluster_2}), with $\kappa>1$ and $\delta,\sigma$, $\beta>0$, $1>\ell>0$ and family of truncations $L$ with truncation sets of size $k$.	Then for a.e. DNN with truncation set  $B \in L$ at time $t_0$ we have that  there exists constants $C_2=C_2(\eta, K ,\kappa,\sigma,\beta,\delta_0,\ell,d_{\min},d_{\max})$ and $\xi'$ such that if we have: 
			\begin{equation}\label{eq:explicit_constants_2_main_theorem}
		\quad \frac{d_{\min}\beta}{2} \geq \eta, \quad 0\leq \xi' \in [d_{\min}\ell, d_{\max}\ell],   \quad \quad \xi' < \eta, \quad   \delta_0<\delta,
		\end{equation} 
			and if the good and bad sets at $t=t_0$ satisfy
			\begin{equation}\label{good_set_big_and_bad_set_empty_B}
				\mu(G_\eta(t_0))>1-\delta_0\quad\text{and}\quad B_{-\xi' }(t_0)=\emptyset,
			\end{equation}
			then for all $t\geq t_0$,
			\begin{equation}\label{eq:stability_B}	
    \text{loss}(t)\leq \log(2)C_2, \quad
    \text{acc}(t)=\mu(G_0(t))\geq 1-C_2.	\end{equation}
			
		\end{thm}

  Here we have that \begin{equation}
\label{eq:constant_C_2}
\begin{split}
\ & C_2(\eta, K ,\kappa,\sigma,\beta,\delta_0,\ell,d_{\min},d_{\max}):=  \\
& \max_{\xi \in [d_{\min}\ell , d_{\max}\ell]}  \frac{1}{\log(2)} \left(\log(1+(K-1)e^{-\eta})+\log(1+(K-1)e^{-(\frac{\eta}{\kappa})})\delta_0\right. \\
&\left.+ \frac{\delta_0}{(\sigma+1)^{\frac{\log\left(\frac{\eta}{\xi}\right)}{\log \kappa}-1}}\log(1+e^{\xi}(K-1))\right. \\
&\left.+ \sum_{i=0}^{p-1} \frac{\delta_0}{(\sigma+1)^{\frac{\log\left(\frac{\eta}{\xi \kappa^{i+1}}\right)}{\log \kappa}-1}}\log(1+(K-1) e^{-\xi \kappa^{i+1}})\right),
\end{split}
\end{equation}
where $p=\lfloor \frac{\log\left(\frac{\eta}{\xi}\right) }{\kappa}\rfloor$.

Here, the condition that the truncation set of the DNN, $B$, is an element of $L$, and condition \eqref{eq:explicit_constants_2_main_theorem} are only required at $t_0$.

\begin{remark}
    We would like the constant $C_2$ to be as small as possible so that accuracy is bounded by a large number (a number close to $1$) and loss is bounded by a small number. It is important to note that the $\log$ functions which appear in $C_2$ come from the cross-entropy loss \eqref{loss_function} and depend on the $\log$ in the cross-entropy loss. The exponential functions (such as $e^{-\eta}$), on the other hand, come from softmax \eqref{soft_max}. Because of this, if we use $\log_{10}$ for the cross-entropy while training (instead of $\ln$) we will have a smaller $\log(2)C_2$ constant, meaning that loss will be smaller. This, however, will not decrease the constant $C_2$ (or increase the bound on accuracy) given that we are dividing by $\log(2)$.

    Changing base in softmax \eqref{soft_max} does, however, change the constant $C_2$. By changing the base in softmax from $e$ to some constant $b$, one can replace every $e$ in $C_2$ with $b$. In Fig. \ref{constant_b} we show how $C_2$ changes based on $b$, though in this example we fix $\xi=.4$. We also take $K = 2$, $\eta = 18$, $\delta_0 = 0.2$, $\kappa = 2$ and $\sigma = 0.9$. We see that $C_2$ decreases from $.06$ when $b=2$ to $.03$ when $b=5$. This is a significant change in terms of stability of accuracy (i.e. bounding accuracy by $.97$ is much better than bounding it by $.94$). Throughout the paper, we use Desmos to compute $C_2$ because of its level of precision.  

    \begin{figure}[h!]
\includegraphics[width=.8\textwidth]{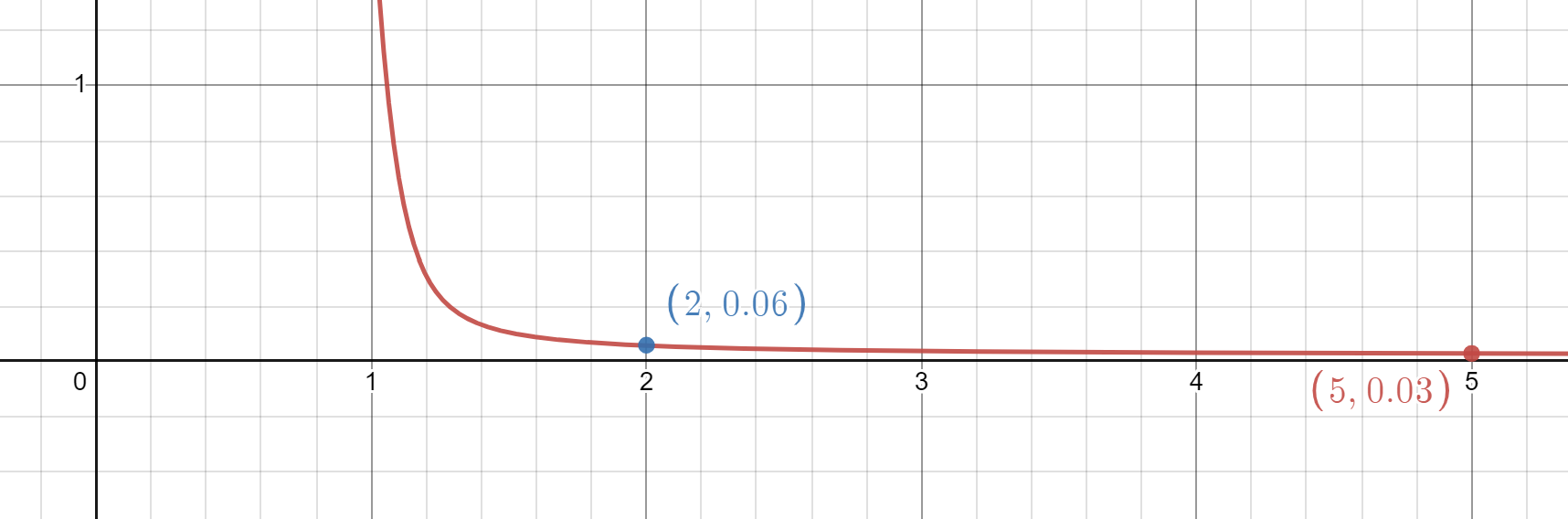}
\caption{$C_2$, represented on the $y$-axis, for different values of $b$, represented on the $x$-axis, with fixed $\xi=.4$. We see that $C_2$ decreases from $.06$ when $b=2$ to $.03$ when $b=5$. This is a significant change in terms of stability of accuracy.}
\label{constant_b}
\end{figure}

  \end{remark}

\subsection{Interpretation of the constant $C_2$}
  We would like the constant $C_2$ to be small so that accuracy is high. To make $C_2$ small, we need to understand how the different parameters in the equation affect its value. Let's analyze each part of the equation:
\begin{enumerate}
    \item $\log(1+(K-1)e^{-\eta})$: To minimize this term, we want $\eta$ to be large, because when $\eta$ is large, $e^{-\eta}$ approaches $0$, and the term inside the logarithm approaches $1$, leading to a smaller value of the logarithm. The connection between $\eta$ and the uniform doubling condition is given in the condition \eqref{eq:explicit_constants_2_main_theorem}: 
    \begin{equation}
        \frac{d_{\min}\beta}{2} \geq \eta.
    \end{equation}
    This means that we want $\beta$, from the doubling condition, to be large. That is, we want the widths of the slabs which satisfy the doubling condition to be as large as possible. 
\item $\log(1+(K-1)e^{-(\frac{\eta}{\kappa})})\delta_0$: We want $\delta_0$ to be small to minimize this term since it is multiplied by the whole expression. Additionally, having a large $\eta$ will also help in making this term small for the same reasons explained in point (1).

The connection between $\delta_0$ and the uniform doubling condition is given by the condition \eqref{eq:explicit_constants_2_main_theorem}:

\begin{equation}
 \delta_0<\delta,
\end{equation}
with $\delta$ coming from the uniform doubling condition. In essence, the larger $\delta$ is, the bigger the measure of the training set that must satisfy the DC, see \eqref{eq:no_small_isolated_data_cluster_2}. 

\item $\frac{\delta_0}{(\sigma+1)^{\frac{\log\left(\frac{\eta}{\xi}\right)}{\log \kappa}-1}}\log(1+e^{\xi}(K-1))$: To minimize this term, we want $\xi$ to be small, because when $\xi$ is small, $e^{\xi}$ approaches 1, and the term inside the logarithm approaches $1+(K-1)$. Thus, we would want $d_{\max}\ell$ to also be small and we can ensure that $d_{\max}\ell$ is small by checking the uniform doubling condition on the training set for small $\ell$. This is, the size of the widths of the slabs that we start checking the UDC with should be small.

Minimizing this term also requires $\kappa$ and $(1+\sigma)$ to be roughly the same size, as this causes the term to decrease. Essentially, this means that "doubling" the width of the slab should "double" the number of points contained in the doubled slab. Ideally, $(1+\sigma)$ should be much larger than $\kappa$, but satisfying the uniform doubling condition is unlikely in such cases. Also, a small $\delta_0$ will help in minimizing this term.

 \item $\sum_{i=0}^{p-1} \frac{\delta_0}{(\sigma+1)^{\frac{\log\left(\frac{\eta}{\xi \kappa^{i+1}}\right)}{\log \kappa}-1}}\log(1+(K-1) e^{-\xi \kappa^i})$: To minimize this term, we want $\kappa$ and $(1+\sigma)$ to be of about the same size, because when their sizes are similar, the term inside the summation becomes smaller. Additionally, as $\xi$ decreases, the terms inside the logarithms increase, and as a result, the sum increases. However, increasing $\kappa$ can help in reducing the size of the terms inside the logarithms, leading to a smaller sum overall. Finally, $p=\lfloor \frac{\log\left(\frac{\eta}{\xi}\right) }{\kappa}\rfloor$. This value of $p$ represents the number of terms in the sum and is determined by the value of $\xi$, which is also included in the denominator of the logarithmic term inside the sum. As $\xi$ decreases, the value of $p$ increases, meaning that the sum contains more terms, and thus, the overall value of the sum might increase. Again, a small $\delta_0$ will help in minimizing this term.

\item The function depends on $K$ through the terms inside the logarithms. As $K$ increases, the terms inside the logarithms also increase, but since we are taking the logarithm of these terms, the effect of increasing $K$ is dampened. Therefore, we can say that the function has a sub-linear dependence on $K$.
\end{enumerate}

In summary, to make the equation small, we would like a large $\eta$, small $\xi$, small $\delta_0$, large $\beta$ and large $d_{\min}$, so that \eqref{eq:explicit_constants_2_main_theorem} is satisfied, and to make $\kappa$ and $(1+\sigma)$ of about the same size. The dependence on $K$ is sub-linear.

A large $\eta$ makes the good set $G_\eta(t_0)$ small and so $\delta_0$ would be large. In general, we try to balance having a large $\eta$ and small $\delta_0$. Numerically it seems like having $\eta \approx 18$ and $\delta_0 \approx .2$ is reasonable for examples such as Example \ref{example_1} when the DNN has high accuracy. To have a large $\eta$, we need a large $\beta$ and $d_{\min}$ to not be to small, see Remark \ref{d_min} for more on $d_{\min}$. We therefore also want $\delta$ to be large, so that $\delta_0$ can be large, see \eqref{eq:explicit_constants_2_main_theorem}. Finally, because $C_2$ is defined as the max over $ {\xi \in [d_{\min}\ell , d_{\max}\ell]}$, we want $d_{\max}\ell$ to also be small. However, we don't want it to be too small otherwise finding a ${\xi' \in [d_{\min}\ell, d_{\max}\ell]}$ for which $B_{-\xi' }(t_0)=\emptyset$ might be impossible.  

For a simple example, when   $\delta_0=0.2$, $\eta=18$, $K=2$, $\kappa=2$, $\sigma=0.9$, $\ell=.001$, $d_{\min}=.01$ and $d_{\max}=800$  we have that $C_2=.05$, and so assuming the conditions of Theorem \ref{weakenedDC1} are satisfied we have that the accuracy is bounded by $.95$ and loss is bounded by $.015$. More on this will be discussed in Section \ref{p_results}.

Finally, the variable $\xi$ is probably the most important parameter for the constant $C_2$, and while we don't know it before training (given that we don't know $d_{\min}$ and $d_{\max})$, we can still study how stable training on some given data might be for different assumptions of what $\xi$ will be after training. In Fig. \ref{constant_xi} we show how $C_2$ depends on $\xi$, with the other parameters again being $\delta_0=0.2$, $\eta=18$, $K=2$, $\kappa=2$, $\sigma=0.9$. We see that for $\xi<1$ one can obtain significant stability of accuracy results (i.e. accuracy will be bounded from below in a significant manner). We can obtain a small $\xi$ by ensuring that $\ell d_{\max}$ is small, which means checking the uniform doubling condition on $T$ for a small enough $\ell$. We also need the bad set to be small, but we are assuming that accuracy on the training set is already high in which case the bad set would be small, see Example \ref{result_deltax}. Having a small bad set does not mean that training should stop, because it might still be worthwhile to increase the size of the good sets $G_{\eta}$ and make the bad set even smaller. However, stability of accuracy will ensure a lower bound on the accuracy and loss of the DNN as training progresses. $\eta$ is probably the second most important constant and we will talk more about this constant in Section \ref{p_results}.  

\begin{figure}[h!]
\includegraphics[width=.8\textwidth]{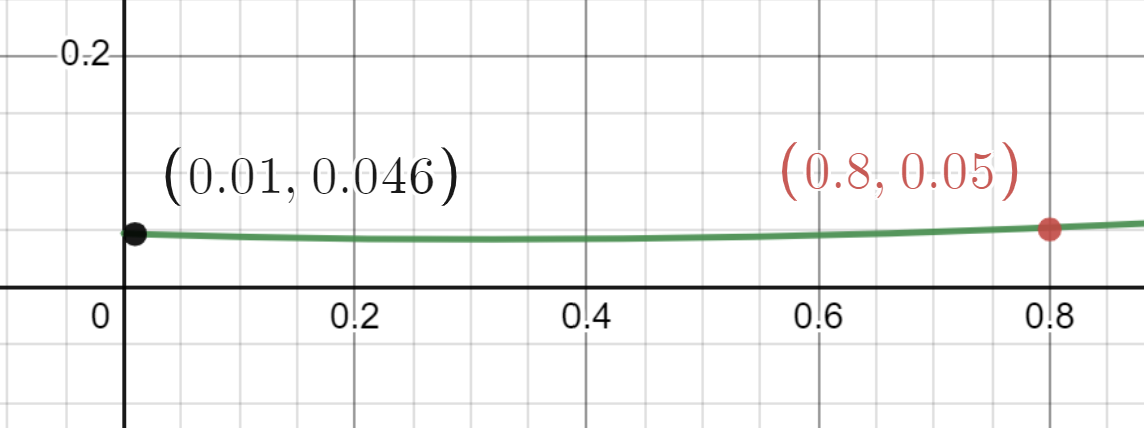}
\caption{$C_2$, represented on the y-axis, for different values of $\xi$, represented on the x-axis.}
\label{constant_xi}
\end{figure}

\begin{remark}
    One can change condition \eqref{eq:explicit_constants_2_main_theorem} so that $\xi'  \in [d_{\min}\ell\kappa^i, d_{\max}\ell\kappa^{i'}]$ for $i, i' \in \N$. Then $C_2$ is taken as the max over $\xi'  \in [d_{\min}\ell\kappa^i, d_{\max}\ell\kappa^{i'}]$. For $i,i'=0$, we obtain the previous result. This can allow better control when choosing $\xi'$ and also means that we can obtain a smaller $C_2$.   
\end{remark}

Now, we discuss the conditions on the DNN. First, the truncation set $B$ of the DNN is a family of sets of hyperplanes in $\R^n$ and we require that no object in the training set be contained in one of these hyperplanes. This is why we say that for a.e. DNN for which $B \in L$ do we have the stability of accuracy. The a.e. is to ensure that no element in $T$ is contained in a hyperplane of one of the sets of $B$. Further, we require that $B \in L$ at $t_0$, which in practice would mean that for each element in $L$ the number of truncations would be very large (i.e. exponential in the layers $l$ of the DNN) and so checking the uniform doubling condition with even one truncation set might not be feasible. 

This is why we introduce the simplified version of the uniform doubling condition and numerically show that, on a model example, this SUDC is correlated with the stability of accuracy, see Subsection \ref{DX_doubling_num}. In this sense, Theorem \ref{weakenedDC1} is still useful because we can check the SUDC on a training set for constants that would make $C_2$ small (assuming the good set is large and the bad set is small). Thus, even though the truncations are being ignored, we can still study numerically how the SUDC provides information on the stability of accuracy of the DNN using $C_2$ as a guide, see Example \ref{DC_vs_loss_example_1}. In Subsection \ref{are_truncations_important} we explain why truncations might not be important when verifying the uniform doubling condition, and why it might be reasonable to assume that the SUDC provides important information for the stability of  DNNs.                        

In conclusion, in this theorem, a training set $T$ which satisfies the uniform doubling condition and a set of DNNs is given, with the DNNs defined by their architecture and parameters. If at $t_0$ a DNN from this set trained on the training set $T$ has high accuracy, in the sense that its good set is big and its bad set is small, then the DNN will have high accuracy for all time. This is because the DNNs loss will be small at $t_0$. In Subsection \ref{DC_delta_X} we show how a doubling condition on $\delta X$ leads to a similar result, and this is because \eqref{eq:cond_B_loss_estimateWDC} gives a bound on the loss while \eqref{acc_loss} bounds accuracy in terms of loss. The doubling condition on $\delta X$ can be checked numerically, but it is much simpler to just look at the loss at $t_0$ and see if it is small, in the sense that equation \eqref{acc_loss} is the reason accuracy is high.

\begin{remark}
  \label{d_min}

 For most DNNs $d_{\min}$ can be very small (almost zero) and so satisfying \eqref{eq:explicit_constants_2_main_theorem} might be difficult. However, it has been shown that one can perform an SVD on the layer matrices of a DNN and remove the small singular values without affecting the accuracy of the DNNs, see \cite{xu2019trained, yang2020learning,xue2013restructuring,cai2014fast,anhao2016svd, berlyand2023enhancing}. This SVD pruning method is a valuable technique when training DNNs. During the training process of a DNN, every few epochs one can apply the SVD on the weight layer matrices, denoted as \(W_i\). Singular values that are smaller than a specified threshold can be pruned (set to zero) to reduce the size of the model. By judiciously selecting an appropriate threshold, one can ensure that the accuracy of the DNN remains largely unaffected, thus achieving a balance between model compression and performance. This pruning technique works well with standard training algorithms such a SGD and Adams; see \cite{berlyand2023enhancing}.

 For example, it was been shown that for DNNs trained on MNIST with a single $1000 \times 1000$ weight layer matrix $W$, one can generally remove the small singular values to make $d_{\min}>\sqrt{700}$ without changing the accuracy of the DNN, see \cite{shmalo2023deep}. In general, many of the small singular values of a DNN can be removed without changing its accuracy, see \cite{yang2020learning, shmalo2023deep,berlyand2023enhancing, staats2023boundary}. After the small singular values are removed, the new DNN will have the stability result given in this theorem, though with much bigger $d_{\min}$. 

Further, for a DNN with a single $n \times n$ weight layer $W$, under the assumption that by removing small singular values of a weight matrix $W$ one is only removing the noise in $W$ while preserving the information contained in $W$, it should be the case that one can remove all singular values of $W$ smaller than $2.858 \sigma_{med}$ where $\sigma_{med}$ is the median empirical singular value of $W$, without decreasing accuracy, see \cite{donoho2013optimal}. More on this will be discussed in a different paper. For more on applying singular value decomposition for DNNs, see \cite{yang2020learning,xue2013restructuring,cai2014fast,anhao2016svd}.        
\end{remark}

\subsection{Some numerical results}
\label{DX_doubling_num}

We verify these results numerically for examples such as Example \ref{example_1}. The given numerical simulations aim to investigate the relationship between the uniform doubling condition of a slab and the loss and accuracy of deep neural networks (DNNs) on various data sets. In these examples, the absolute value and Leaky ReLU activation functions are used and cross-entropy loss function is what we train with. We only check the uniform doubling condition for slabs and don't involve the truncation sets in these numerics. We call this the simplified uniform doubling condition (SUDC).  

Here, the SUCD involves the expansion of the slab's thickness by a factor $\kappa$, followed by counting the number of data points inside the slab. This process is repeated until the number of new data points inside the slab is less than $\sigma+1$ times the number of data points in the previous iteration, see Fig. \ref{SUDC_example1}. The code calculates the average number of doubling steps, the average number of points in the final slab, and the average width of the final slab for multiple slabs. The average width of the final slab is the most important of these quantities and gives us an estimate of what $\beta$ might be for the training data set. From here on, we will call this quantity $\bar{\beta}$ and assume that it is a good estimator for $\beta$. For simplicity, we also sometimes refer to $\bar{\beta}$ as the average width of DC (i.e. the average of the largest widths for which SUDC is satisfied). 

The data sets are generated using the same method as in Example \ref{example_1}. To create variability across the data sets, the number of samples is chosen randomly within a predefined range. In Example \ref{DC_vs_loss_example_1} each data set has between $200$ and $15,000$ objects and the number of objects is chosen uniformly from this range. Each data set comprises data points and corresponding labels, where the data points lie in a two-dimensional space and the labels represent the target class of each point. 

After generating a data set, it is split into training and testing subsets. The number of training samples is also chosen randomly, within a range of $100$ to the total number of samples in the data set. This random selection process ensures that the trained DNNs are exposed to varying amounts of training data, enabling the analysis of the relationship between the doubling condition, loss, and accuracy of the DNNs under different training scenarios.      

For each of the generated data sets, the code trains a DNN and evaluates its performance in terms of loss and accuracy. The SUCD is then checked for the training data sets, and $\bar{\beta}$ is calculated for the training data. The results are collected and sorted by the size of $\bar{\beta}$. For an example of the SUCD applied to a data set, see Fig. \ref{SUDC_example1}. 

\begin{figure}	\includegraphics[scale=.65]{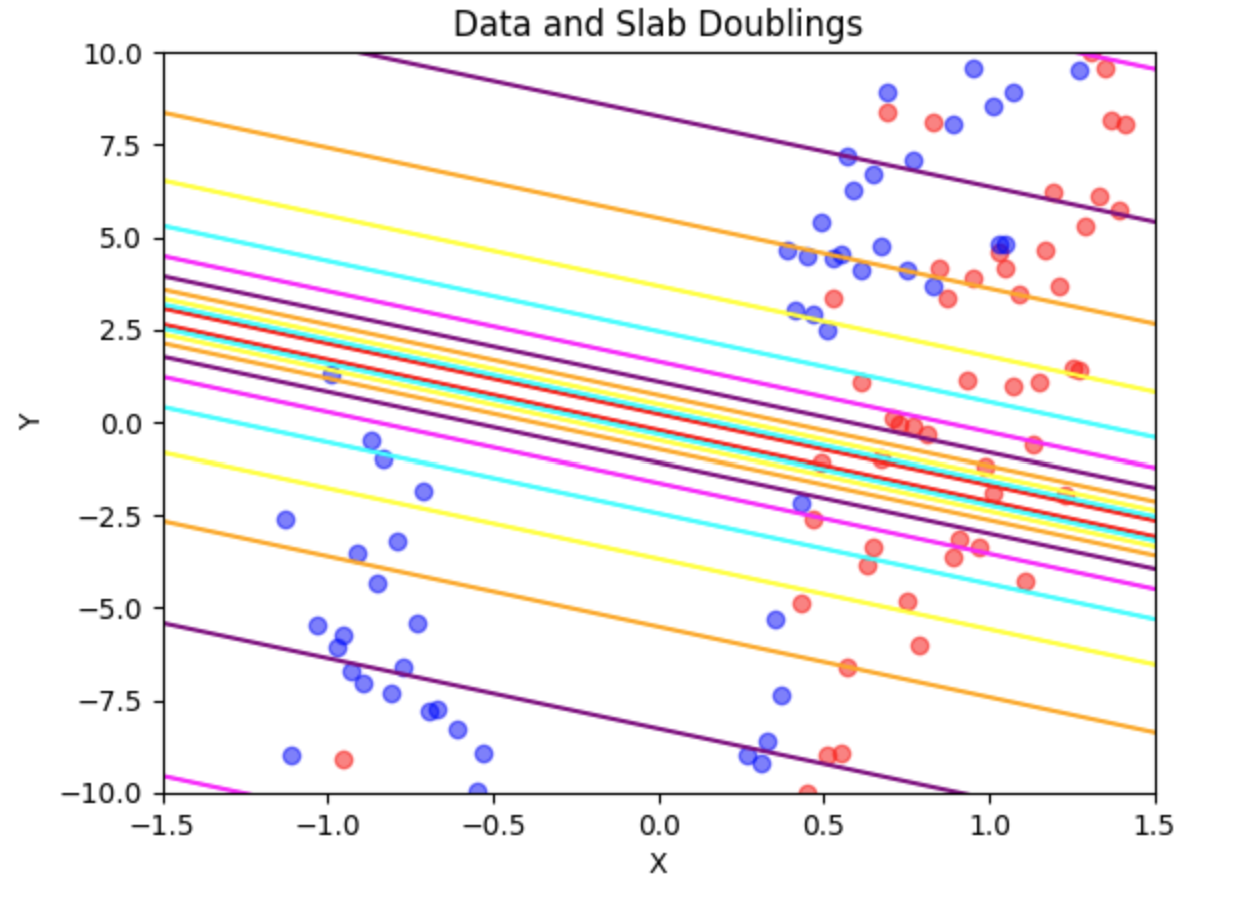}
			\caption{SUDC for data set with two classes, red and blue. We start with a single slab of small width and continue to double it (shown in different colors) until the doubling condition is not satisfied.}
			\label{SUDC_example1}
		\end{figure}

The code proceeds to plot various graphs to visualize the relationships between the doubling condition, loss, and accuracy of the DNNs. 

\begin{ex}
\label{DC_vs_loss_example_1}
    In this example, we take $\delta=1$, $\kappa=2$, $\sigma=.9$, and $\ell=.001$. We train $100$ different DNNs on $100$ different data sets and check the SUCD for all data sets using $50,000$ initial slabs and doubling them until the SUDC is not satisfied. We train the DNNs for $5$ epochs, and each DNN has the same number of layers ($4$ hidden layers with $1000$ nodes in each layer). In Fig. \ref{ACC_VS_CD_stability_1} we show the accuracy of the DNNs on the training set vs $\bar{\beta}$.  Then, we find the average accuracy of the DNNs on the training set and the $50$ DNNs with accuracy closest to this average accuracy. In Fig. \ref{Loss_vs_DC_stability_1} we graph the loss on the training set vs $\bar{\beta}$ (for the training set), for the $50$ chosen DNNs. Fig. \ref{Loss_vs_DC_stability_1} shows that even DNNs with similar accuracy will have lower loss if the data they trained on had more slabs that satisfied this simplified uniform doubling condition (i.e. the $\bar{\beta}$ is bigger). This means that those DNNs, with lower loss, will have greater stability of accuracy. As we see from both graphs, training sets that satisfy the SUDC for more slabs  have that the DNNs trained on them obtain lower loss and, in this example, even higher accuracy after the same epochs of training. 

    One can use $C_2$ to analyze the behavior in Fig. \ref{Loss_vs_DC_stability_1}. The main reduction in loss happens when $\bar{\beta}$ changes from $0$ to $.1$. By increasing the size of $\bar{\beta}$, we increase the possible size of $\eta$. For these DNNs, $d_{\max} \approx 1000$ and $d_{med} \approx 700$ (though there is variation), meaning that taking $\eta<700 \frac{\bar{\beta}}{2}$ is probably reasonable, see \eqref{eq:explicit_constants_2_main_theorem}. That means $0<\eta<35$. We assume that $\delta_0 \approx 1-1.03^{-\eta}$, which means that the size of the good set decreases exponentially with $\eta$, see Fig. \ref{deltaX_final}. Finally, we can assume that $\xi \in [.7,1.5]$ and the bad set can be of size $1.5$. Fig. \ref{C_2forexample1} shows the constant $C_2$ under these assumptions, when $\eta$ varies from $2$ to $16$. We see a similar drop in $C_2$ from $.12$ to $.06$.  In this example, we see that the loss itself does seem to depend on $\bar{\beta}$ though the loss does not lead to tight bounds on the accuracy. Nevertheless, we still obtain important information on the size of the loss. We also see that $C_2$ might provide an explanation for why loss drops so quickly when $\bar{\beta}$ goes from $0$ to $.1$ and then plateaus. This is because $C_2$ itself plateaus when $\eta$ gets large, see Fig. \ref{C_2forexample1}.              

    \begin{figure}[h!]
       \label{acc_loss_DC_stability_1}  \begin{subfigure}[b]{0.4\textwidth}
         \centering
\includegraphics[width=1.4\textwidth]{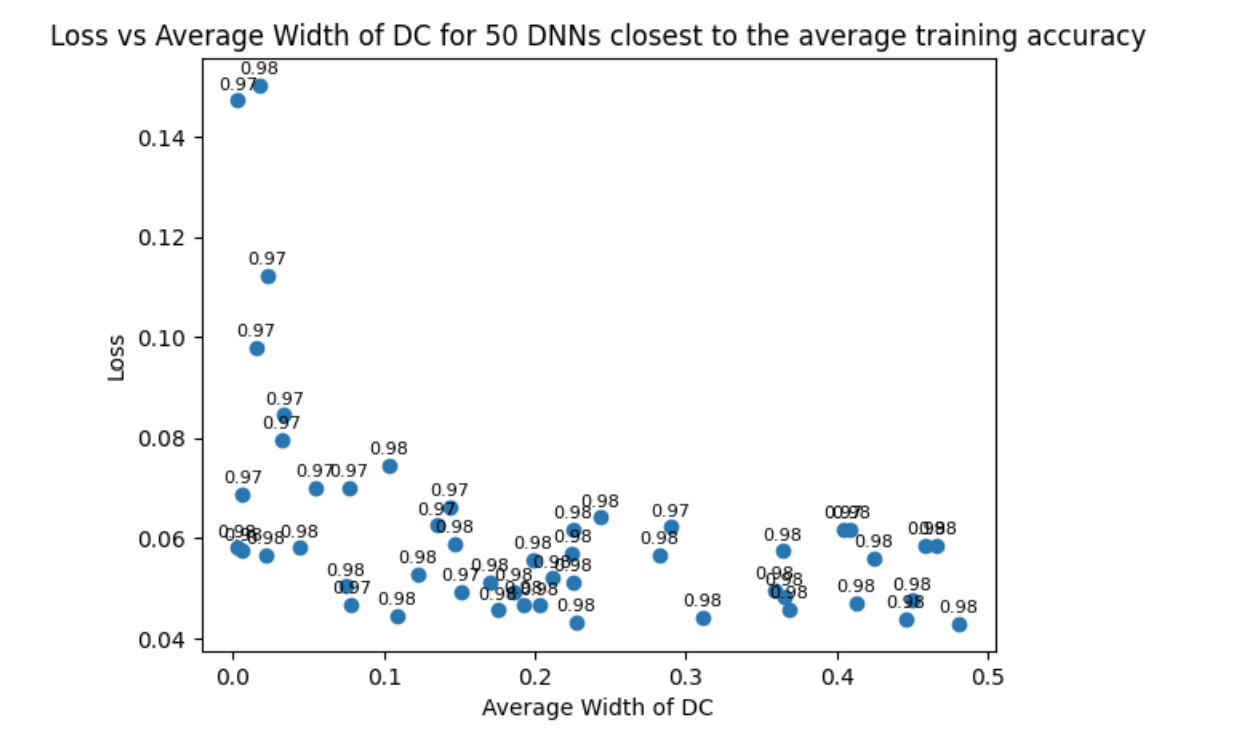}
         \caption{Training loss, represented on the y-axis, vs $\bar{\beta}$ (or average width of DC), represented on the x-axis. }    \label{Loss_vs_DC_stability_1}
     \end{subfigure}
     \hfill
     \begin{subfigure}[b]{0.4\textwidth}
         \centering    \includegraphics[width=1.1\textwidth]{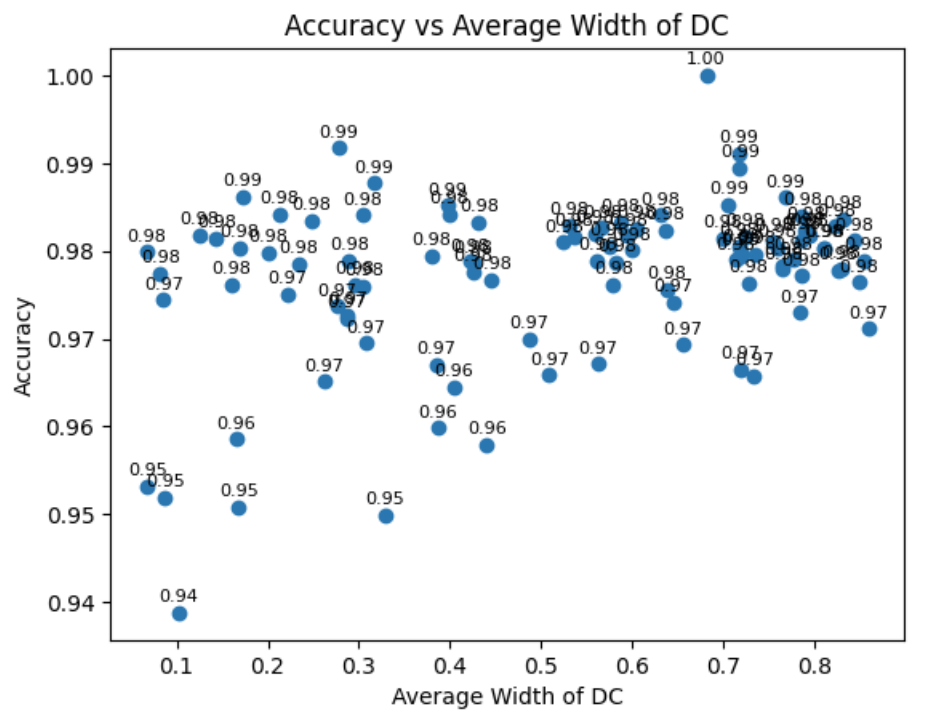}
         \caption{Training accuracy, represented on the y-axis, vs $\bar{\beta}$, represented on the x-axis.}
         \label{ACC_VS_CD_stability_1}
     \end{subfigure}
\caption{DNN training loss and training accuracy vs $\bar{\beta}$.The numbers above the blue points represent the accuracy of the DNNs.}

\end{figure}

	\begin{figure}	\includegraphics[scale=.55]{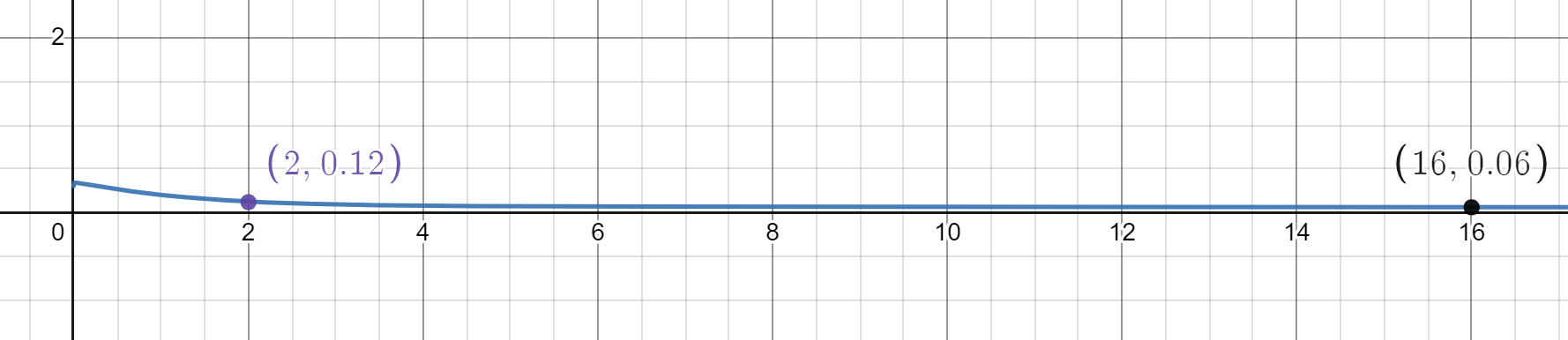}
			\caption{$C_2$, represented on the y-axis, vs $\eta$, represented on the x-axis.}
			\label{C_2forexample1}
		\end{figure}

\end{ex}

\begin{ex}
\label{DC_vs_loss_example_2}
    In this example, we take $\delta=1$, $\kappa=2$, $\sigma=.9$ and $\ell=.003$. We train $100$ different DNNs on $100$ different training sets, each having between $10,000$ to $20,000$ objects, and check the SUDC for all training sets using $50,000$ initial slabs and doubling them until the simplified uniform doubling condition is not satisfied. We train the DNNs for $5$ epochs, and each DNN has the same number of layers ($4$ hidden layers with $1000$ nodes in each layer). In Fig. \ref{ACC_VS_CD_3_stability_2}, we show the accuracy of the DNNs on the training set vs $\bar{\beta}$. Then, we find the average accuracy of the DNNs on the training set and the $50$ DNNs with accuracy closest to this average accuracy, and in Fig. \ref{Loss_vs_DC_3_stability_2} we graph the loss on the training set vs $\bar{\beta}$. Fig. \ref{Loss_vs_DC_3_stability_2} shows that even DNNs with similar accuracy will have lower loss if the data they trained on had more slabs that satisfied this simplified uniform doubling condition. Again, this means that those DNNs, with lower loss, will have greater stability of accuracy.  In this example, accuracy does not seem to depend on the simplified uniform doubling condition, though more rigorous analysis does show a correlation on similar problems, see Example \ref{new_example2} and Subsection \ref{example1_new}.    

    \begin{figure}[h!]
       \label{acc_loss_DC_3_stability_2}  \begin{subfigure}[b]{0.4\textwidth}
         \centering
\includegraphics[width=1.4\textwidth]{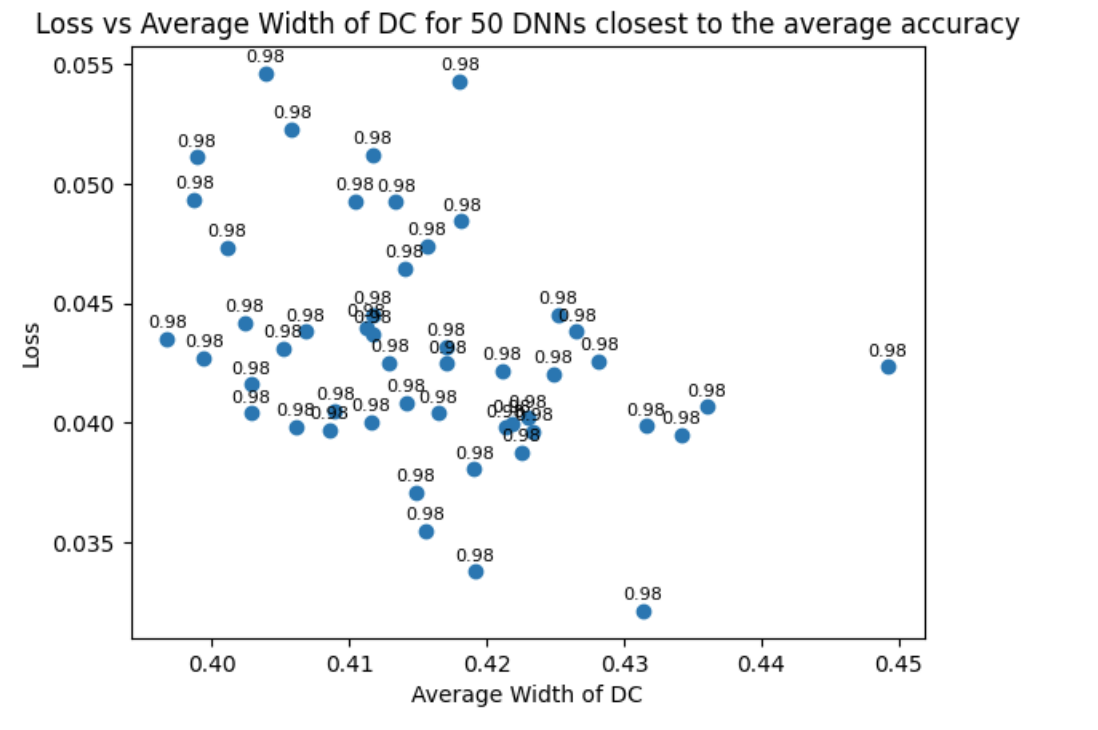}
         \caption{Training loss, represented on the y-axis, vs $\bar{\beta}$, represented on the x-axis.}    \label{Loss_vs_DC_3_stability_2}
     \end{subfigure}
     \hfill
     \begin{subfigure}[b]{0.4\textwidth}
         \centering    \includegraphics[width=1.4\textwidth]{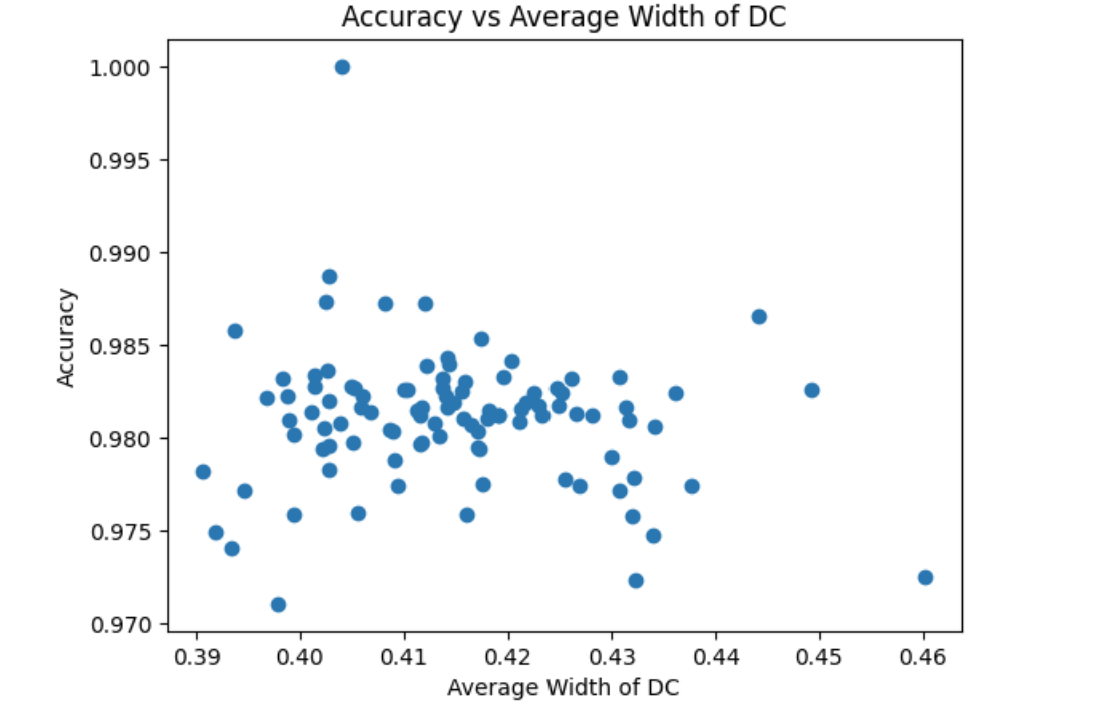}
         \caption{Training accuracy, represented on the y-axis vs $\bar{\beta}$, represented on the x-axis.}
         \label{ACC_VS_CD_3_stability_2}
     \end{subfigure}
\caption{DNN training loss and training accuracy vs $\bar{\beta}$. The numbers above the blue points represent the accuracy of the DNNs.}

\end{figure}

In Fig. \ref{points_vs_DC_1}  we show a graph of the number of points in the training set vs $\bar{\beta}$, to illustrate that there does not seem to be an immediate connection between the number of points in the training set and the SUDC. 

\end{ex}

\begin{ex}
\label{DC_vs_loss_example_3}
    In this example, we take $\delta=1$, $\kappa=2$, $\sigma=.9$ and $\ell=.003$. We train $100$ different DNNs on $100$ different data sets and check the simplified doubling condition for all data sets using $100,000$ initial slabs and doubling them until the SUDC is not satisfied. We train the DNNs for $5$ epochs, and each DNN has the same number of layers ($4$ hidden layers with $1000$ nodes in each layer). In Fig. \ref{ACC_VS_CD_stability_3}, we show the accuracy of the DNN on the test vs $\bar{\beta}$. Then, we find the average accuracy of the DNNs on the training set and the $50$ DNNs with accuracy (on the training set) closest to this average accuracy, and in Fig. \ref{Loss_vs_DC_stability_3} we graph the loss on the training set vs $\bar{\beta}$. Fig. \ref{Loss_vs_DC_stability_3} shows that even DNNs with similar accuracy will have lower loss if the data they trained on had more slabs that satisfied this simplified doubling condition. So again, we have greater stability of accuracy. There does not seem to be a good correlation between $\bar{\beta}$ and accuracy on the test set in this example, though more rigorous analysis does show a correlation on similar problems, see Example \ref{new_example2} and Subsection \ref{example1_new}.

    \begin{figure}[h!]
       \label{acc_loss_DC}  \begin{subfigure}[b]{0.4\textwidth}
         \centering
\includegraphics[width=1.4\textwidth]{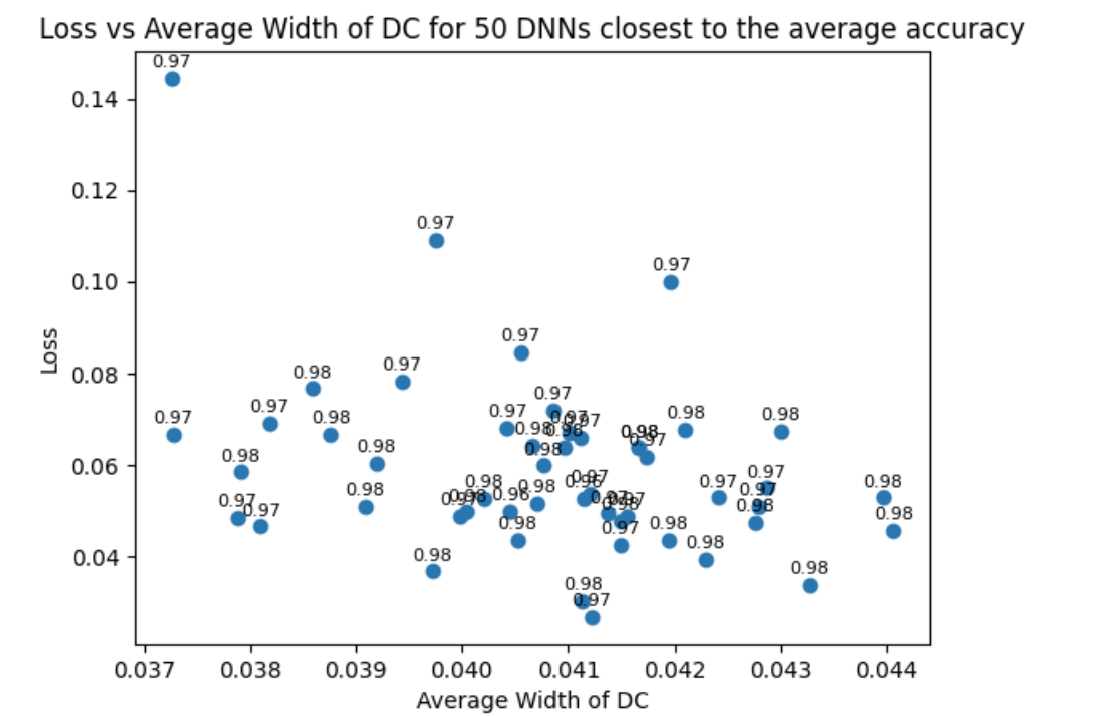}
         \caption{Loss on the training set, represented on the y-axis, vs $\bar{\beta}$, represented on the x-axis.}    \label{Loss_vs_DC_stability_3}
     \end{subfigure}
     \hfill
     \begin{subfigure}[b]{0.4\textwidth}
         \centering    \includegraphics[width=1.4\textwidth]{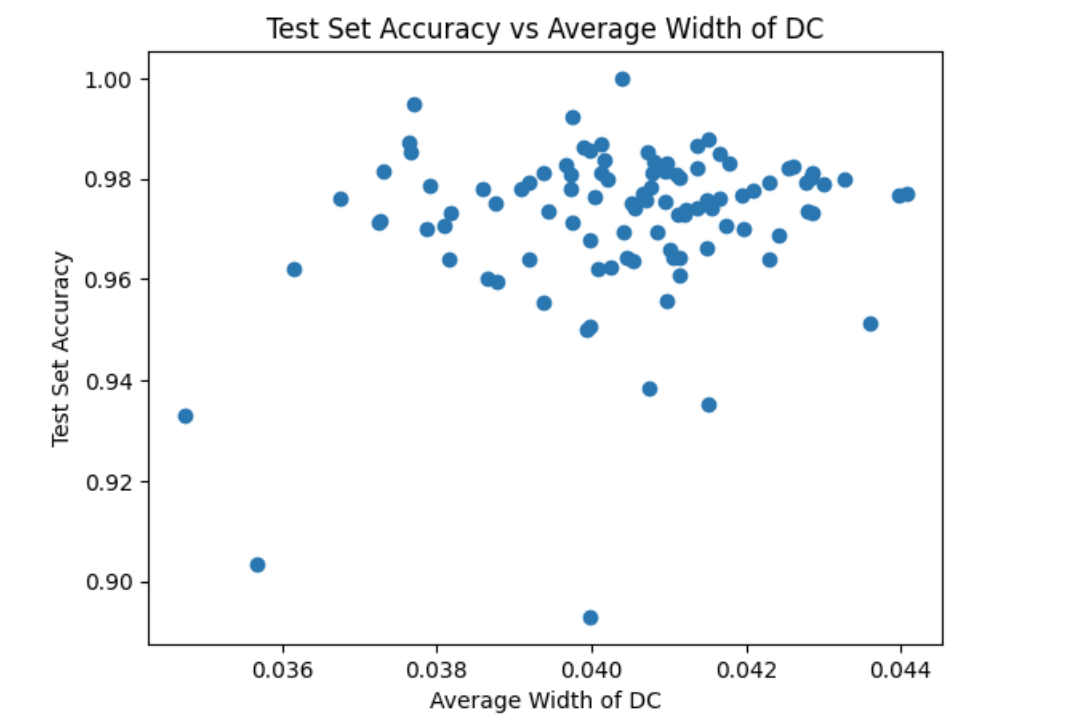}
         \caption{Accuracy on the test set, represented on the y-axis, vs $\bar{\beta}$, represented on the x-axis.}
         \label{ACC_VS_CD_stability_3}
     \end{subfigure}

\caption{DNN loss on training set and accuracy on test set vs $\Bar{\beta}$. The numbers above the blue points represent the accuracy of the DNNs.}

\end{figure}

Fig. \ref{numberofpoints_vs_DC_2} shows a graph of the number of points in the training set vs $\bar{\beta}$, to illustrate that there does not seem to be an immediate connection between the number of points in the training set and the simplified uniform doubling condition.

    \begin{figure}[h!]
       \label{acc_loss_DC_3}  \begin{subfigure}[b]{0.4\textwidth}
         \centering
\includegraphics[width=1.4\textwidth]{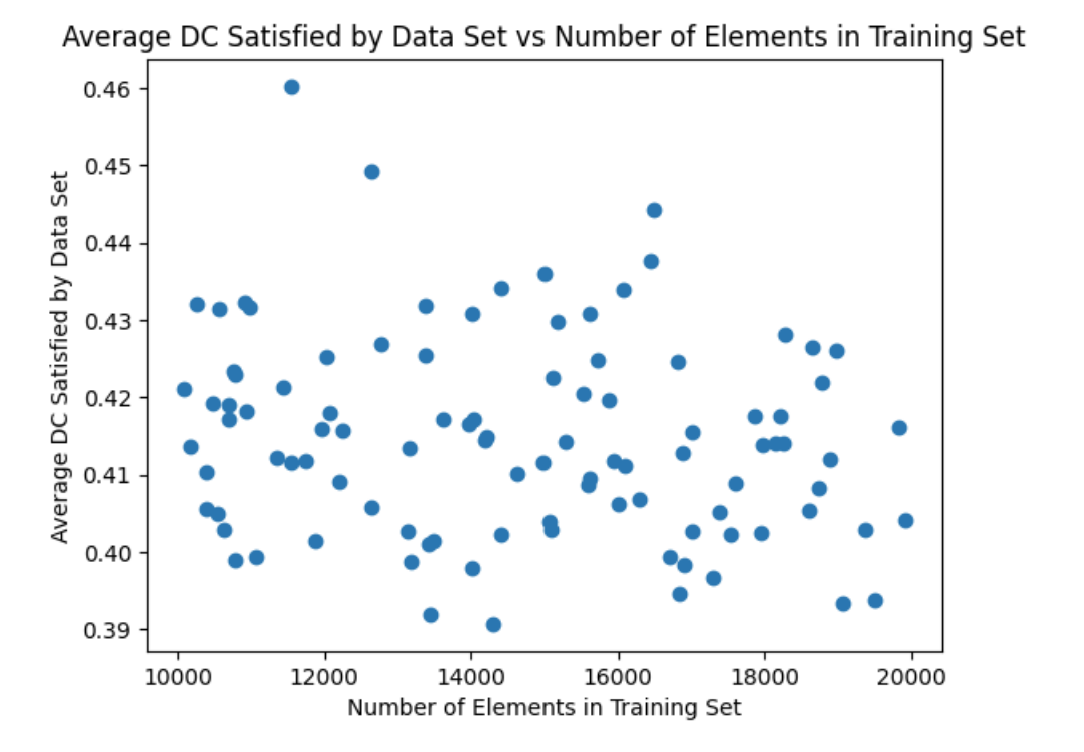}
         \caption{$\bar{\beta}$, represented on the y-axis, vs the number of elements in the training set for Example \ref{DC_vs_loss_example_2}, represented on the x-axis.}    \label{points_vs_DC_1}
     \end{subfigure}
     \hfill
     \begin{subfigure}[b]{0.4\textwidth}
         \centering    \includegraphics[width=1.4\textwidth]{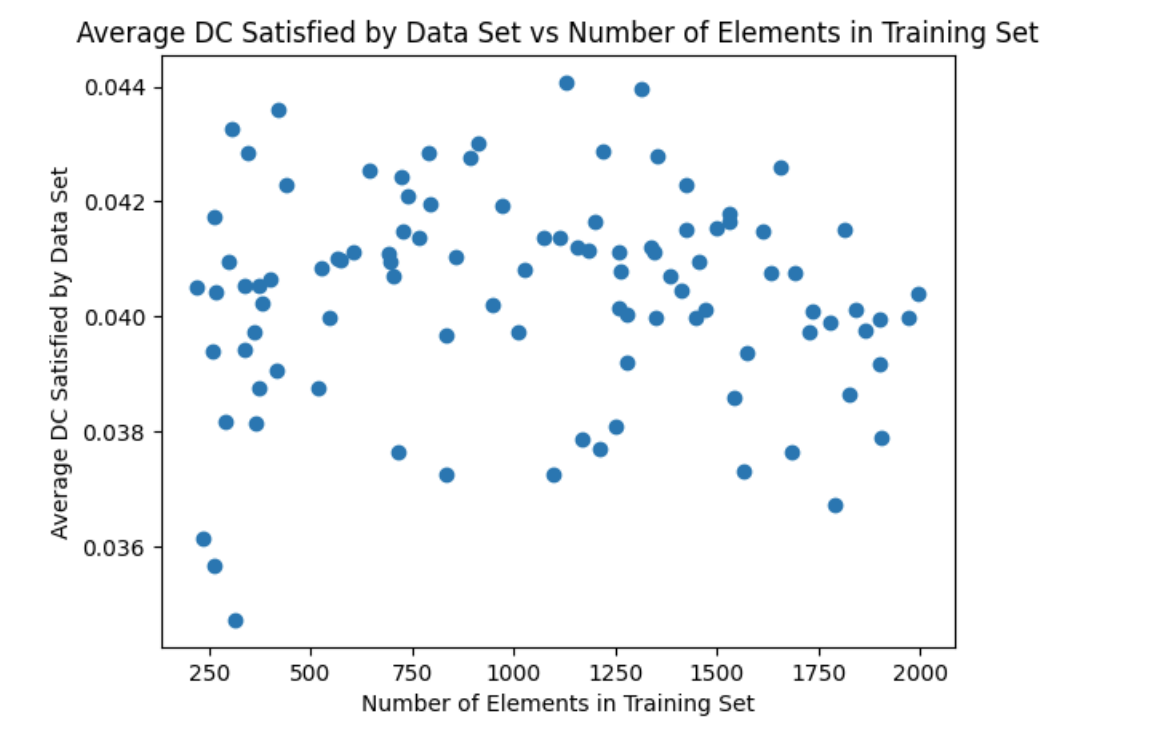}
         \caption{$\bar{\beta}$, represented on the y-axis, vs the number of elements in the training set for Example \ref{DC_vs_loss_example_3}, represented on the x-axis.}
         \label{numberofpoints_vs_DC_2}
     \end{subfigure}
\caption{Number of objects in data set vs $\bar{\beta}$.}
\end{figure}

\end{ex}

\begin{ex}
\label{new_example2}
    In the following example, we perform a similar analysis as the previous examples, but we fit the data to a simple linear regression model to see the underlying correlation between $\bar{\beta}$, loss, and accuracy on the test sets. Here, we use the Leaky ReLU activation function:
\begin{equation}
f(x) = 
\begin{cases} 
x & \text{if } x \geq 0 \\
0.1x & \text{if } x < 0. 
\end{cases}
\end{equation}

Note that while Theorem \ref{weakenedDC1} is given for DNNs with the absolute values activation function, we have that a similar theorem can be obtained for DNNs with the Leaky ReLU activation function, see Section \ref{otheractivationfunctions}.  

 In this example, we configure the parameters as follows:
\begin{itemize}
    \item Number of data sets we train on: $200$
    \item Epochs trained: $10$
    \item \( \kappa \): $2.1$
    \item $\sigma$: $0.9$
    \item Width of the original slab: $0.001$
    \item Number of original slabs: $700,000$
\end{itemize}

See Fig. \ref{acc_loss_DC_test_training} for the results of this numerical simulation. We see a positive correlation between $\bar{\beta}$ and accuracy on the test set. 

 \begin{figure}[h!]
    \centering
    
    \begin{subfigure}[b]{0.45\textwidth}
        \centering
        \includegraphics[width=\textwidth]{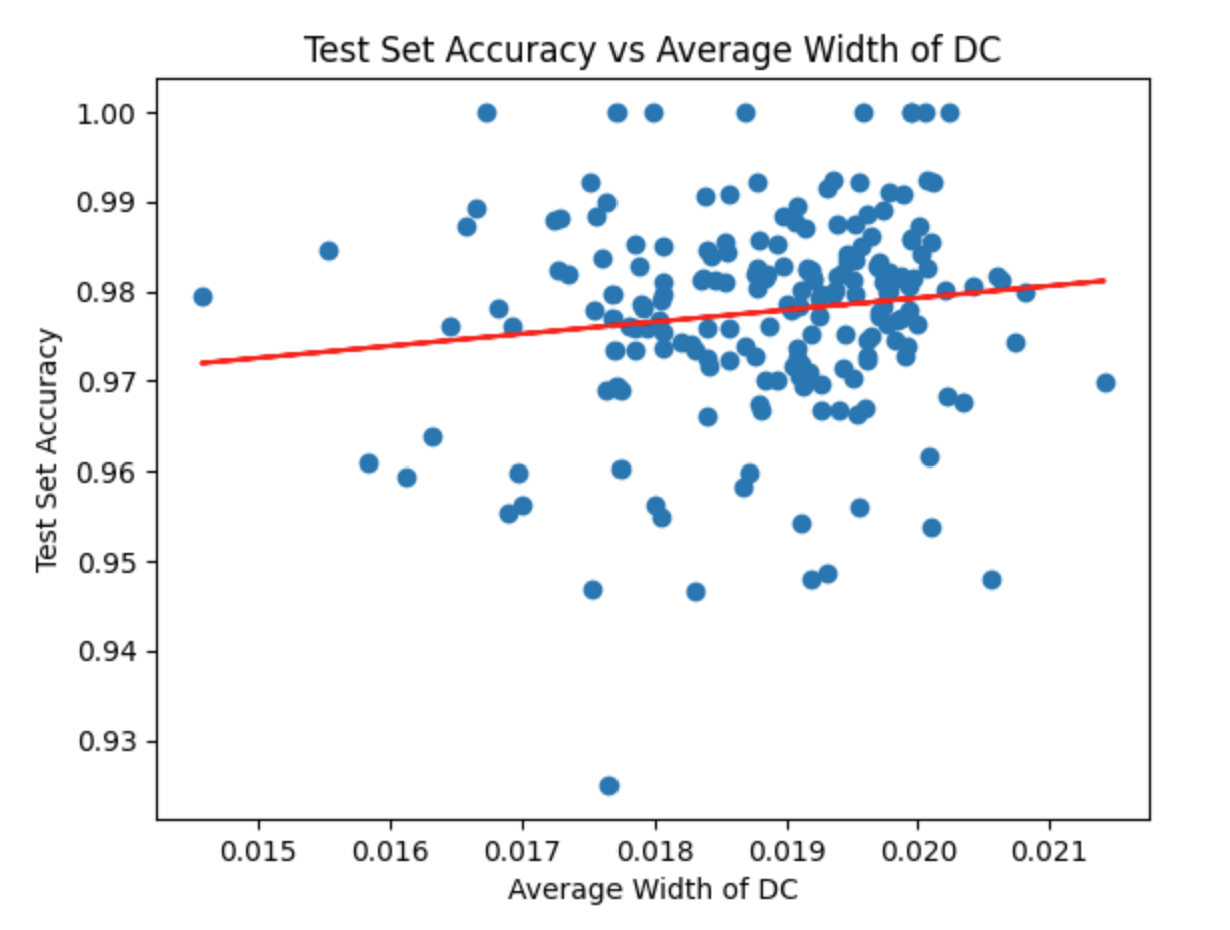}
        \caption{Accuracy on the test set, represented on the y-axis, vs $\bar{\beta}$, represented on the x-axis.}
        \label{acc_vs_DC_example2new}
    \end{subfigure}
    \hfill
    \begin{subfigure}[b]{0.48\textwidth}
        \centering
        \includegraphics[width=\textwidth]{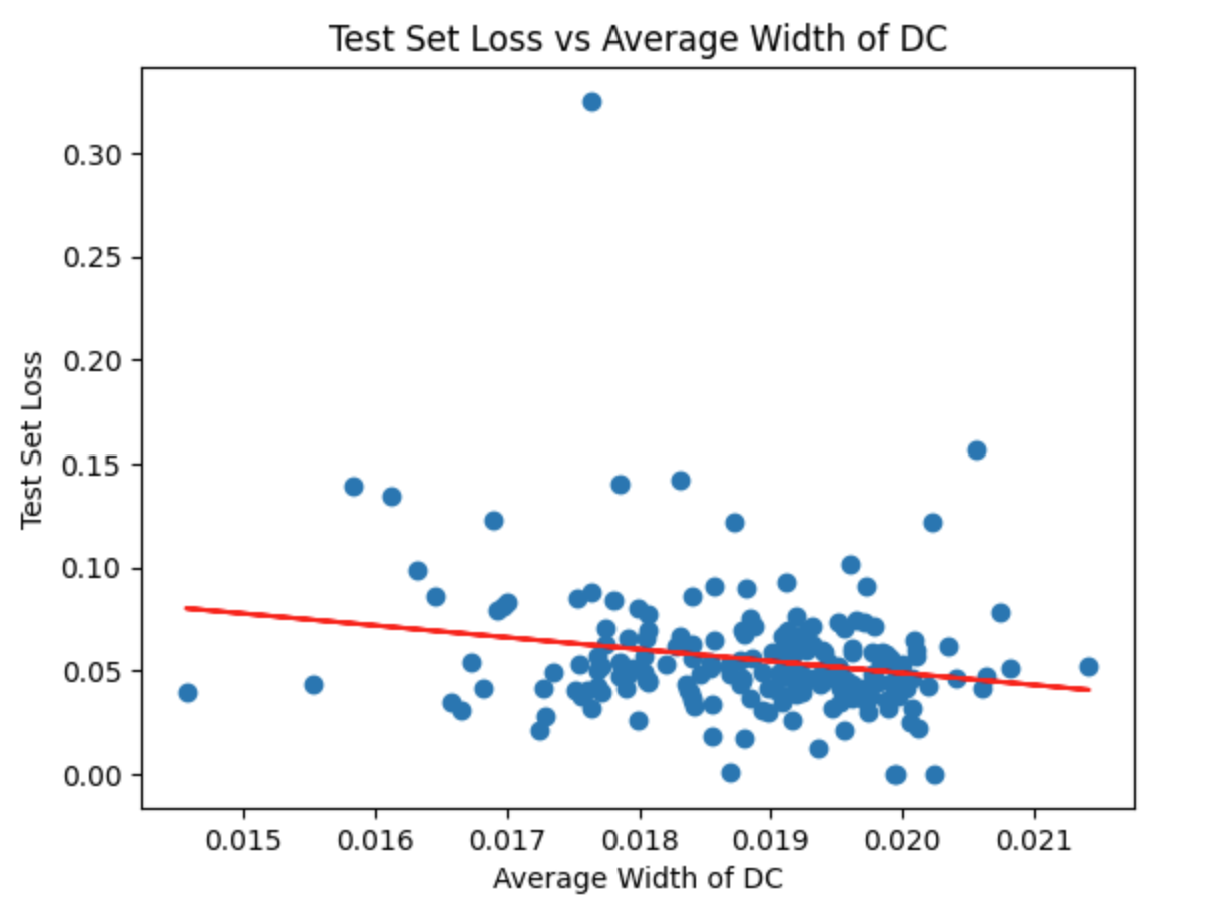}
        \caption{Loss on the test set, represented on the y-axis, vs $\bar{\beta}$, represented on the x-axis.}
        \label{ACC_VS_CD}
    \end{subfigure}

    \vspace{1em} 

    \begin{subfigure}[b]{0.45\textwidth}
        \centering
        \includegraphics[width=\textwidth]{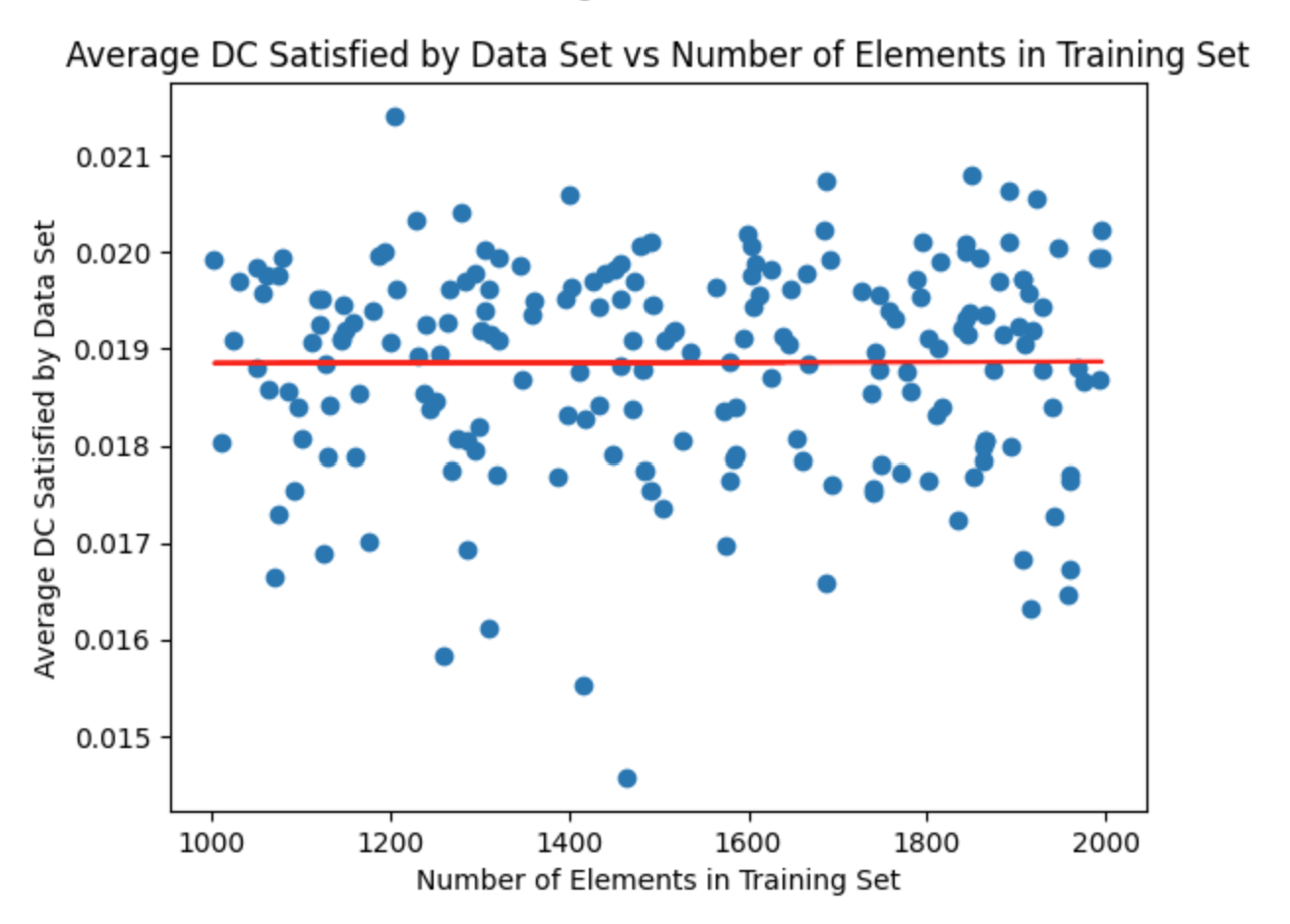}
        \caption{$\Bar{\beta}$ vs number of points in training set.}
        \label{LabelForFourthImage}
    \end{subfigure}

    \caption{DNN loss on test set and accuracy on test set vs $\Bar{\beta}$. }
    \label{acc_loss_DC_test_training}
\end{figure}

\end{ex}

\begin{ex}
\label{new_example3}
    In this example, we perform a similar analysis as the previous examples, again fitting the data to a simple linear regression model to see the underlying correlation between $\bar{\beta}$, loss, and accuracy on the test sets. In this example, we configure the parameters as follows:
\begin{itemize}
    \item Number of data sets we train on: $400$
    \item Epochs trained: $10$
    \item \( \kappa \): $2.1$
    \item $\sigma$: $1$
    \item Width of the original slab: $0.0001$
    \item Number of original slabs: $2,000,000$
\end{itemize}

See Fig. \ref{acc_loss_DC_test_training_3} for the results of this numerical simulation. Again, we see a strong correlation between the SUDC and the accuracy and loss on the test set. 

 \begin{figure}[h!]
    \centering
    
    \begin{subfigure}[b]{0.44\textwidth}
        \centering
        \includegraphics[width=\textwidth]{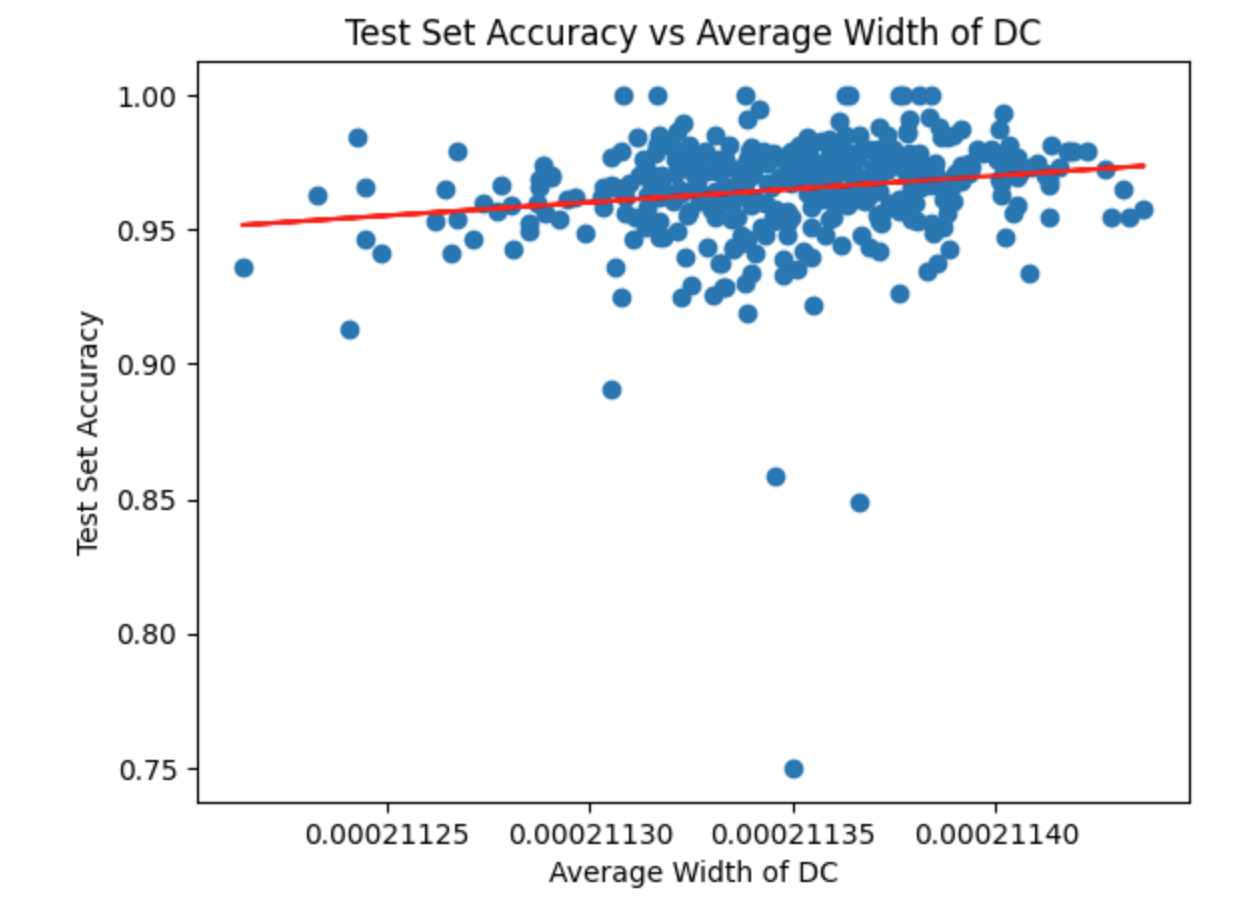}
        \caption{Accuracy on the test set, represented on the y-axis, vs $\bar{\beta}$, represented on the x-axis.}
        \label{acc_vs_DC_example2new}
    \end{subfigure}
    \hfill
    \begin{subfigure}[b]{0.53\textwidth}
        \centering
        \includegraphics[width=\textwidth]{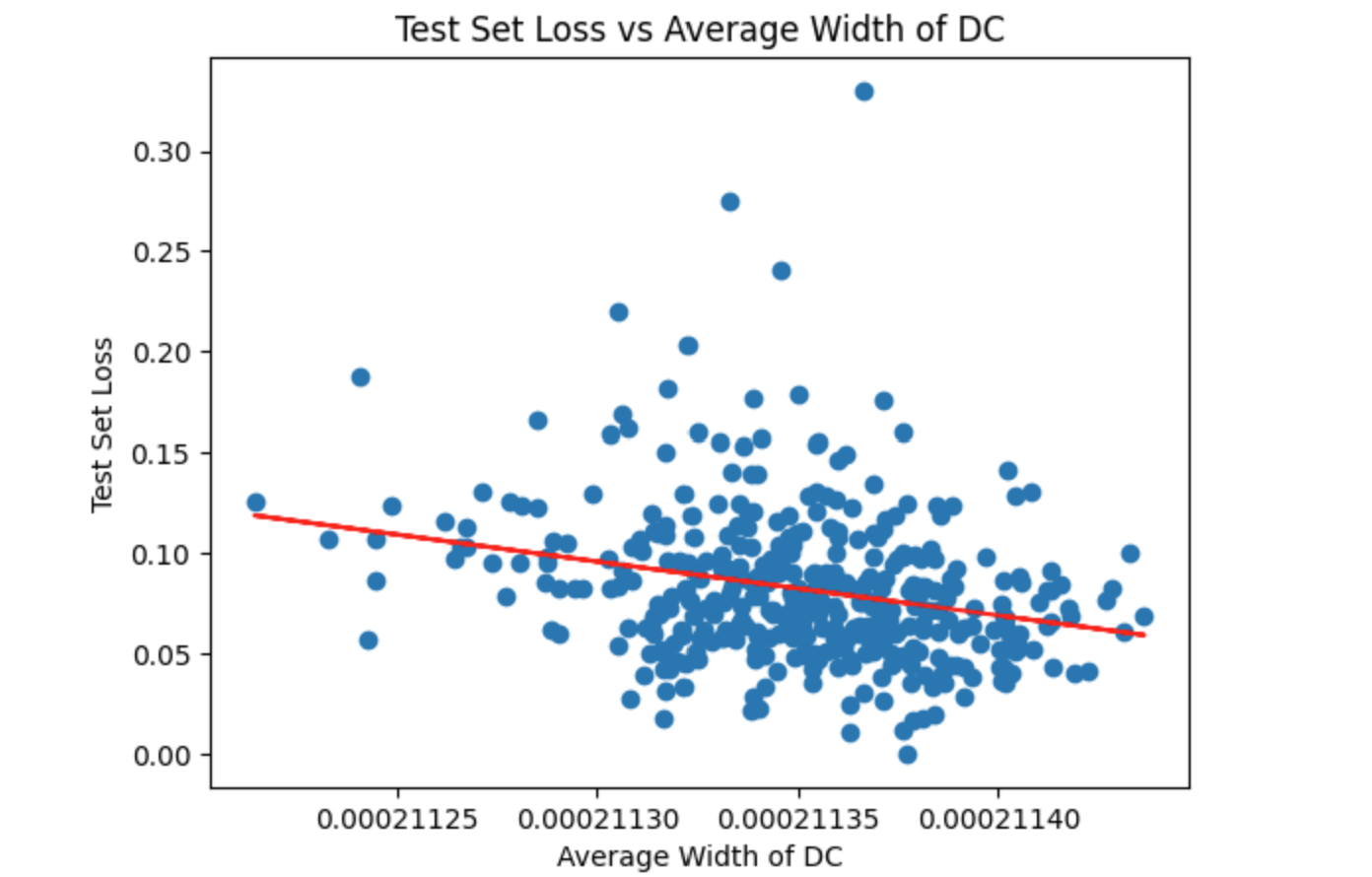}
        \caption{Loss on the test set, represented on the y-axis, vs $\bar{\beta}$, represented on the x-axis.}
        \label{ACC_VS_CD_3}
    \end{subfigure}

    \vspace{1em} 

    \begin{subfigure}[b]{0.45\textwidth}
        \centering
        \includegraphics[width=\textwidth]{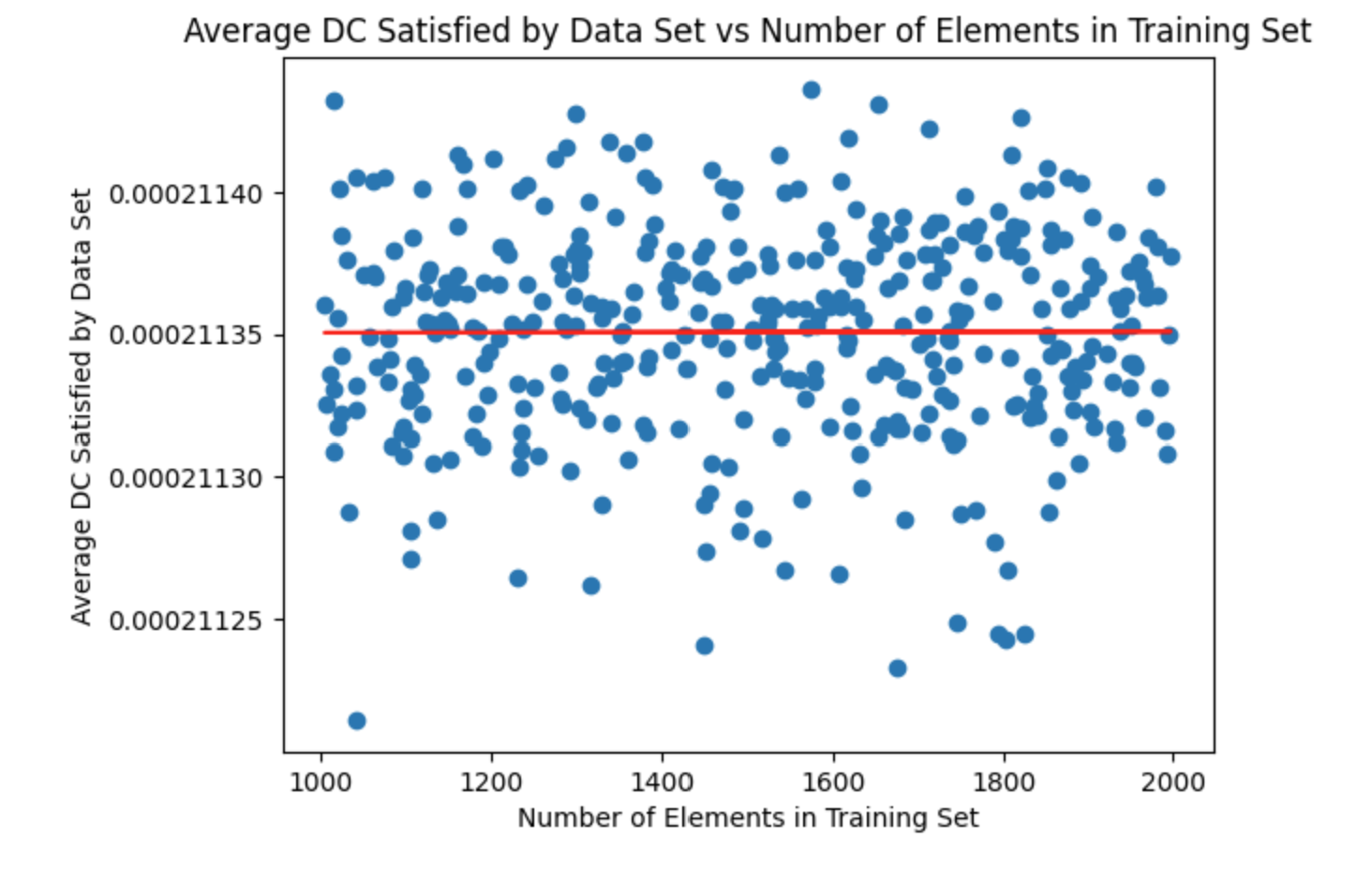}
        \caption{$\Bar{\beta}$ vs number of points in training set.}
        \label{LabelForFourthImage_3}
    \end{subfigure}

    \caption{DNN loss on test set and accuracy on test set vs $\Bar{\beta}$. }
    \label{acc_loss_DC_test_training_3}
\end{figure}

\end{ex}

\subsection{Numerics for a different classification problem}
\label{example1_new}

Given an integer \( M \), let's first define a sequence of random functions. Each function is represented as:

\[
f_i(x) = \cos\left(\pi a_i x + \theta_1^i\right) \times \sin\left(\pi b_i x + \theta_2^i\right).
\]

Where the frequencies \( a_i \) and \( b_i \) are randomly sampled from the discrete set \(\left\{ \frac{1}{5}, \frac{2}{5}, \frac{3}{5}, \frac{4}{5} \right\}\) and the phases \( \theta_1^i \) and \( \theta_2^i \) are sampled from a uniform distribution over the interval \( [0, 2\pi] \). The vector function $f(x)$ has a total of $M$ components $f(x)_i$.

Furthermore, we define a random matrix \( B \in \mathbb{R}^{M \times M} \) where each entry of \( B \) is sampled uniformly from the interval \( [0, 5] \). This matrix interacts with our random functions as follows:

\[
g(x) = B f(x),
\]
where the product represents a matrix-vector multiplication, and \( f(x) \) is a vector of the evaluated random functions at a given \( x \).

 We take a series of \( N \) points  \( x_j \) which are sampled uniformly from the interval \( [-10, 10] \). The classification problem arises when we generate two classes of data points around the curve defined by \( g_1 \), the first component of $g$. For each $j$, we obtain a point in $\R^2$ from one of the two classes, with both classes having an equal number of elements. The first class, represented in red, is generated by taking \( g_1(x_j) \) and subtracting a random positive number from it to obtain $z_j$, whereas the second class, in blue, is obtained by adding a random positive number to the curve.

Our dataset \( D \) is then composed of these points:

\[
D = \{ (x_j, z_j) \}_{j=1}^{N}.
\]

 We use this to generate both a training and test data set with $N$ elements in total. Additionally, Gaussian noise is added to both classes of the training set to scatter the points, see Fig. \ref{fig:sample_data}.

The task is to discriminate between these two classes, recognizing the underlying structure provided by the function \( g_1 \) and the added perturbations. We use a DNN with the same architecture as in the previous examples (see Example \ref{new_example2}) to train on the training set in order to classify the red and blue points. In this example, we randomly create a training set of between $40,000$ to $60,000$ data points and a test set composed of a similar size of data points.

In this example, we configure the parameters as follows:
\begin{itemize}
    \item Number of data sets we train on: $200$
    \item Epochs trained: $10$
    \item \( \kappa \): $2.1$
    \item $\sigma$: $.9$
    \item Width of the original slab: $0.001$
    \item Number of original slabs: $700,000$
\end{itemize}

\begin{figure}[h]
    \centering
    \includegraphics[width=1\textwidth]{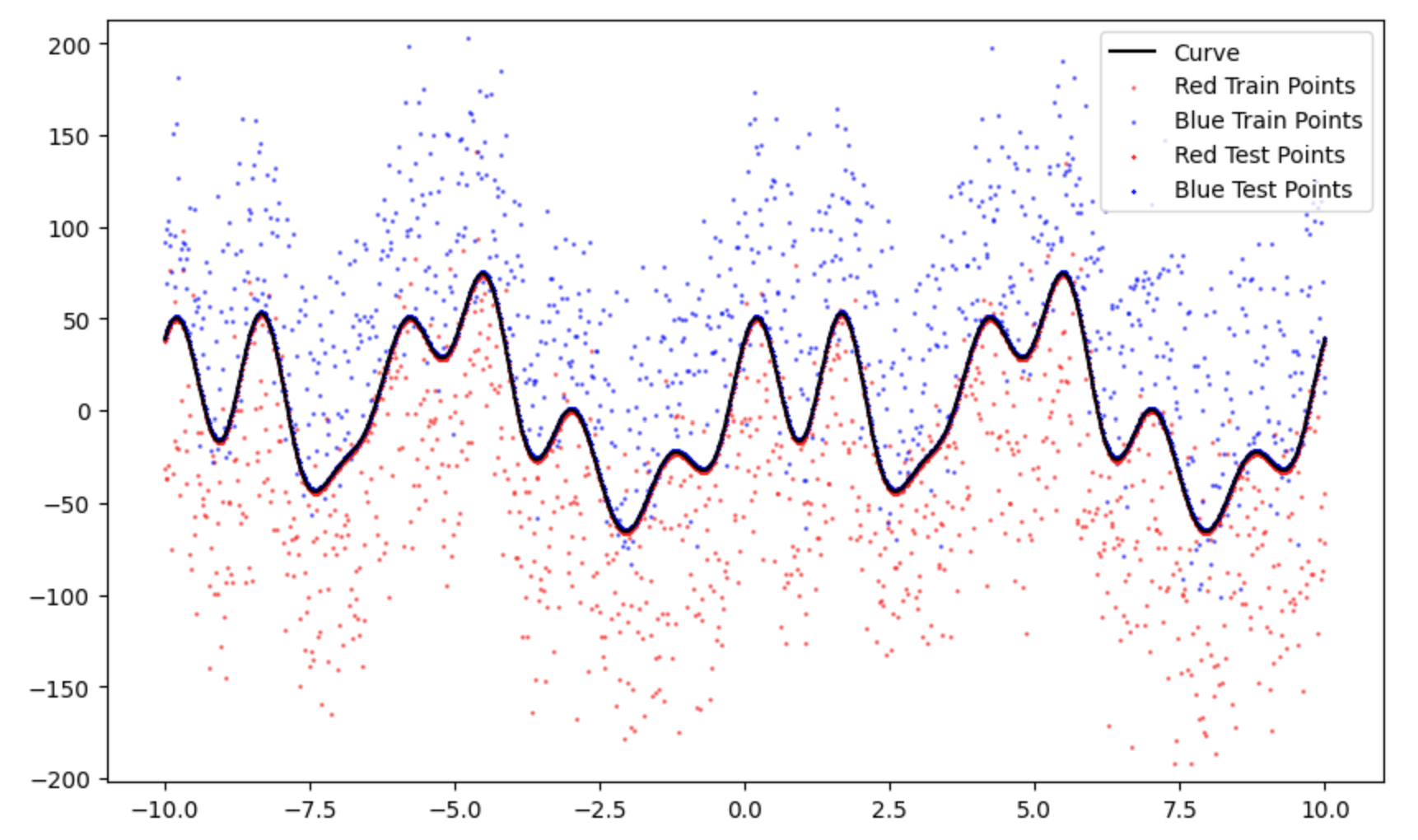}
    \caption{Illustrative example of the generated data. The black curve showcases the function \( g_1(x) \), the red points indicate the first class, and the blue points signify the second class.}
    \label{fig:sample_data}
\end{figure}

 For the numerical results, see Fig. \ref{acc_loss_DC_test_training_example1new}. The classification problem here is more difficult than the previous problem, and we only train for $10$ epoch, leading to low accuracy (a little above $50\%$ for most of the DNNs trained, which is slightly better than pure probability given that this is a classification problem with two classes). Nevertheless, we still see a distinct linear correlation between the loss and accuracy on the test set and $\Bar{\beta}$. Again, we confirmed that no linear correlation existed between the number of data points being trained on and $\Bar{\beta}$.

In future work, we will use a new random matrix theory pruning approach (see \cite{shmalo2023deep, berlyand2023enhancing}) to achieve better accuracy for this simple classification problem, and we will perform this simple SUDC analysis to see if the correlations between $\Bar{\beta}$ and accuracy on the test set still persist in this more complex situation.

 \begin{figure}[h!]
    \centering
    
    \begin{subfigure}[b]{0.45\textwidth}
        \centering
        \includegraphics[width=\textwidth]{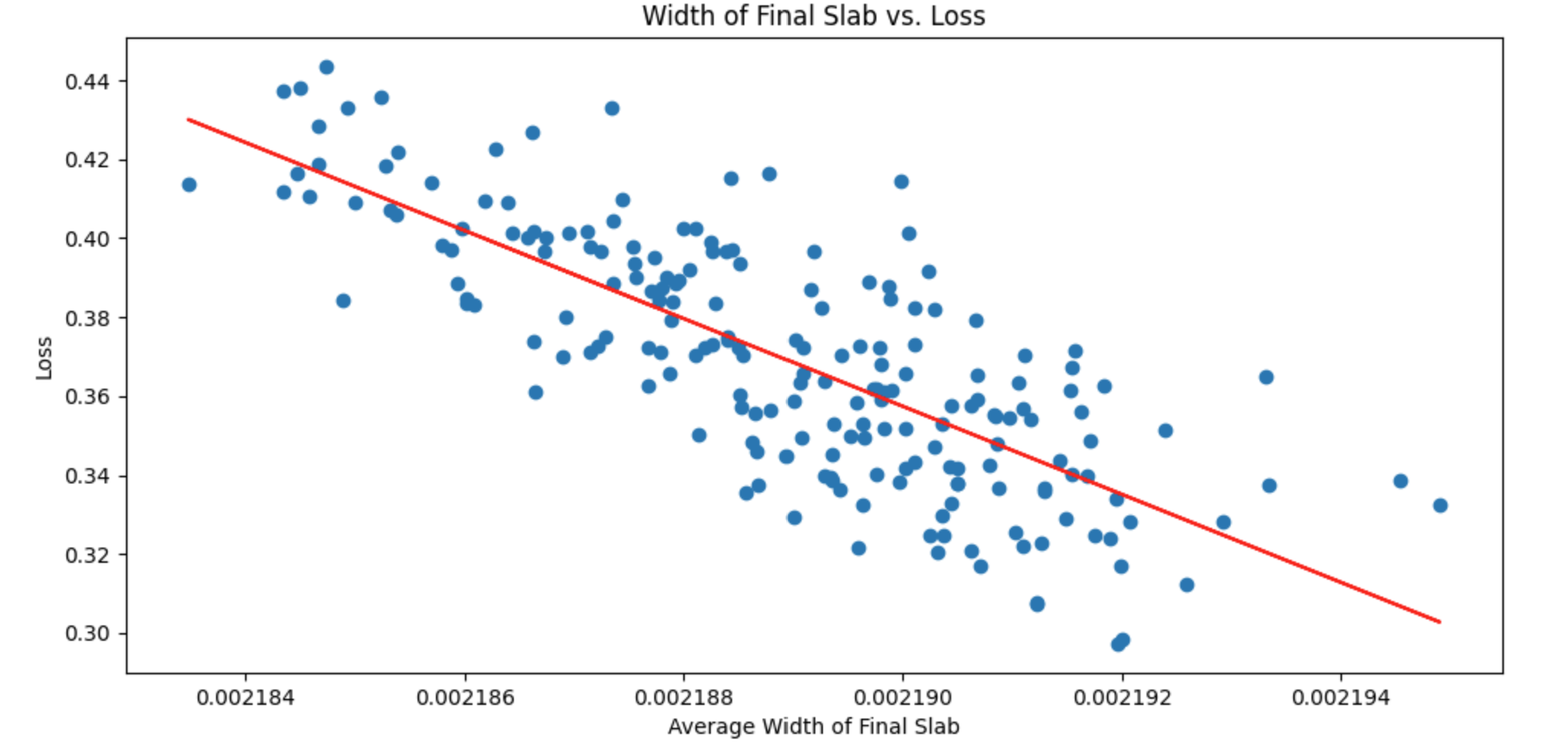}
        \caption{Loss on the training set, represented on the y-axis, vs $\bar{\beta}$, represented on the x-axis.}
        \label{loss_vs_DC_example1new}
    \end{subfigure}
    \hfill
    \begin{subfigure}[b]{0.48\textwidth}
        \centering
        \includegraphics[width=\textwidth]{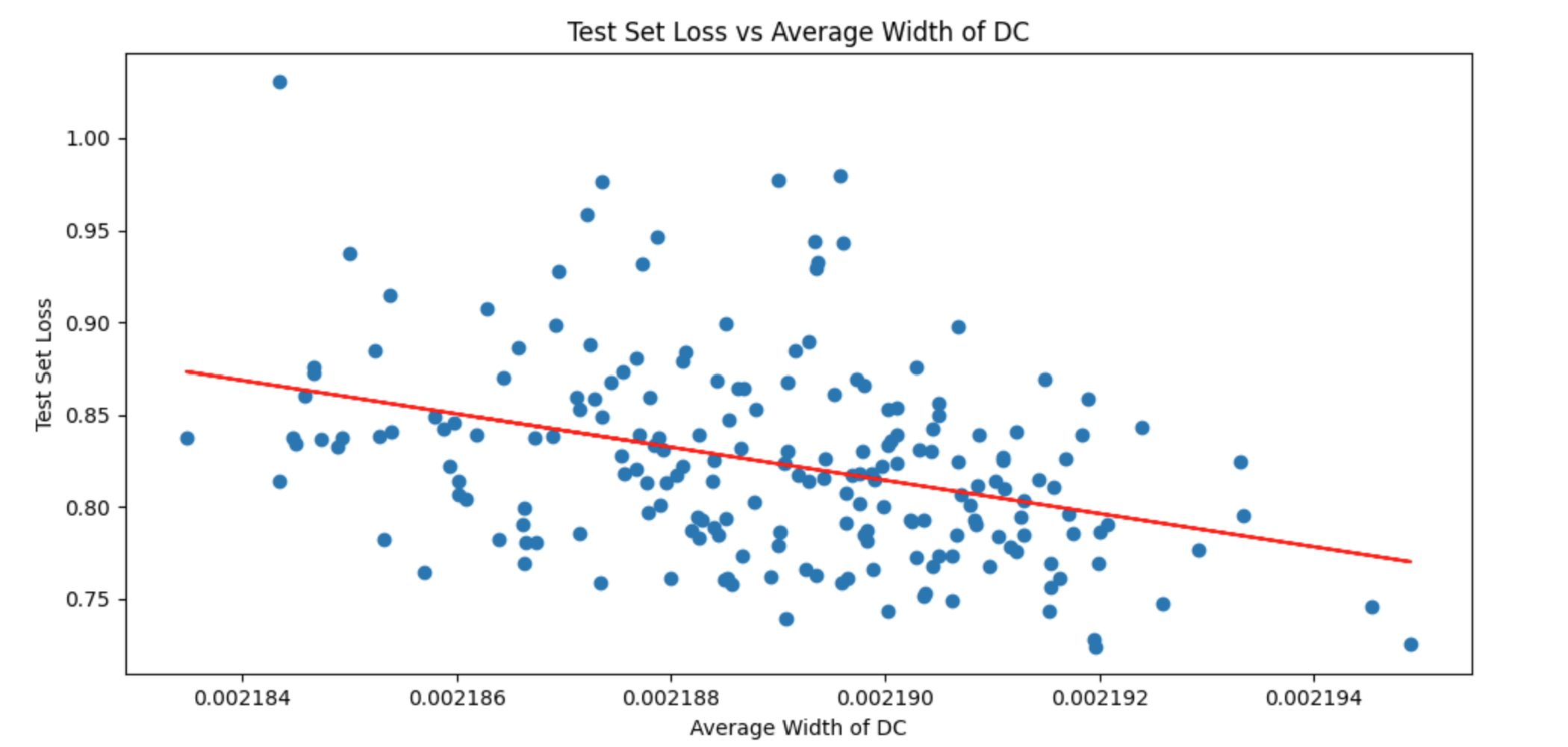}
        \caption{Loss on the test set, represented on the y-axis, vs $\bar{\beta}$, represented on the x-axis.}
        \label{DC_vs_test_loss_new_example1}
    \end{subfigure}

    \vspace{1em} 

    \begin{subfigure}[b]{0.45\textwidth}
        \centering
        \includegraphics[width=\textwidth]{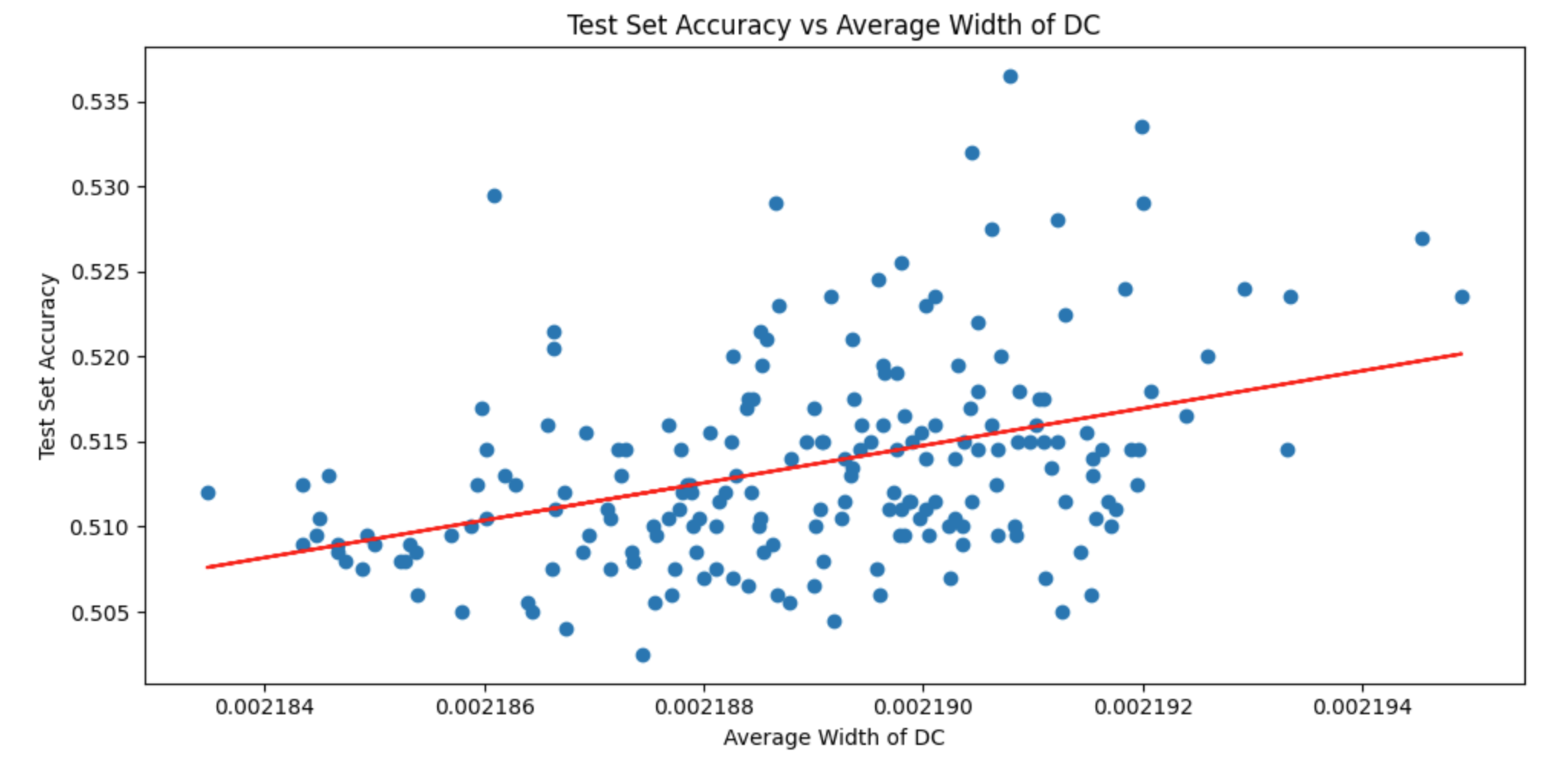}
        \caption{Accuracy on test set vs $\Bar{\beta}$.}
        \label{LabelForFourthImage}
    \end{subfigure}

    \caption{DNN loss on test set and accuracy on test set vs $\Bar{\beta}$. }
    \label{acc_loss_DC_test_training_example1new}
\end{figure}

\subsection{An application of the SUDC}

In the field of machine learning, data augmentation has emerged as a powerful technique to improve the generalization and robustness of deep neural networks. By introducing variations to the original training data, models are encouraged to recognize patterns that are invariant to these modifications, see \cite{goodfellow2014explaining, szegedy2013intriguing, zheng2016improving,thulasidasan2019mixup}. 

The essence of data augmentation is the creation of modified versions of the original dataset. Formally, given a dataset $T = \{(x_i, y_i)\}$, data augmentation yields a new dataset $T' = \{(x_i', y_i')\}$ where the transformations applied to produce $x_i'$ and $y_i'$ are determined by the chosen augmentation strategy. Training typically proceeds either on $T'$ or on an extended dataset, often the union of the original and augmented data, $T \cup T'$.

While numerous studies have underscored the efficacy of various data augmentation strategies in bolstering model robustness  the challenge remains in choosing the right augmentation technique. Factors such as the type and extent of augmentation can significantly influence model performance.

A potential approach to this challenge lies in the utilization of the SUDC. For different augmented training sets derived from $T$, the SUDC can provide a measure, $\bar{\beta}$, which may correlate with the training and generalization capabilities of a model trained on such data. By comparing $\bar{\beta}$ for different augmentation strategies, one could potentially infer which method might be most beneficial. For instance, a higher $\bar{\beta}$ might indicate a more favorable training set for a given problem. Harnessing this insight, one could optimize the augmentation process to yield training datasets that promise better test-time performance. A more exhaustive exploration of this application will be the focus of future work.

\section{Previous results in this direction.}
\label{p_results}

It is useful to review the first results given in this direction. In \cite{LJS} it was shown that small clusters in $\delta X$ lead to instability of accuracy, and so a doubling condition on $\delta X$ was introduced which prevents the existence of small clusters.
    	
    	The following is the doubling condition on $\delta X$:

 \begin{defn}\label{doubling_condition_DX}
The set $\{\delta X(s,\alpha(t_0)):s\in T\}$ satisfies \emph{the doubling condition}  at time $t_0$ if there exists a mass $m_0>0$ and constants $\kappa>1$, and $\delta,\sigma, \beta>0$ such that for $x=0$, and for all intervals $I\subset\mathbb R$ centered at $x=0$ with $|I|\leq \beta$,
		\begin{equation}\label{eq:cond_B}
		\mu\left(\{s\in T:\delta X(s,\alpha(t_0))\in\kappa I\}\right)\geq\min\Big\{\delta,\;\max\big\{m_0,\;(1+\sigma)\,\mu\left(\{s\in T:\delta X(s,\alpha(t_0))\in I\}\right)\big\}-m_0\Big\},
		\end{equation}
		where $\kappa I=I_{\kappa}$ is the interval with the same center as $I$ but whose width is multiplied by $\kappa$.
	
\end{defn}

Essentially, this doubling condition requires that 

\begin{equation}\label{eq:cond_B}
		\mu\left(\{s\in T:\delta X(s,\alpha(t_0))\in\kappa I\}\right)\geq (1+\sigma)\,\mu\left(\{s\in T:\delta X(s,\alpha(t_0))\in I\}\right),
		\end{equation}
		for all intervals $I$ centered at $x=0$. This means that as we double the interval $I$ by the constant $\kappa$ (which is therefore referred to as the doubling constant), we must capture more data than what was originally in $I$.
		
		In Fig. \ref{NSDC} we see that $\delta X$ satisfies the doubling condition at $x=0$, however, it does not satisfy the doubling condition for the intervals $I$ and $\kappa I$ shown in the figure. This is because as we double $I$ by the constant $\kappa$ to obtain $\kappa I$ we don't capture any more of the data.    

	\begin{figure}
			\includegraphics[scale=.7]{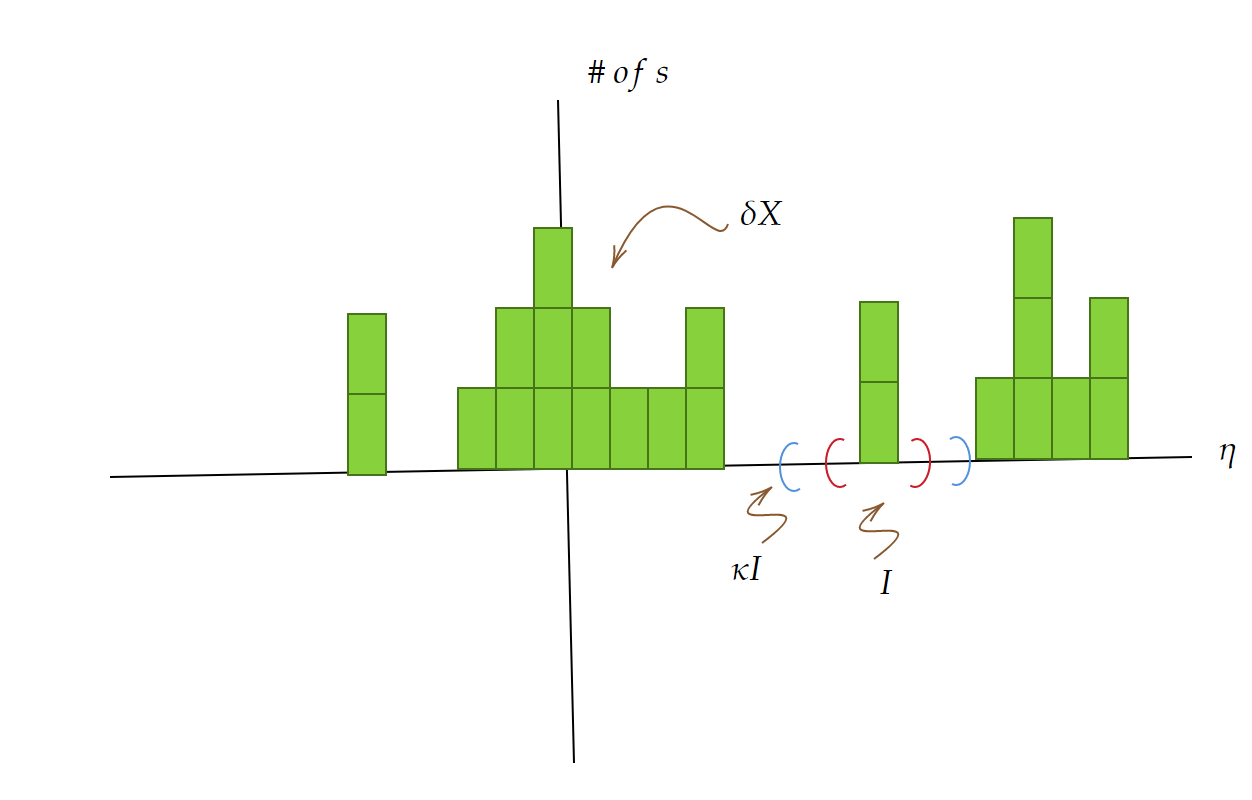}
			\caption{An example of a histogram of $\delta X(T)$ with slabs (in this case intervals) $\kappa I$ and $I$ not satisfying the doubling condition.}
			\label{NSDC}
		\end{figure}
		
We call $m_0$ the residual mass. If $(1+\sigma)\mu(I)\leq m_0$ the doubling condition is automatically satisfied, and so we can think of $m_0$ as the smallest mass for which the doubling condition must be satisfied.

This doubling condition on $\delta X$ provides for the stability result given in Theorem \ref{thm:stabilityB}, however, one can only verify the doubling condition after training has commenced (in order to obtain $\delta X$). Thus, one must know the architecture of the DNN and its parameters to obtain $\delta X$ and verify whether the DNN has the stability of accuracy during training. As we will see, a doubling condition on the training set directly can also be used to ensure stability.

	 We now state the stability of accuracy result with the doubling condition on $\delta X$. 
  
  \begin{thm}[Berlyand, Jabin, Safsten, 2021] \label{thm:stabilityB}
		Assume that the set $\{\delta X(s,\alpha(t_0)):s\in T\} \subset \R^1$ satisfies the doubling condition given in Def. \ref{doubling_condition_DX} with constants $\beta, \sigma, \kappa, m_0, \delta$ at point $x=0$ at some time $t_0$.
		Then for every $\varepsilon>0$ there exists a constant $C_1=C_1(\gamma,\eta, K ,\kappa,\sigma)$ such that if
		\begin{equation}\label{eq:explicit_constants_2}
		m_0\leq \frac{\eps}{2C_1},\quad \kappa(\log \frac{1}{\eps}+C_1)\leq \eta\leq \frac{\beta}{2},\quad \delta_0\leq \frac{\eps}{C_1}\eta^\gamma,
		\end{equation}
		  and if the good and bad sets at $t=t_0$ satisfy
				\begin{equation}\label{good_set_big_and_bad_set_empty_B_1}
				\mu(G_\eta(t_0))>1-\delta_0\quad\text{and}\quad B_{-1}(t_0)=\emptyset,
				\end{equation}
				then for all $t\geq t_0$,
				\begin{equation}\label{eq:stability_B}
				\text{acc}(t)=\mu(G_0(t))\geq 1-\varepsilon.
				\end{equation}
				and
		\begin{equation}\label{eq:stability_generalized_2}
		\mu(G_{\eta^*}(t))\geq 1-2\,\log 2\,\eps\,e^{\eta^*},
		\end{equation}
		for all $\eta^\ast>0$.
	\end{thm}
	
We take any $\gamma$ such that 
\begin{equation}\label{gamma}
    0<\gamma\leq\min\{1,\log(1+\sigma)/\log\kappa\}.
\end{equation}

Recall, $K$ is the number of classes in the DNN. 	\begin{remark}\label{constant}
	Here \begin{equation}\label{constantC_1}
C_1=\max \{6\eta^\gamma, \;\frac{1}{\log(2)}\kappa 6(K-1)(\sigma+1)a_p\,(e+2\kappa^\gamma)\,\}.
	\end{equation}
	with $a_j=\sum_{k=1}^{j}(1+\sigma)^{k}$ and $p=\lfloor\log(\eta)/\log\kappa\rfloor$. Already for $K=10$, $\sigma=1$ and $\kappa=2$ we have that $C_1$ is at least $800$ (and so $\eta$ is large and $\mu(G_\eta(t_0))$ is probably small). In Section \ref{main_result_1}, we will address this difficulty by finding a smaller constant $C_3$ and changing  the three conditions  given in (\ref{eq:explicit_constants_2}). 
	
	Note that $\eta>C_1$ and so when $\gamma=1$ (which depends on the doubling condition via the constants $\sigma$ and $\kappa$) we have that $\delta_0>\epsilon$. This makes it easier to satisfy condition \ref{good_set_big_and_bad_set_empty_B_1}. However only if $\eta$ can remain small.      
	\end{remark}
	
In this work, a new version of Theorem \ref{thm:stabilityB} will be proved, one which will be used to prove the main result given in Theorem \ref{weakenedDC1}. This theorem uses a uniform doubling condition on $\delta X$ (see Def. \ref{eq:uniform_doubling_condtion_deltaX}) instead of the doubling condition. Changing the doubling condition to the uniform doubling condition allows us to find a smaller constant than $C_1$ and the uniform doubling condition can, at least in theory, be checked numerically  giving a more practical stability of accuracy result. The new version of Theorem \ref{thm:stabilityB} is  given in  Proposition \ref{lemmadoublingconditionondeltaX}.

\subsection{Stability of accuracy via uniform doubling condtion on $\delta X(T,\alpha)$}
\label{DC_delta_X}

We now present another stability of accuracy result which complements Theorem \ref{weakenedDC1}.  First, we define a uniform doubling condition on $\delta X$:

\begin{defn}\label{eq:uniform_doubling_condtion_deltaX}
The set $\{\delta X(s,\alpha(t_0)):s\in T\}$ satisfies \emph{the uniform doubling condition} at time $t_0$ if there exists constants $\kappa>1$, and $\delta,\sigma, \ell, \beta>0$ such that for $x=0$, and for all intervals $I\subset\mathbb R$ around $x=0$ with $ |I|=\ell\kappa^i\leq \beta$ with $i \in \N$,
		\begin{equation}
		\mu\left(\{s\in T:\delta X(s,\alpha(t_0))\in\kappa I\}\right)\geq\min\{\delta,\;(1+\sigma)\,\mu\left(\{s\in T:\delta X(s,\alpha(t_0))\in I\}\right)\},
		\end{equation}
		where $\kappa I$ is the interval with the same center as $I$ but whose width is multiplied by $\kappa$.
	
\end{defn}

Essentially, this doubling condition requires that 

\begin{equation}
		\mu\left(\{s\in T:\delta X(s,\alpha(t_0))\in\kappa I\}\right)\geq (1+\sigma)\,\mu\left(\{s\in T:\delta X(s,\alpha(t_0))\in I\}\right),
		\end{equation}
		
		for some intervals $I$ centered at $x=0$ with length  $\ell \kappa^i, i \geq 0$ and smaller than $\beta$. This means that as we double the interval $I$ by the constant $\kappa$, we must capture more data than what was originally in $I$.

  Based on this uniform doubling condition on $\delta X$, we prove the following stability of accuracy result.

  \begin{prop}\label{lemmadoublingconditionondeltaX}
		Assume that the set $\{\delta X(s,\alpha(t_0)):s\in T\} \subset \R^1$ satisfies the uniform  doubling condition given in Def. \ref{eq:uniform_doubling_condtion_deltaX}, with $\kappa>1$ and $\delta,\sigma$, $\beta>0$, $1>\ell>0$  at point $x=0$ at some time $t_0$.
Then for the constant $C_3$ given in \eqref{eq:constant_C_3} we have: 
			\begin{equation}\label{eq:explicit_constants3}
		\quad \frac{\beta}{2} \geq \eta, \quad  \xi:=\ell \kappa^i  \quad i \in \N,\quad  \delta_0<\delta.
		\end{equation} 
			and if good and bad sets at $t=t_0$ satisfy
			\begin{equation}\label{good_set_big_and_bad_set_empty_B}
				\mu(G_\eta(t_0))>1-\delta_0\quad\text{and}\quad B_{-\xi }(t_0)=\emptyset,
			\end{equation}
			then for all $t\geq t_0$,
			\begin{equation}\label{eq:stability_B}
				\text{loss}(t)\leq \log(2)C_3, \quad \text{acc}(t)=\mu(G_0(t))\geq 1-C_3.
    \end{equation}
	\end{prop}

\begin{ex}
\label{result_deltax}
  We trained a DNN for the problem given in Example \ref{example_1} and obtained a $98\%$ accuracy on the training set, which had $1000$ objects. $\delta X$ for the $600$th epoch of training is given in Fig. \ref{deltaX_final}.  

  \begin{figure}
			\includegraphics[scale=.7]{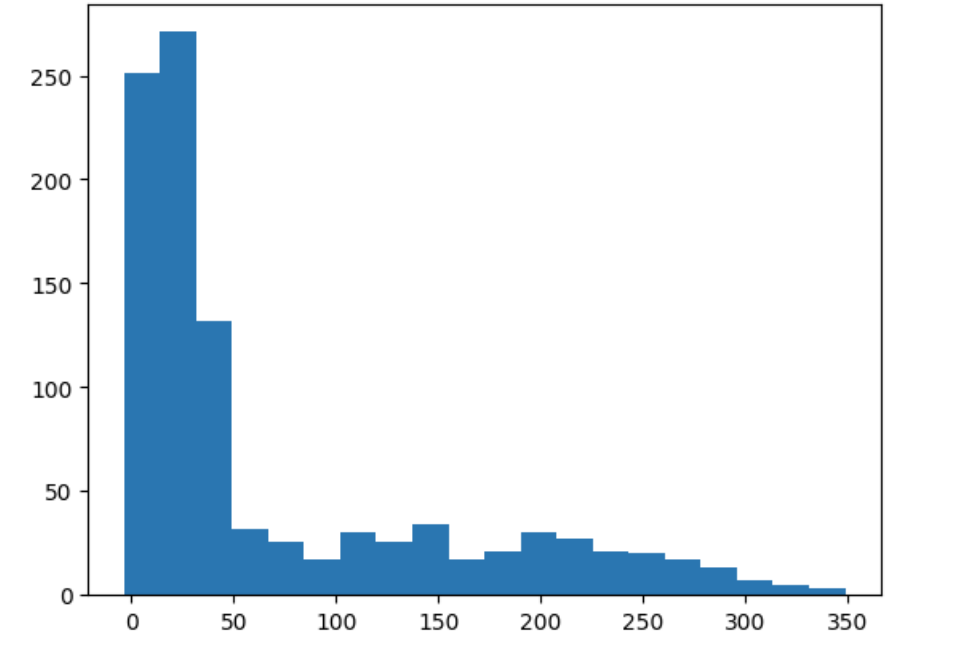}
			\caption{$\delta X$ for final epoch of DNN.}
			\label{deltaX_final}
		\end{figure}

  In this example, $\xi=.7$ and we checked the uniform doubling condition on $\delta X$ for the parameters: $\sigma=.9$, $\kappa=2$, $\delta=.2$, and $\ell=.7$. The uniform doubling condition was satisfied in this example, and taking $\eta=18$ we have that $\delta_0=.2$ and $C_3=.047$. This means that loss is bounded by $.015$, which indeed it was  (loss was $.015$) and accuracy was bounded by $.95$ for the rest of the training.      
  
\end{ex}

Of course, once training has started we need not check the uniform doubling condition on $\delta X$ as the loss will give us all the information we need regarding the stability of accuracy of the DNN. Still, to obtain stability of accuracy from the training data $T$ only, we first need to prove Proposition \ref{lemmadoublingconditionondeltaX}.  

\subsection{Doubling condition on the data set}

In \cite{LJS} a doubling condition was also introduced on the training set directly and it was shown that this doubling condition propagates through the DNN, such that if the training set satisfies the doubling condition then $\delta X$ also satisfies the doubling condition. One can then ensure that $\delta X$ satisfies the doubling condition by checking that the training set satisfies the doubling condition. Thus, verifying that the training set satisfies the doubling condition will, in turn, ensure that we have the stability of accuracy.


  We can now state the doubling condition for the training set $T$:
  
  	\begin{defn}\label{def:no_small_islated_data_cluster1}
		Let $\bar\mu$ be the extension of the measure $\mu$ from $T$ to $\mathbb R^n$ by $\bar\mu(A)=\mu(A\cap T)$ for all $A\subset\mathbb R^n$. The \emph{doubling} condition holds on $T$ if there exists  $\kappa>1$, $k \in \N$ and $m_0, \delta,\sigma>0$ so that $\forall$ truncated slabs $S$ $\exists \beta_S$ (depending on the slab) such that for any $\eps \leq \beta_S$ we have
		\begin{equation}\label{eq:no_small_isolated_data_cluster_1}
		\bar\mu(S_{\kappa\,\eps})\geq \min\Big\{\delta,\;\max\big\{m_0,\;(1+\sigma)\,\bar\mu(S_{\eps\,})\big\}-m_0\Big\}.
		\end{equation}
		\end{defn}

Here, $\beta$ depends on the truncated slab, and so we denote $\beta_S$ as the $\beta$ for which the truncated slab $S$ satisfies the doubling condition. See Fig. \ref{doubling_condition_with_data} for an example of the doubling condition. $m_0>0$ is needed in the doubling condition on a discrete data set $T$ because we can always find a slab with a small enough width such that it and its double contains only one data point of $T$. Thus, no discrete data set $T$ will satisfy the doubling condition with $m_0=0$. The residual mass $m_0$ is the smallest amount of data for which the doubling condition must be satisfied and we want it to at least be bigger than the measure of a single data point in the training set $T$.

Using this doubling condition, the following theorem was essentially proved in \cite{LJS}:

  \begin{thm}[Berlyand, Jabin, Safsten, 2021]\label{Pro_DC}
  Suppose the training set $T$ satisfies the doubling condition with constants $m_0$, $\kappa$, $\sigma$, and $\beta_S$ then for a DNN with the absolute value activation function, $\delta 
  X$ will also satisfy the doubling condition though with different constants $2^{\sum_{i=1}^l c(i)}K m_0$, $\frac{\sigma}{2^{\sum_{i=1}^l c(i)}K m_0}$ and $\kappa$. Here $c(i)$ is the number of nodes in the $i$th layer of the DNN and $K$ is the number of classes.   
  \end{thm}
  
  There are a couple of difficulties with using Theorem \ref{Pro_DC} and Theorem \ref{thm:stabilityB}. First, the constant $2^{\sum_{i=1}^l c(i)}K m_0$ can be large if the number of nodes in the DNN is large, and so $2^{\sum_{i=1}^l c(i)}K m_0 \leq \frac{\eps}{2C_1}$ in condition \eqref{eq:explicit_constants_2} will be hard to satisfy. Further, Theorem \ref{Pro_DC} does not provide a method for obtaining $\beta$ on $\delta X$ from the $\beta_S$ on the training set $T$, and so one cannot verify condition \eqref{eq:explicit_constants_2} from the training set alone. This will be described in greater detail in the next section which motivates our new results. Using the uniform doubling condition, one can obtain the following similar result:  
  
   \begin{prop}\label{Pro_DC2}
  Suppose the training set $T$ satisfies the uniform doubling condition with constants $\ell$, $\kappa$, $\sigma$, and $\beta$ and with the family of truncation sets $L$, then for a.e. DNN with the absolute value activation function and truncation set $B \in L$, $\delta 
  X$ will also satisfy the doubling condition though with different constants $\ell d_{\min}\leq \ell' \leq \ell d_{\max}$, $\beta d_{\min}$, $\sigma$ and $\kappa$.   
  \end{prop}

  \subsection{Why truncations in the uniform doubling condition might not be important}
\label{are_truncations_important}

The essence of the proof for Theorem \ref{weakenedDC1} is as follows. First, we prove Prop. \ref{lemmadoublingconditionondeltaX} for stability of accuracy based on the uniform doubling condition on $\delta X$. Then we use Prop. \ref{Pro_DC2} to say that the uniform doubling condition propagates from the training set $T$ to $\delta X$. What we want is that the interval $I_{\ell}$ of size $\ell$ around $0$ in $\delta X$ satisfies the uniform doubling condition given in Def. \ref{eq:uniform_doubling_condtion_deltaX}. The preimage of this interval, under the DNN, is a union of truncated slabs in $\R^n$ which we denote $U_{\ell}$. For large DNNs, we might have millions if not more truncated slabs in this union. If by doubling the widths of all truncated slabs in this union by $\kappa$ we obtain a set, which we call $U_{\kappa\ell}$,  that contains a total of $(1+\sigma)$ more elements (counting all elements in two or more truncated slabs only once) than were in the original union, we would know that 
\begin{equation}    \mu(I_{\kappa\ell})>(1+\sigma)\mu(I_{\ell}).
\end{equation}

In general, if 

\begin{equation}
\label{U_ell}
    \mu(U_{\kappa^{i+1}\ell})>(1+\sigma)\mu(U_{\kappa^i\ell})
\end{equation}

then

\begin{equation}
    \mu(I_{\kappa^{i+1}\ell})>(1+\sigma)\mu(I_{\kappa^i\ell}).
\end{equation}

So if \eqref{U_ell} holds for all $\ell\kappa^i<\beta$, we would have the stability of accuracy result given in Theorem \ref{lemmadoublingconditionondeltaX} and Theorem \ref{weakenedDC1}. Equation \eqref{U_ell} can be checked directly for different unions of truncated slabs $U'$, though the computation might not practically be feasible. However, the essence of \eqref{U_ell} is that for some large set $U$ of possibly millions of truncated slabs we have some "doubling" property. Meaning that if we "double" $U$ we obtain $(1+\sigma)$ more elements. Because $U$ is large, it might be practically useful to consider other methods for finding the global "doubling" properties of the data set. This can be done by either ignoring the truncations in the uniform doubling condition (i.e. SUDC) or possibly considering checking the doubling condition for arbitrary sets (possibly even very complicated ones such as a subset of what $U$ might look like). The more different subsets of $\R^n$ satisfy an equation of the form given in \eqref{U_ell}, the more likely it is that \eqref{U_ell} itself is satisfied and that we have stability of accuracy.

\section{More motivation for introducing the uniform doubling condition}\label{Motivation}

As mentioned, we want to have a doubling condition on the training set that can be propagated through the DNN to $\delta X$ so that verifying the doubling condition on the training set ensures the stability of accuracy during training. This is why Definition \ref{def:no_small_islated_data_cluster1} and Theorem \ref{Pro_DC} are introduced.  
Note, however, that in Definition \ref{doubling_condition_DX} we have a single constant  $\beta$ in the doubling condition on $\delta X$, while in Definition \ref{def:no_small_islated_data_cluster1} the constant $\beta_S$ depends on the slab $S$. In \cite{LJS} no method is given for obtaining the $\beta$ for the doubling condition on $\delta X$ from the $\beta_S$ on $T$. Thus, we say that there is no way of quantifying the relationship between $\beta_S$ on $T$ and $\beta$ on $\delta 
  X$.  Even if we show that the doubling condition holds on $T$, we cannot use Theorem \ref{Pro_DC} and Theorem \ref{thm:stabilityB} to obtain a complete stability of accuracy result. This is because, we must know $\beta$ on $\delta X$ to know whether the condition \eqref{eq:explicit_constants_2} in Theorem \ref{thm:stabilityB} is satisfied, and there is no way of knowing $\beta$ from the $\beta_S$ given in the doubling condition on $T$. Specifically, in condition \eqref{eq:explicit_constants_2} $\beta$ is used to obtain an upper bound on $\eta$ (which is important for quantifying the good and bad sets), yet there is no way of knowing what $\beta$ on $\delta X$ is from the $\beta_S$ in the doubling condition on $T$.

  Thus the  doubling condition on $T$ does not give a fully quantifiable stability result. That is, the doubling condition on the training set provides some constants which are then used in the stability of accuracy theorem, however in \cite{LJS} the full relationship between the constants of the doubling condition and the stability theorem had not been established. 
The first goal of the new results is to introduce a new uniform doubling condition (given in Def. \ref{def:no_small_islated_data_cluster_2}) on the training set which ensures the stability of accuracy during training in a quantifiable manner. Thus, by verifying the doubling condition on the data set, one can ensure the stability of accuracy without having to know $\delta X$ directly.  

        This uniform doubling condition has some benefits over the doubling condition given in Definition \ref{def:no_small_islated_data_cluster1}. First, $\beta$ is fixed and does not depend on the slab $S$. This allows us to propagate this $\beta$ through the DNN to a doubling condition on $\delta X$ and thus obtain a quantifiable stability of accuracy result. Thus, all constants in the stability theorem come directly from the given data $T$, though only DNNs with specific architecture and parameters and reasonable $d_{\min}$ and $d_{\max}$ will have that $\delta X$ satisfies the doubling condition. Knowing what $d_{\min}$ and $d_{\max}$ are and if the DNN truncation set $B \in L$ requires knowledge of the DNN, but no knowledge of $X(T)$ is needed.     Second, we replace the constant $m_0$ with the constant $\ell$, and while the constant $m_0$ blows up in Theorem \ref{Pro_DC} the constant $\ell$ does not blow up during the propagation through the DNN. Recall that the constant $m_0$ ensured that a data set with discrete data points might still satisfy the doubling condition, however, because in the uniform doubling condition we don't need to check the doubling condition for all $\epsilon<\beta$ but only for $\ell\kappa^i<\beta$, $m_0$ is not needed. Thus, in the doubling condition, if we did not have $m_0$ then no data set would satisfy the doubling condition because we can take slabs small enough that only one data point is contained in $S_\epsilon$ and $S_{\kappa \epsilon}$.  However, in the uniform doubling condition, we only consider slabs bigger than $\ell$ and so we can remove the constant $m_0$ from the condition.    Finally, for any given truncated slab $S$, we need only check the condition for a finite number of slabs $S_{\ell \kappa^i}$ with $\ell \kappa^i<\beta$  rather than an infinite number of slabs $S_{\epsilon}$ with $\epsilon<\beta$.

\begin{remark}
Take a manifold $H$ to satisfy the uniform doubling condition with $k>0$. For many such $H$, with $H$ a smooth manifold with measure $\mu$ distributed uniformly, we must have a restriction on the truncation sets in the family of truncation sets $L$. For simplicity, it suffices to explain why this is true for the case when $H$ is a bounded curve in $\R^2$ with a boundary. Suppose $L$ includes all possible truncation sets of size $k$, then we can truncate a slab with center $x$ which is $\beta$ distance from the edges of the curve in such a manner that the uniform doubling condition will be violated for every $H$. This is because we are $\beta$ close to the boundary of the curve $H$ and so $S_{\kappa \epsilon}$ will not contain more of the curve if the truncation is normal to the curve and passes through $x$. See Fig. \ref{QE1}. Thus, $H$ would only satisfy the uniform doubling condition if no truncation sets in $L$ has a truncation $\beta$ close to the boundary of $H$. This is why the family of truncation sets $L$ must be introduced in the case of the uniform doubling condition (in which $\beta$ does not depend on the slabs $S$).

\label{rem3}
\end{remark}

\begin{figure}[h!]
\includegraphics[scale=.8]{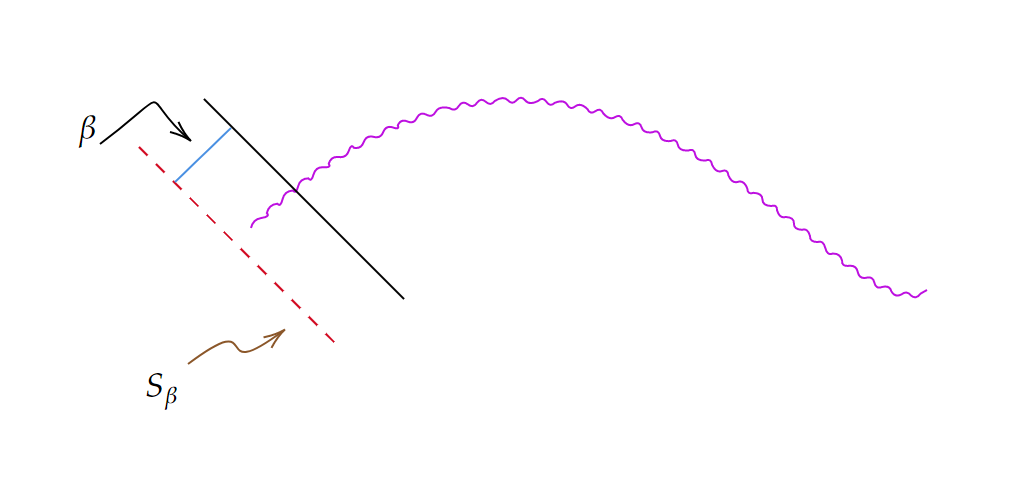}
\caption{The curve $H$ is shown as a squiggly curve  (in purple). Here $\mu(S_{\beta})=\mu(S_{\kappa \beta})$, and so would not satisfy the doubling condition.}
\label{QE1}
\end{figure}


\section{Proof of main result: stability theorem via uniform doubling condition on $T$}\label{main_result_1}

The main result in this section is given in Theorem \ref{weakenedDC1}. But first, we provide a couple of lemmas that show how one can quantitatively propagate the uniform doubling condition from the training set $T$ to $\delta X$.  

 First we introduce a slight variation of the uniform doubling condition which is only needed for a technical reason:  

	\begin{defn}\label{def:no_small_islated_data_cluster_5}
		Let $\bar\mu$ be the extension of the measure $\mu$ from $T$ to $\mathbb R^n$ by $\bar\mu(A)=\mu(A\cap T)$ for all $A\subset\mathbb R^n$. Let $L$ be a family of size $k$ truncation sets. The \emph{uniform doubling} condition holds on $T$ if there exists  $\kappa>1$ and $\delta,\sigma, \beta$ so that $\forall$ truncated slabs $S$ s.t. $S$ is generated by a truncated set $h \in L$ there exists  $\ell_S>0$, such that for all $\ell_S\kappa^i<\beta$ with $i \in \mathbb N$,  we have

		\begin{equation}\label{eq:no_small_isolated_data_cluster_0}
		\bar\mu(S_{\ell_S\kappa^{i+1} })\geq \min\{\delta,\;(1+\sigma)\,\bar\mu(S_{\ell_S\kappa^i})\}.
		\end{equation}
		
	\end{defn}

 \begin{remark}
     The only difference between the uniform doubling condition and this definition is that now $\ell$ depends on the slab $S$.  
 \end{remark}

We can now state the following lemma which shows how the doubling condition on $T$ can propagate through the absolute value activation function:

\begin{lem}\label{Lemmaactivationfunction2}
			Suppose the data set $T$ satisfies the uniform doubling condition with constants $\kappa>1$, $\delta,\beta,\ell,\sigma>0$, family  $L$ of size $k$ truncation sets consisting of a single set of coordinate planes $\{P_1, \dots, P_n \}\ $ and truncation number $k=n$. Define $\lambda:\R^n \to \R^n$: $$\lambda(x_1,x_2,\dots, x_n)= (|x_1|,|x_2|,\dots, |x_n|).$$ 
			
			Then $\lambda(T)$ also satisfies the doubling condition for scale parameter $\beta$, scale factor $\ell$, truncation number $0$, $\delta$ and $\frac{\sigma}{2^n}$. Furthermore, if no elements of $T$ are contained in the coordinate planes $P_1, \dots, P_n$ then $\sigma$ stays constant. $\lambda$ also propagates the variation uniform doubling condition given in Def. \ref{def:no_small_islated_data_cluster_5} with $\ell_S$ and all other constants staying the same.

		\end{lem}

	\begin{proof}

Note that $\lambda=\lambda_1\circ\lambda_2 \circ \dots \circ \lambda_n$, with $\lambda_k(x_1,x_2,\dots, x_n)= (x_1,x_2,\dots, |
x_k|, \dots, x_n)$. Thus we prove this lemma for the simpler case, that $\lambda_1(x_1,\cdots,x_n) \to (|x_1|,\cdots, x_n)$ preserves the doubling condition on $T$ with the same scale parameter and $\frac{\sigma}{2}$ instead of $\sigma$. Then one can use the same proof to show that $\lambda_k(x_1,x_2,\dots, x_n)$ also preserves the doubling condition with $\frac{\sigma}{2}$ instead of $\sigma$ and the proof for $\lambda$ follows from $\lambda=\lambda_1\circ\lambda_2 \circ \dots \circ \lambda_n$.       
For all truncated slabs $S$, we want to show that $S_{\ell_S \kappa^i}$ satisfies the doubling condition for $\ell_S\kappa^i<\beta$. Take $\bar S_{\eps}$ to be the preimage of $S_\epsilon$ with $x_1 \geq 0$ and $\bar S'_{\eps}$ to be the preimage with $x_1 \leq 0$. It then makes sense to separate
\[\begin{split}
&\bar S_{\eps}=S^2_{\eps}\cup S^0_{\eps},\quad \bar S'_{\eps}=S^1_{\eps}\cup S^0_{\eps},\\
& S^2_\eps=\bar S_{\eps}\cap \{0>x_1\},\quad S^0_{\eps}=\bar S_{\eps}\cap \{x_1=0\}=\bar S'_{\eps}\cap \{x_1=0\},\quad S^{{1}}_\eps=\bar S'_{\eps}\cap \{x_1<0\}.
  \end{split}\]

By the doubling condition on $ \bar S_{\eps}$, we have for $\epsilon=\ell_S\kappa^i<\beta$
\begin{equation}
\mu(\bar S_{\kappa\,\eps\,})=\mu(S^2_{\kappa\,\eps})+\mu(S^0_{\kappa\,\eps})\geq \min\{\delta,\;(1+\sigma)\,\mu(\eps\,\bar S)\}.
\label{measE+}
\end{equation}
Similarly with regard to $ \bar S'_{\eps}$, for $\epsilon=\ell_S \kappa^i<\beta$
\begin{equation}
\mu(\bar S'_{\kappa\,\eps\,})=\mu(S^1_{\kappa\,\eps})+\mu(S^0_{\kappa\,\eps})\geq \min\{\delta,\;(1+\sigma)\,\mu(\eps\,\bar S')\}.
\label{measE-}
\end{equation}

 Notice, $\sigma$ remains fixed when $\mu(P_1)= \emptyset$, due to the fact that the truncated slab $\bar{S}_\epsilon$ and $\bar{S}'_\epsilon$ both satisfy the doubling condition  independently. Thus, if $\mu(P_1)$ is empty then simply applying \eqref{measE+} and \eqref{measE-} we get \begin{equation}
   \mu(S_{\kappa\epsilon})> \min\{\ \delta, (1+\sigma)\mu (S_{\epsilon}) \}\
\end{equation}

For the more complicated case, when $\mu(P_1) \neq \emptyset$, see Section \ref{Appendix_A}.  

\begin{remark}

  $\lambda_1(T)$ satisfies doubling condition with truncation number $k-1$ because one truncation is required at $(0,x_2,\dots,x
			_n)$. Thus, $\lambda$ will satisfy the doubling condition with truncation number $0$.

\end{remark}

	\end{proof}

	We now show how the doubling condition propagates through a linear map. First, we collect some useful facts from linear algebra in the following lemma. This lemma will help determine how $\beta$ and $\ell$ change under a linear map. 
	
	\begin{lem}\label{dilation}
 Let $M$ be an $m \times n$ matrix. Then the minimum value of
$||Mx||$, where $x$ ranges over unit vectors in $\R^n$ orthogonal to $Ker(M)$, is the smallest nonzero singular value $\sigma_{\min}$ of $M$. Similarly, the maximum values of $||Mx||$ is the largest singular value $\sigma_{\max}$ of $M$.    
\end{lem}

For a linear map $M$ and a family of truncation sets $L$, take $M(L)$ to be the mapping of the truncation sets in $L$ by $M$. Thus, $M(L)$ is also a family of truncation sets. 

	\begin{lem}\label{linear1}
If $T$ satisfies the uniform doubling condition (Definition \ref{def:no_small_islated_data_cluster_5}) for some constants and some family of truncation sets $L$, then $M(T)$ also satisfies the variation uniform doubling condition (Definition \ref{def:no_small_islated_data_cluster_5}) for any linear map $M$ with the family of truncation sets $M(L)$. Take $d_{\min}=min_i{(\sigma_i)^2}$ and $d_{\max}=\max_i{(\sigma_i)^2}$ with $\sigma_i$ the nonzero singular values of $M$. Then all slabs $S \subset M(T)$ satisfy the doubling condition with $d_{\min}\beta$  instead of $\beta$, some $\ell'_S$ with $d_{\min}\ell \leq \ell'_S\leq d_{\max}\ell$ instead of $\ell$. 
 All other constants stay the same. Similarly with regard to Def. \ref{def:no_small_islated_data_cluster_5}. If $T$ satisfies Def. \ref{def:no_small_islated_data_cluster_5} then all slabs $S \subset M(T)$ satisfy the variation of the uniform doubling condition with $d_{\min}\beta$  instead of $\beta$, some $\ell'_S$ with $d_{\min}\ell_S \leq \ell'_S\leq d_{\max}\ell_S$ instead of $\ell_S$.
\end{lem}

\begin{proof}
Suppose $v_i\cdot x\leq t_i$ then any $y$ s.t. $M\,y=x$ also satisfies that
  \[
(M^T\,v_i)\cdot y\leq t_i. 
\]

Take $S$ to be a truncated slab in $M(T)$ and $S'$ to be its preimage in $T$. By Lemma \ref{dilation} the width of every slab $M(S')$ is bigger or equal to $d_{\min}$, and thus the width of $M(S'_{\beta})$ is bigger or equal to $d_{\min}\beta$. Thus, we can take $d_{\min}\beta $ to be the scale parameter in $M(T)$. We now need to find the scale factor for the doubling condition in $M(T)$ 

For the slab $S$, there exists $d_{\min}\ell\leq \ell'_S\leq d_{\max} \ell$ such that the preimage of $S_{\ell'_S}$ is the slab $S'_{\ell_{S'}}$. We also have that the preimage of $S_{\kappa^i\ell'_S}$ is $S'_{\kappa^i\ell_{S'}}$. Thus because $S_{\kappa^i\ell_{S}}$ satisfies the doubling condition for ${\kappa^i\ell_S}<\beta$ we must have that $S_{\kappa^i\ell_S'}$ also satisfies the doubling condition for $\kappa^i\ell_S'<d_{\min}\beta$ with $d_{\min}\ell \leq \ell_S'<d_{\max}\ell$ for all $S$.        
We arrive at the conclusion that $d_{\min}\beta$ is the scale parameter in $M(T)$ and each slab $S$ in $M(T)$ now has a different scale factor $\ell'_S$.  
\end{proof}

In both of these lemmas, we fully quantify the manner in which the constants of the doubling condition change, and so given that a layer of a DNN is a linear map composed with the activation function we have a quantifiable manner in which the doubling condition propagates through a DNN. If we find an  $\beta$ for which the doubling condition holds on $T$, and we know $d_{\min}$ and $d_{\max}$ of every linear map of the DNN then we know exactly how $\delta X$ satisfies the doubling condition.

\begin{remark}
\label{parallel}
There is difficulty with allowing the family of truncation sets $L$ to consist of all possible truncation sets when $k>1$ (the size of the truncation sets is bigger than $1$). This is because two truncations $\beta$ close to each other will ensure that the uniform  doubling condition is violated, see Fig \ref{paralleltruncations}. Thus no data set will satisfy the current uniform doubling condition if $L$ contains all possible truncation sets with $k>1$. The truncation set does not move under the doubling of a slab $S$, and so we can bound any slab (regardless of the dimension we are in), by two truncations which are $\beta$ close to each other. 

However, we need the truncation number $k \geq$ the number of nodes in DNN. By Lemma \ref{Lemmaactivationfunction2} each layer of the DNN reduces the truncation number $k$ by the nodes in that layer. This is because each activation function $\lambda(x_1,x_2, \cdots, x_n) \to (|x_1|,|x_2|, \cdots, |x_n|)$ requires a reduction of the truncation number $k$ by at least $n$. So the doubling condition on $T$ implies stability of DNN only if $k \geq$ is the number of nodes in the DNN. But as mentioned, for $k \geq 2$, slabs will never satisfy the doubling condition for any data $T$. Example: if truncation number $k=2$ the slab $S^0_{\beta}$ does not satisfy the uniform doubling condition (truncations don't move under doubling), see Fig. \ref{paralleltruncations}.

		\begin{figure}[h!]
			\includegraphics[scale=1]{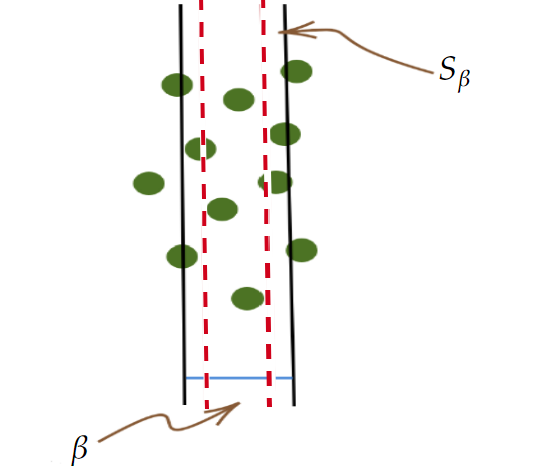}
			\caption{Slabs in $\R^2$. Red dashed lines are the original slab, and black solid lines are two truncations. $\mu(S_{\kappa \beta})=\mu(S_{ \beta})$ (slab bounded by truncations), thus the doubling condition is not satisfied.}
			\label{paralleltruncations}
		\end{figure}

'Almost' parallel truncations also pose a problem. This is because not only parallel but also 'almost' parallel truncations ensure that the doubling condition is never satisfied. The data set $T$ is bounded, so 'almost' parallel truncations can capture lots of data and ensure the doubling condition is not satisfied, see Fig. \ref{nearlypa}.

		\begin{figure}[h!]
			\includegraphics[scale=.8]{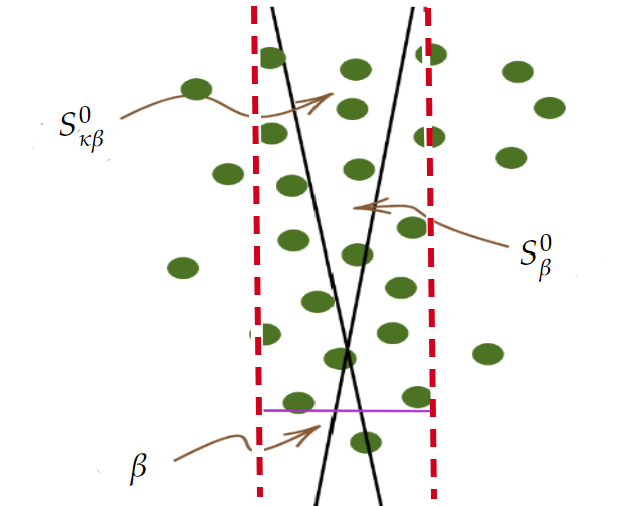}
			\caption{Slab in $\R^2$. Red dashed lines are the original slab (without the truncation), and black solid lines are two truncations. $S^0$ is the center slab after truncation. Even though truncations are not parallel, still  $\mu(S^0_{\kappa \beta})=\mu(S^0_{ \beta})$, thus the doubling condition is not satisfied.}
			\label{nearlypa}
		\end{figure}

		\begin{itemize}
			\item  So data $T$ never satisfies the uniform doubling condition with single $\beta$ because of 'almost'  parallel truncations
			\item Thus, we introduce the set $L$ of all admissible truncations and we use $L$ to impose restrictions on the allowed truncation sets. \end{itemize}
\end{remark}

	We now give a lemma providing an upper bound on the loss function. 	\begin{lem}\label{Upper_Bound_on_Loss}
For a DNN with the cross entropy loss function we have
    \begin{equation}
\sum_{s\in G_{\eta^*}^c(t)} \mu(s)\,\log (1+e^{-\eta^*})\leq \bar L(t)\leq \bar L(t_0).
\label{chebyshev}
\end{equation}

\end{lem}

See Section \ref{Appendix_A} for proof of the lemma. Now we derive an upper bound on the loss function based on the uniform doubling condition on $\delta X$.
\begin{lem}\label{lem:loss_bound_3}
	Suppose that $T\subset\mathbb \R^n$ with weights $\mu(s)$ is a training set for a softmax DNN such that $\delta X(T)$ satisfies the uniform doubling condition, at time $t_0$ at $x=0$, for some constants $\sigma, \delta, \beta$, $1>\ell>0$ and $\kappa>1$.  If for some $1<\eta<\frac{\beta}{2}$, $\delta_0>0$, and $\delta_0<\delta$,
	\begin{equation}\label{eq:good_set_big_2a2}
	\mu(G_\eta(t_0))>1-\delta_0
	\end{equation}
	and
	\begin{equation}\label{eq:bad_set_empty_2a}
	B_{-\xi}(t_0)=\emptyset,
	\end{equation}
	then the cross-entropy loss is bounded by:
	\begin{equation}
	\label{eq:loss_bounded_2}
	\begin{split}
	  \bar L(t_0)
          \leq  \log(1+(K-1)e^{-\eta})+\log(1+(K-1)\,e^{-(\frac{\eta}{\kappa} )})\,\delta_0+\\ \frac{\delta_0}{(\sigma+1)^{\frac{\log\left(\frac{\eta}{\xi}\right)}{\log \kappa}-1}}\log(1+e^{\xi}\,(K-1))
\\       +\sum_{i=0}^{p-1} \frac{\delta_0}{(\sigma+1)^{\frac{\log\left(\frac{\eta}{\ell \kappa^{i+1}}\right)}{\log \kappa}-1}}\log(1+(K-1) e^{-\ell \kappa^{i+1}}).
 	\end{split}
\end{equation}

 with  $\xi=\ell \kappa^i$ for some $i \geq 0$ and with $p=\lfloor\log(\frac{\eta}{\xi})/\log\kappa\rfloor$.
\end{lem}
\begin{proof}
 We have that $\eta\geq 1$.  Consider $\kappa>1$ and $I:=[-\ell,\ell)\subset [-\infty,\eta)$. Since $\eta<\frac{\beta}{2}$, we may apply the doubling condition at point $0$ and we have that
      
	\begin{equation*}
	\mu\left(\left\{s\in  T:\delta X(s,\alpha(t_0))\in \kappa\,I\right\}\right)\geq\min\{\delta,\;(1+\sigma)\,\mu\left(\left\{s\in T:\delta X(s,\alpha(t_0))\in I\right\}\right)\}.
	\end{equation*}
	Applying the doubling condition repeatedly $j$ times in this last case, we conclude that
	\begin{equation*}
	\mu\left(\left\{s\in  T:\delta X(s,\alpha(t_0))\in \kappa^j\,I\right\}\right)\geq\min\{\delta,\;(1+\sigma)^j\mu\left(\left\{s\in  T:\delta X(s,\alpha(t_0))\in I\right\}\right)\}.
	\end{equation*}
	By \eqref{eq:good_set_big_2a2}, $\mu\left(\left\{s\in T:\delta X(s,\alpha(t_0))<\eta\right\}\right)\leq\delta_0$. Therefore, provided $\kappa^j I\subset (-\infty,\;\eta)$,
	\begin{equation*}
	\delta_0\geq\mu\left(\left\{s\in T:\delta X(s,\alpha(t_0))\in \kappa^j\, I\right\}\right)\geq\min\{\delta,\;(1+\sigma)^j\mu\left(\left\{s\in  T:\delta X(s,\alpha(t_0))\in I\right\}\right)\}.
	\end{equation*}
We conclude that
	\begin{equation}\label{eq:interval_size_estimate_2WDC}
	\mu\left(\left\{s\in  T:\delta X(s,\alpha(t_0))\in I\right\}\right)\leq \frac{\delta_0}{(\sigma+1)^j}
	\end{equation}
	for all $j$ so that $\kappa^j I\subset(-\infty,\eta)$, or equivalently, for all $j$ so that $\ell \kappa^{j}\,\leq\eta$. Obviously, \eqref{eq:interval_size_estimate_2WDC} is best for $j$ as large as possible with the largest value given by
	\begin{equation*}
j_{\text{max}}=\left\lfloor\frac{\log\left(\frac{\eta}{\ell}\right)}{\log \kappa}\right\rfloor>\frac{\log\left(\frac{\eta}{\ell}\right)}{\log\kappa}-1.
	\end{equation*}
	Therefore         we have thus proved the bound,
	\begin{equation*}
 \label{interval_bound}
	\begin{split}
	\mu\left(\left\{s\in  T:\delta X(s,\alpha(t_0))\in I\right\}\right)&\leq \frac{\delta_0}{(\sigma+1)^{j_{\text{max}}}}\\
	&<\frac{\delta_0}{(\sigma+1)^{\frac{\log\left(\frac{\eta}{\ell}\right)}{\log \kappa}-1}}\\
        \end{split}
	\end{equation*}

	Now we can apply the calculation which obtained \eqref{interval_bound} to the intervals $[0,\ell \kappa^{i})$ with $i\in \N$. Let $p=\lfloor\log(\eta /\ell)/\log\kappa\rfloor$ and let $I_i=[\ell \kappa^i,\ell \kappa^{i+1})$ for all $i\in \N$ with $i<p$. We also define $I_{-1}=[-\xi,\ell)$ and $I_p=[\ell \kappa^p,\ \eta)$. For $i\geq 0$, $I_i\subset [-\ell \kappa^{i+1},\ \ell \kappa^{i+1}]$. Therefore by the above calculation which gave  \eqref{interval_bound} we obtain,
	\begin{equation}\label{eq:interval_size_estimate_3WDC}
	\mu\left(\left\{s\in  T:\delta X(s,\alpha(t_0))\in I_i\right\}\right)<\frac{\delta_0}{(\sigma+1)^{\frac{\log\left(\frac{\eta}{\ell \kappa^{i+1}}\right)}{\log \kappa}-1}}\\ 
	\end{equation}
	The interval $I_{-1} \subset [-\xi,\xi]$ which is centered at $0$ and has width smaller than $\xi$, so
	\begin{equation}\label{eq:interval_size_estimate_4WDC}
	\mu\left(\left\{s\in T:\delta X(s,\alpha(t_0))\in I_{-1}\right\}\right)\leq \frac{\delta_0}{(\sigma+1)^{\frac{\log\left(\frac{\eta}{\xi}\right)}{\log \kappa}-1}}.
	\end{equation}
        Finally since $\kappa^p>\kappa^{\log\eta/\log\kappa-1}$, we simply bound for $I_p$ 
        \begin{equation}\label{eq:interval_size_estimate_5WDC}
	\mu\left(\left\{s\in T:\delta X(s,\alpha(t_0))\in I_{p}\right\}\right)\leq \delta_0.
	\end{equation}
	For each $s\in T$, either $\delta X(s,\alpha(t_0))\geq \eta$, or $\delta X(s,\alpha(t_0))\in I_i$ for integer $i\geq-1$. Therefore, 
	\begin{equation*}
	\begin{split}
	\bar L(t_0)&\leq\sum_{s\in T}\mu(s)\log\left(1+(K-1)e^{-\delta X(s,\alpha(t_0))}\right)\\
	&=\sum_{\substack{s\in T\\ \delta X(s,\alpha(t_0))\geq\eta}}\mu(s)\log\left(1+(K-1)e^{-\delta X(s,\alpha(t_0))}\right)+\sum_{i=-1}^{p}\sum_{\substack{s\in T\\ \delta X(s,\alpha(t_0))\in I_i}}\mu(s)\log\left(1+(K-1)e^{-\delta X(s,\alpha(t_0))}\right)\\
	&\leq\log(1+(K-1)e^{-\eta})\sum_{\substack{s\in T\\ \delta X(s,\alpha(t_0))\geq\eta}}\mu(s)+\sum_{i=-1}^{p}\log\left(1+(K-1)e^{-\inf I_i}\right)\sum_{\substack{s\in T\\ \delta X(s,\alpha(t_0))\in I_i}}\mu(s).\\
        \end{split}
        \end{equation*}
By decomposing       
	\begin{equation*}
	\begin{split}
	\bar L(t_0)&	\leq \log(1+(K-1)e^{-\eta})\mu(G_\eta(t_0))+\sum_{i=-1}^{p}\log\left(1+(K-1)e^{-\inf I_i}\right)\mu\left(\left\{s\in T:\delta X(s,\alpha(t_0))\in I_i\right\}\right)\\
	&\leq \log(1+(K-1)\,e^{-\eta})+\log\left(1+(K-1)e^\xi\right)\mu\left(\left\{s\in T:\delta X(s,\alpha(t_0))\in I_{-1}\right\}\right)\\
        &+\log(1+(K-1)\,e^{-(\frac{\eta}{\kappa})})\,\mu\left(\left\{s\in T:\delta X(s,\alpha(t_0))\in I_{p}\right\}\right)\\
        &+\sum_{i=0}^{p-1}\log(1+(K-1)e^{-\kappa^{i+1}})\mu\left(\left\{s\in T:\delta X(s,\alpha(t_0))\in I_{i+1}\right\}\right).\\
\end{split}
	\end{equation*}
We now use \eqref{eq:interval_size_estimate_3WDC}, \eqref{eq:interval_size_estimate_4WDC} and \eqref{eq:interval_size_estimate_5WDC} to derive  
        	\begin{equation}\label{eq:cond_B_loss_estimateWDC}
	\begin{split}
	  \bar L(t_0)
          \leq  \log(1+(K-1)e^{-\eta})+\log(1+(K-1)\,e^{-(\frac{\eta}{\kappa })})\,\delta_0+\\ \frac{\delta_0}{(\sigma+1)^{\frac{\log\left(\frac{\eta}{\xi}\right)}{\log \kappa}-1}}\log(1+e^{\xi}\,(K-1))
\\       +\sum_{i=0}^{p-1} \frac{\delta_0}{(\sigma+1)^{\frac{\log\left(\frac{\eta}{\ell \kappa^{i+1}}\right)}{\log \kappa}-1}}\log(1+(K-1) e^{-\ell \kappa^{i+1}}).
 	\end{split}
	\end{equation}
         %

\subsection{Proof of Proposition \ref{lemmadoublingconditionondeltaX}}
We start with the following bound given in Lemma \ref{Upper_Bound_on_Loss} 
\begin{equation}
\sum_{s\in G_{\eta^*}^c(t)} \mu(s)\,\log (1+e^{-\eta^*})\leq \bar L(t)\leq \bar L(t_0).
\label{chebyshev}
\end{equation}
As a consequence, we obtain from Lemma \ref{lem:loss_bound_3} that provided $G_{-\xi}^c(t_0)=\emptyset$ and $\mu(G_\eta(t_0))>1-\delta_0$,
\begin{equation}
\begin{split}
\label{eq:constant_C_3}
\mu(G_0^c(t))
          \leq C_3(\eta,K,\kappa,\sigma,\ell,\xi):= \frac{1}{\log(2)}  (\log(1+(K-1)e^{-\eta})+\log(1+(K-1)\,e^{-(\frac{\eta}{\kappa})})\,\delta_0+\\ \frac{\delta_0}{(\sigma+1)^{\frac{\log\left(\frac{\eta}{\xi}\right)}{\log \kappa}-1}}\log(1+e^{\xi}\,(K-1))
      +\sum_{i=0}^{p-1} \frac{\delta_0}{(\sigma+1)^{\frac{\log\left(\frac{\eta}{\ell \kappa^{i+1}}\right)}{\log \kappa}-1}}\log(1+(K-1) e^{-\ell \kappa^{i+1}})).
 	\end{split}
  \end{equation}
Here $p=\lfloor\log(\eta/\ell)/\log\kappa\rfloor$. This completes the proof. 

\end{proof}

   We can now present the proof for the main result in this paper, Theorem \ref{weakenedDC1}

\subsection{Proof of Theorem \ref{weakenedDC1}}

\begin{proof}
		Take $I$ to be an interval in $\R^1$ centered at $x=0$ with width $ \ell' \kappa^i<d_{\min} \beta$ with $\ell_S<d_{\max}\ell$. We wish to prove that $\exists \ell'$ for which $\mu(I_{ \kappa})>\min \{\ \delta, (1+\sigma)\mu (I) \}\ $ at $t_0$. Take $\tilde T$ to be the image of the data $T$ by the DNN. We note that if $f^{-1}_i(\tilde T)$ (the preimage of $\tilde T$ with respect to the last layer in the DNN) satisfies the uniform doubling condition with truncation set $\{f^{-1}_i(P_1),f^{-1}_i(P_2), \cdots f^{-1}_i(P_{c(i)})\}\ $ $\forall f^{-1}_i(P_k)$ with $1 \leq k \leq c(i)$ such that $f^{-1}_i(P_k)$ are strict subsets of $R^{c(i-1)}$, then $\tilde T$ satisfies the variation uniform doubling condition (Def. \ref{def:no_small_islated_data_cluster_5}) though with different constants given in lemma \ref{Lemmaactivationfunction2}. In fact, by lemmas \ref{Lemmaactivationfunction2} and \ref{dilation}, we get that if $f^{-1}_i(\tilde T)$ satisfies the uniform doubling condition with $\beta$, $\sigma$, $k$, $\ell$ and the mentioned truncation set, then $\tilde T$ satisfies the variation uniform doubling condition with $\frac{\sigma}{2^{c(i)}}$, $d'_{\min}\beta$ and some $\ell d'_{\min}\leq \ell'\leq d'_{\max}\ell$. Here $d'_{\min}$ is the smallest singular value of the matrix part of the affine function $M_i$ and $d'_{\max}$ is the largest singular value.   However, for a.e. DNN the truncation set $\mu(\{f^{-1}_i(P_1),f^{-1}_i(P_2), \cdots f^{-1}_i(P_{c(i)})\})$ has measure zero (the measure with respect to the data), and so by lemma \ref{Lemmaactivationfunction2} we have that $\sigma$ remains the same.
		
		Similarly, assume  $f^{-1}_{i-1}(f^{-1}_i(\tilde T))$ satisfies the uniform doubling condition with truncation set 
		$$    
		\{f^{-1}_{i-1}(P_1),f^{-1}_{i-1}(P_2), \cdots f^{-1}_{i-1}(P_{c(i-1)})\}\ \cup \{f^{-1}_{i-1}(f^{-1}_i(P_1)), \dots, f^{-1}_{i-1}(f^{-1}_i(P_{c(i})))\}\  
		$$
		
		(again only considering the preimages which are strict subsets of $R^{c(i-2)}$). Then we have that $f^{-1}_i(\tilde T)$ satisfies the variation uniform doubling condition with truncation set $$\{f^{-1}_i(P_1),f^{-1}_i(P_2), \cdots f^{-1}_i(P_{c(i)})\}\ $$ and so $\tilde T$ satisfies the doubling condition for $d_{\min} \ell \leq \ell_S \leq \ell d_{\max}$ with $d_{\min}$ and $d_{\max}$ coming from $W_2 \circ W_1$.  
		
		Thus, by induction if $T$ satisfies the uniform doubling condition with truncation set $B \in L$ then, for a.e. DNN, $\tilde T$ satisfies the variation uniform doubling condition with $d_{min}\beta$, and some $d_{\max}\ell\leq \ell_S \leq d_{\max}\ell$. Finally, we must show that if $\tilde{T}$ satisfies the uniform  doubling condition then so does $\delta X(\tilde{T})$, though possibly with different constants. Note that $\delta X(s)$ for $s$ in the class $i(s)$ is a shift of $X_{i(s)}(s,\alpha)$ by $\max_{j\neq i(s)} X_j(s,\alpha)$ and then a projection of the other components of $X(s,\alpha)$ to zero. Thus, $\delta X(s,\alpha)$ restricted to every class $i(s)$ and further restricted to elements $s \in T$ such that $\max_{j\neq i(s)} X_j(s,\alpha)$ occurs for class $i'$ is a linear map with $d_{\max}=d_{\min}=1$.  Thus, by using Proposition \ref{lemmadoublingconditionondeltaX} and the fact that $\delta X$ satisfies the uniform doubling condition at $x=0$ at time $t_0$, with the constants $d_{\min}\beta$, some $ \ell d_{\min} \leq \ell_I \leq \ell d_{\max}$ and $\sigma$, we obtain the result. Specifically, the constant $C_2$ is obtained from taking the max over all $ \ell d_{\min} \leq \ell_I \leq \ell d_{\max}$ of the constant $C_3$. All other constants stay the same.

		\end{proof}

\section{Main Result 2: Other Activation Functions }\label{otheractivationfunctions}

The goal of this section is to prove that piecewise linear activation functions with only finitely many critical points also preserve the uniform doubling condition. As we will see, there is only one problem that arises when trying to show that a piecewise linear function $g$ (with only finitely many critical points) preserves the doubling condition. This is dealing with the critical points of $g$. 
\begin{defn}
Take $g$ to be a $1D$ function with only finitely many critical points in every bounded interval. Take $v_g$ to be the infimum distance between the critical values of neighboring critical points of $g$. Each critical point $u$ has at most two neighbors, which are the closest two critical points to $u$ with one on its left and the other on its right.
\end{defn}

\begin{lem}\label{PAF}
Let $g_1: \R \to \R$ be any piece-wise continuous linear functions with only a finite number of critical points in every bounded interval. Suppose for $g_1$ that $v_{g_1}$ is bigger than $2\kappa\beta$. Take $\{a^1_1, \dots, a^i_1\}$ to be the set of critical points of $g_1$, and suppose the data set $T$ satisfies the uniform doubling condition with the set $L$ consisting of the hyperplanes $\{P_{a^1_1}, \dots, P_{a^n_1}\}$ with $P_{a^k_1} \subset \R^n$ and $P_{a^k_1}= \{(x_1,x_2, \dots, x_n): x_1=a^k_1\}$. Then $$g:(x_1,x_2,\dots, x_n)
\to (g_1(x_1),x_2,\dots, x_n)$$ preserves the doubling condition for $T$ with new constants, $\alpha_1\ell$ instead of $\ell$, $\sigma/2$ instead of $\sigma$, $\alpha_2 \beta$ instead of $\beta$ and $k-\hat{l}$ instead of $k$.

Here, we take $l=max_{x\in g(T)}(|x_1|)$ and $\hat{l}$ is the number of critical points of $g_1$ in the interval $[-l,l]$. We take $\alpha_1$ to be the largest slope of the linear piecewise function $g_1$ and $\alpha_1$ to be the smallest slope.  
\end{lem}

\begin{proof}
\par
\noindent

Step 1: $g_1$ has one critical point.

Now let us assume that $g$ has critical points. Note that such critical points are obtained from critical points of $g_1$. In general, when dealing with such critical points we use the same strategy that was used with regard to absolute value activation functions (when $g_1$ is the absolute value). First, let us assume that there is only one critical point for $g$ and that this critical point is at $x_1=a^1_1$.

For all truncated slabs $E$, we want to show that $E$ satisfies the doubling condition for $\beta'=\alpha_2 \beta$ and $\ell_E<\alpha_2 \ell$. Take $\bar {E}_\eps$ to be the preimage of $E_\epsilon$ with $x_1 \geq a^1_1$ and $ \bar E'_{\eps}$ to be the preimage with $x_1 \leq a^1_1$. We can separate
\[\begin{split}
&\bar E_\eps\,=E^2_{\eps}\cup E^0_{\eps},\quad \bar E'_\eps\,=E^1_{\eps}\cup E^0_{\eps},\\
& E^2_\eps=\bar E_\eps\,\cap \{x_1>a^1_1\},\quad E^0_{\eps}=\bar E_\eps\,\cap \{x_1=a^1_1\}=\bar E'_\eps\,\cap \{x_1=a^1_1\},\quad E^{{1}}_\eps=\bar E'_\eps\,\cap \{x_1<a^1_1\}.
  \end{split}\]

Because $\ell'=\alpha_1\ell$ and $\beta'=\alpha_2 \beta$ we have that $\bar {E}_\epsilon$ and $\bar{E}'_\epsilon$ satisfy the doubling condition for $\epsilon=\ell'_E\kappa^i<\alpha_2\beta$ for some $\ell'_E<\alpha_1\ell$ depending on $\bar {E}$ and $\bar{E}'$. Thus, for $ \bar E_\eps$, we have $\forall\eps=\ell'_{\bar{E}}\kappa^i<\alpha_2\beta$
\begin{equation}
\mu(\bar E_{\kappa\,\eps\,})=\mu(E^2_{\kappa\,\eps})+\mu(E^0_{\kappa\,\eps})\geq \min\left\{\delta,\;\max\{m_0,(1+\sigma)\,\mu(\bar E_\eps\,)\}-m_0\right\}.
\label{measE+_new_activation_fuinction}
\end{equation}
Similarly with regard to $ \bar E'_\eps$, $\forall\eps=\ell'_{\bar{E'}}\kappa^i<\alpha_2\beta$
\begin{equation}
\mu(\bar E'_{\kappa\,\eps\,})=\mu(E^1_{\kappa\,\eps})+\mu(E^0_{\kappa\,\eps})\geq \min\left\{\delta,\;\max\{m_0,(1+\sigma)\,\mu(\bar E'_\eps\,)\}-m_0\right\}.
\label{measE-_new_activation_fuinction}
\end{equation}

Now, if either $(1+\sigma)\,\mu(\bar E_\eps\,)\geq \delta$ or $(1+\sigma)\,\mu(\bar E'_\eps\,)\geq\delta$, we are done.

 If both $ (1+\sigma)\,\mu(\bar E_\eps\,)<\delta$ and $ (1+\sigma)\,\mu(\bar E'_\eps\,)<\delta$. Then
\begin{equation}
    \begin{split}
  \mu(g^{-1}(E_{\kappa\,\epsilon\,}))&=\mu(E^1_{\kappa\,\epsilon})+\mu(E^0_{\kappa\,\epsilon}) +\mu(E^2_{\kappa\,\epsilon})\\&\geq \frac{\mu(E^1_{\kappa\,\epsilon})+\mu(E^1_{\epsilon})}{2}+\mu(E^0_{\kappa\,\epsilon}) +\frac{\mu(E^2_{\kappa\,\epsilon}) +\mu(E^2_{\epsilon}) }{2}\\
  &\geq \mu(E^1_{\epsilon})+\mu(E^0_{\epsilon}) +\mu(E^2_{\epsilon})+\frac{\mu(E^1_{\kappa\,\epsilon})+2\mu(E^0_{\kappa\,\epsilon}) +\mu(E^2_{\kappa\,\epsilon})- \mu(E^1_{\epsilon})-2\,\mu(E^0_{\epsilon}) -\mu(E^2_{\epsilon})}{2}.
  \end{split}
  \label{new_activation_function_3}
\end{equation}

Again, we are using the averaging trick to deal with the complication that we only have one $E_\epsilon^0$ and $E_{\kappa \epsilon}^0$.  
Applying now \eqref{measE+_new_activation_fuinction} and \eqref{measE-_new_activation_fuinction}, we find that
\begin{equation}
    \mu(g^{-1}(E_{\kappa\,\eps\,})\geq \mu(g^{-1}(E_\eps\,)) +\frac{\sigma}{2}\,(\mu(\bar E_\eps\,) +\mu(\bar E'_\eps\,)) \geq (1+\sigma/2)\, \mu(g^{-1}(E_\eps\,)) 
\end{equation}
proving the doubling condition with $\sigma'=\sigma/2$.

 Note that this part of the proof shows that absolute value activation functions and Leaky ReLU activation functions also preserve the doubling condition. In all these cases we only obtain a single
$E_\epsilon^0$ while we need two $E_\epsilon^0$ to use the doubling condition. Once again, we deal with this complication by introducing the averaging trick and exchanging $\sigma$ with $\frac{\sigma}{2}$.

\par
\noindent

Step 2: $g_1$ has general critical points.

 At this point, we wish to show that if $g_1$ is any piecewise linear function with only a finite number of critical points in every bounded interval then $$g:(x_1,x_2,\dots, x_i)
\to (g_1(x_1),x_2,\dots, x_i)$$ preserves the doubling condition for any bounded data $T$, though with different constants.

Take $\tilde{E_\epsilon}$ to be the preimage of the slab $E_\epsilon$ and $a^1_1<...<a^i_1$ to be the critical points of $g_1$. Recall, the images of neighboring $a^i_1$ (a critical point has at most two neighbors which are the two closest critical points to it, from the left and right) are more than $2\kappa\beta$ distance away from one another. For $\epsilon<\beta'$ we separate $\bar{E}_{a_i \epsilon} =\tilde{E}_{ \epsilon}\cap \{a^{i+1}_1 \geq x_1 \geq a^i_1\}$ and  $\bar{E}'_{a'_i \epsilon} =\tilde{E}_{ \epsilon}\cap \{ a^{i-1}_1 \leq x_1  \leq  a^i_1\}$. Note each disjoint connected component of $\tilde{E}_\epsilon$ intersects at most one critical point of $g_1$. Furthermore, for each $i$, we have that $\bar{E}_{a_i \kappa\epsilon}\cup \bar{E'}_{a'_i \kappa \epsilon}$  are disjoint.

 Now take, $E^2_{a_i \epsilon}=\bar{E}_{a_i \epsilon}\cap \{x_1 > a^i_1\}$, $E^1_{a'_i \epsilon}=\bar{E'}_{a'_i \epsilon}\cap \{x_1 < a^i_1\}$ and $E^0_{a_i \epsilon}$ is the complement of the two.  

 We must show that $E_\epsilon$ satisfies the doubling condition for $\epsilon=\ell'_E<\beta'$ with $\ell'_E\kappa^i<\ell\alpha_1$ and depending on the slab $E^1_{a_i\epsilon}$ and $E^2_{a'_i\epsilon}$. That is, take $E_\epsilon$ with $\epsilon<\beta'$.  

If $(1+\sigma)\mu(\bar{E}_{a_i \epsilon})>\delta$ or $(1+\sigma)\mu(\bar{E}'_{a'_i \epsilon})>\delta$ for any $i$ we are done. Dealing with the other cases, we look at each pair of slabs $\bar{E}_{a_i \epsilon}$ and $\bar{E}'_{a'_i \epsilon}$ together. If both $(1+\sigma)\bar{E}_{a_i \epsilon}<\delta$ and $(1+\sigma)\bar{E}'_{a'_i \epsilon}<\delta$ (we do this for every $i$ individually) then the same averaging strategy from above works for these two. Specifically,

\begin{equation}
\begin{split}
  \mu(g^{-1}(E_{\kappa\,\epsilon\,}))|_{\bar{E}_{a_i \kappa \epsilon} \cup \bar{E}'_{a'_i \kappa \epsilon}}  >&\mu(E^1_{a'_i\kappa\,\epsilon})+\mu(E^0_{
  a_i\kappa\,\epsilon}) +\mu(E^2_{a_i\kappa\,\epsilon})\\&\geq \frac{\mu(E^1_{a'_i\kappa\,\epsilon})+\mu(E^1_{a'_i\epsilon})}{2}+\mu(E^0_{a_i\kappa\,\epsilon}) +\frac{\mu(E^2_{a_i\kappa\,\epsilon}) +\mu(E^2_{a_i\epsilon}) }{2}\\
  &\geq  \mu(E^1_{a'_i\epsilon})+\mu(E^0_{a_i\epsilon}) +\mu(E^2_{a_i\epsilon})+\\&\frac{\mu(E^1_{a'_i\kappa\,\epsilon})+2\mu(E^0_{a_i\kappa\,\epsilon}) +\mu(E^2_{a_i\kappa\,\epsilon})- \mu(E^1_{a'_i\epsilon})-2\,\mu(E^0_{a_i\epsilon}) -\mu(E^2_{a_i\epsilon})}{2}.  
\end{split}
\end{equation}

Then because $\bar{E}'_{a'_i \epsilon}$ and $\bar{E}_{a_i \epsilon}$ satisfy the doubling condition we get 

$$
\mu(g^{-1}(E_{\kappa\,\eps\,}))|_{\bar{E}_{a_i \kappa \epsilon} \cup \bar{E}'_{a'_i \kappa \epsilon}}     \geq ( \mu(g^{-1}(E_\eps\,))|_{\bar{E}_{a_i \epsilon} \cup \bar{E}'_{a'_i \epsilon}}  $$
 
$$+\frac{\sigma}{2}\,(\mu(\bar E_{\eps a_i\,}) +\mu(\bar E'_{\eps a'_i\,}))) \geq (1+\sigma/2)\, \mu(g^{-1}(E_\eps\,))|_{\bar{E}_{a_i \epsilon} \cup \bar{E}'_{a'_i \epsilon}}   ,
$$

Putting both of these scenarios together and summing over $i$, we get

$$\mu(g^{-1}(E_{\kappa\,\eps\,}))>(1+\sigma/2)\, \mu(g^{-1}(E_\eps\,))$$ given that the ${\bar{E}_{a_i \kappa \epsilon} \cup \bar{E}'_{a'_i \kappa \epsilon}}$ are disjoint for every $i$. Thus, for each pair of slabs, the doubling condition holds with $\frac{\sigma}{2}$ instead of $\sigma$ so it holds for $E_\epsilon$ completing the proof.
\end{proof}

\section{Statements and Declarations}

I hereby declare that I have no competing interests or financial interests.

	\clearpage
\section{Appendix A}
\label{Appendix_A}

	\subsection{Proof of Lemma \ref{Lemmaactivationfunction2}. }
	
	Suppose $\mu(P_1) \neq \emptyset$. If either $(1+\sigma)\,\mu(\eps\,\bar S)\geq \delta$ or $(1+\sigma)\,\mu(\eps\,\bar S')\geq\delta$, we are done.

If both $ (1+\sigma)\,\mu(\eps\,\bar S)<\delta$ and $ (1+\sigma)\,\mu(\eps\,\bar S')<\delta$. Then
\[\label{eq2}
\begin{split}
  \mu(g^{-1}(\kappa\,\epsilon\,S))&=\mu(S^1_{\kappa\,\epsilon})+\mu(S^0_{\kappa\,\epsilon}) +\mu(S^2_{\kappa\,\epsilon})\\&\geq \frac{\mu(S^1_{\kappa\,\epsilon})+\mu(S^1_{\epsilon})}{2}+\mu(S^0_{\kappa\,\epsilon}) +\frac{\mu(S^2_{\kappa\,\epsilon}) +\mu(S^2_{\epsilon}) }{2}\\
  &\geq \mu(S^1_{\epsilon})+\mu(S^0_{\epsilon}) +\mu(S^2_{\epsilon})+\frac{\mu(S^1_{\kappa\,\epsilon})+2\mu(S^0_{\kappa\,\epsilon}) +\mu(S^2_{\kappa\,\epsilon})- \mu(S^1_{\epsilon})-2\,\mu(S^0_{\epsilon}) -\mu(S^2_{\epsilon})}{2}.  
\end{split}\]

Here we are using an averaging trick to deal with the complication that we only have one $S_\epsilon^0$ and $S_{\kappa \epsilon}^0$.  
Applying now \eqref{measE+} and \eqref{measE-}, we find that
\begin{equation}
    \mu(g^{-1}(\kappa\,\eps\,S))\geq \mu(g^{-1}(\eps\,S)) +\frac{\sigma}{2}\,(\mu(\eps\,\bar S) +\mu(\eps\,\bar S'))  \geq (1+\sigma/2)\, \mu(g^{-1}(\eps\,S)) 
\end{equation}

proving the doubling condition with $\sigma'=\sigma/2$. Meaning, for  $\epsilon=\ell_S \kappa^i< \beta$ we have

\begin{equation}
   \mu(S_{\kappa\epsilon})> \min\{\ \delta, (1+\frac{\sigma}{2})\mu (S_{\epsilon}) \}\
\end{equation}

The critical point of $\lambda_1$ at $x_1=0$ would then be the only complicated part of $\lambda_1$.

%



\subsection{Proof of Lemma \ref{Upper_Bound_on_Loss}: Upper bound on the loss}

\begin{proof}

\end{proof}

We know the DNN $\phi$ is of the form

\begin{equation}\label{eq:classifier_with_softmax}
\phi(s,\alpha)=\rho\circ X(s,\alpha)=\left(\frac{e^{X_1(s,\alpha)}}{\sum_{k=1}^{K}e^{X_k(s,\alpha)}},\cdots,\frac{e^{X_K}}{\sum_{k=1}^{K}e^{X_k(s,\alpha)}}\right).
\end{equation} Thus, \eqref{eq:classifier_with_softmax} gives
\begin{equation}\label{eq:transforming_softmax}
p_{i(s)}(s,\alpha)=\frac{e^{X_{i(s)}(s,\alpha)}}{\sum_{k=1}^K e^{X_k(s,\alpha)}}=\frac{1}{\sum_{k=1}^K e^{X_k(s,\alpha)-X_{i(s)}(s,\alpha)}}=\frac{1}{1+\sum_{k\neq i(s)} e^{X_k(s,\alpha)-X_{i(s)}(s,\alpha)}}.
\end{equation}
For each $k\neq i(s)$, $X_k(s,\alpha)-X_{i(s)}(s,\alpha)\leq\max_{k\neq i(s)} X_k(s,\alpha)- X_{i(s)}(s,\alpha)=-\delta X(s,\alpha)$. Using \eqref{eq:defdeltaX} and \eqref{eq:transforming_softmax}, we obtain the following estimates on $p_{i(s)}(s,\alpha)$:
\begin{equation}\label{eq:probability_estimates}
\frac{1}{1+(K-1)e^{-\delta X(s,\alpha)}}\leq p_{i(s)}(s,\alpha)\leq\frac{1}{1+e^{-\delta X(s,\alpha)}}.
\end{equation} 
Finally, \eqref{eq:probability_estimates} gives estimates on loss:    
\begin{equation}\label{eq:loss_estimates}
\sum_{s\in T}\mu(s)\log\left(1+e^{-\delta X(s,\alpha)}\right)\leq\bar L(\alpha)\leq\sum_{s\in T}\mu(s)\log\left(1+(K-1)e^{-\delta X(s,\alpha)}\right).
\end{equation}

From
(\ref{eq:loss_estimates})  we get
\begin{equation}
\sum_{s\in G_{\eta^*}^c(t)} \mu(s)\,\log (1+e^{-\eta^*})\leq \bar L(t)\leq \bar L(t_0).
\label{chebyshev}
\end{equation}

    	     \section*{Acknowledgements}
       I am very grateful to my advisor Dr. Berlyand for the formulation of the problem and many useful discussions and suggestions. I am also grateful to Dr. Jabin, Dr. Sodin, Dr. Golovaty, Alex Safsten, Oleksii Krupchytskyi, and Jon Jenkins  
 for their useful discussions and suggestions that led to the improvement of the manuscript. Your help and time are truly appreciated. Finally, I would like to express my sincere gratitude to the
referees who provided detailed feedback and valuable insights, which not only
helped to improve the quality of this work but also deepened my understanding
of the topic.

	\bibliographystyle{unsrt}

 \begin{figure}[h!]
 \caption{Classification confidence for the toy DNN for the Example \ref{example_1} at different stages of training.}
\label{CC_toy_example1}
 
 \begin{subfigure}{0.35\textwidth}
\includegraphics[height=5.5cm,width=7.5cm]
{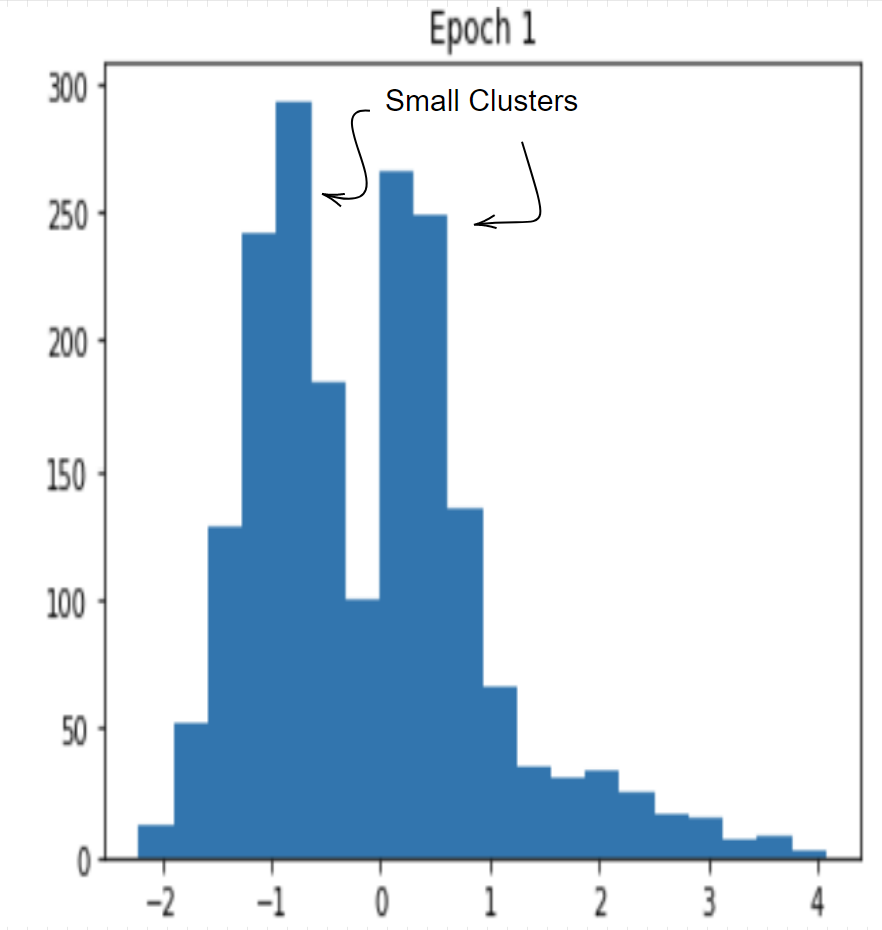} 
\end{subfigure}
\hfill
\begin{subfigure}{0.35\textwidth}
\includegraphics[width=9cm]{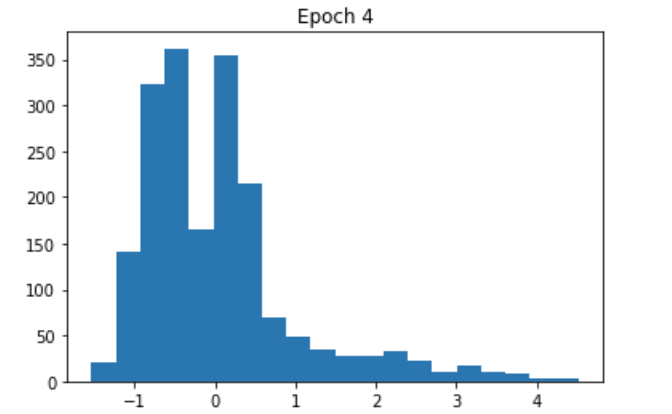} 
\end{subfigure}
\hfill
\begin{subfigure}{0.35\textwidth}
\includegraphics[width=9cm]{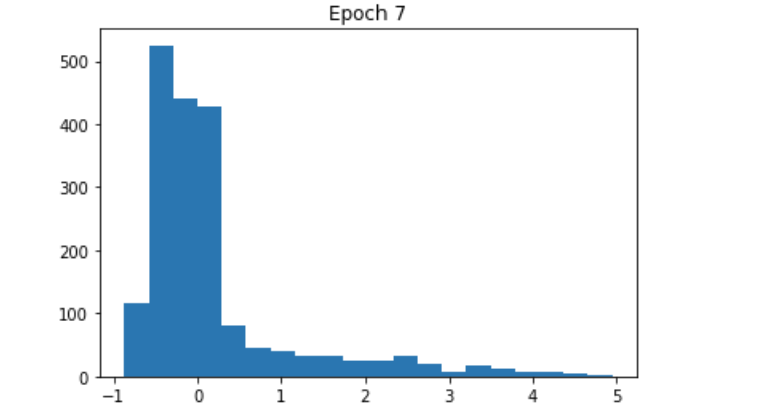} 
\end{subfigure}
\hfill
\begin{subfigure}{0.35\textwidth}
\includegraphics[width=9cm]{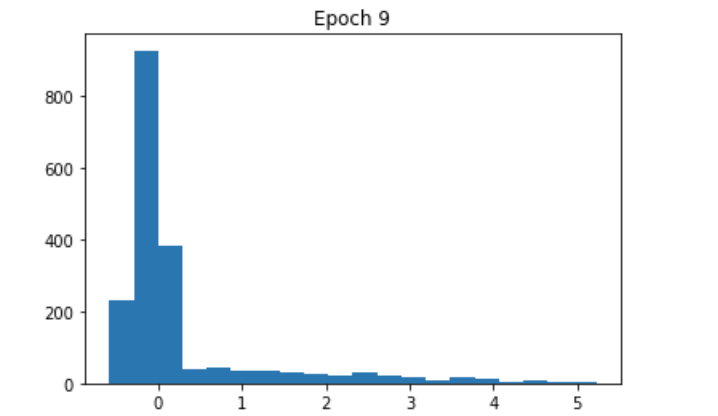}
\end{subfigure}
\hfill
\begin{subfigure}{0.35\textwidth}
\includegraphics[width=9cm]{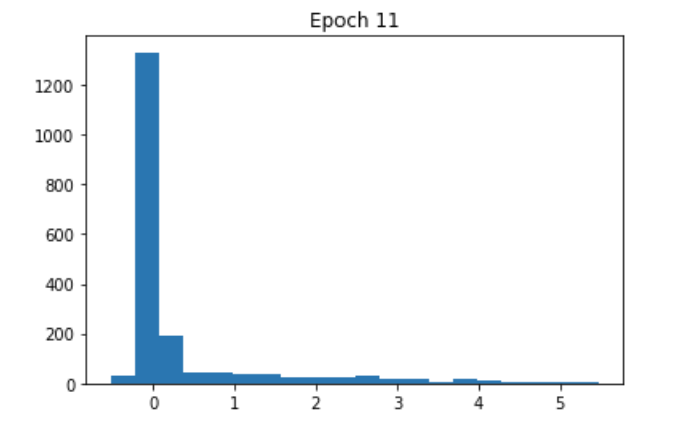} 
\end{subfigure}
\hfill
\begin{subfigure}{0.35\textwidth}
\includegraphics[width=9cm]{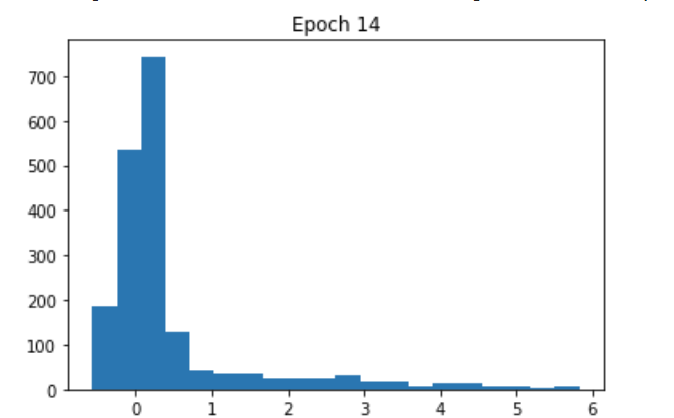} 
\end{subfigure}
\hfill
\begin{subfigure}{0.35\textwidth}
\includegraphics[width=9cm]{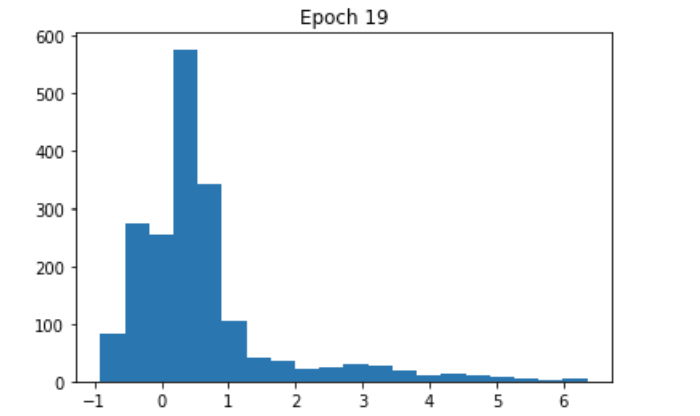} 
\end{subfigure}
\hfill
\begin{subfigure}{0.35\textwidth}
\includegraphics[width=9cm]{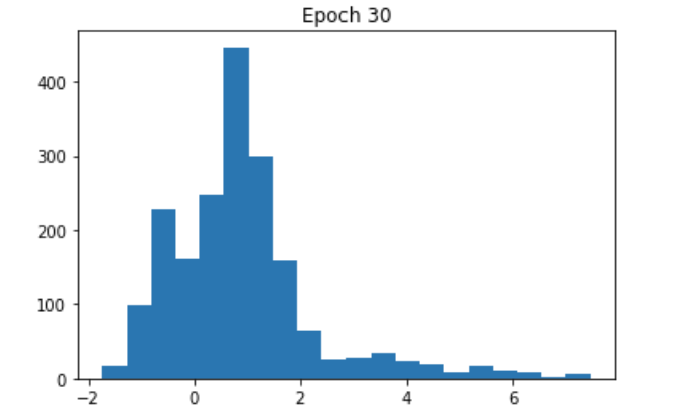} 
\end{subfigure}
\end{figure}

\bibliography{DNNstability_biblio}

\begin{thebibliography}{10}

\bibitem{LBD}
Yann LeCun, Bernhard Boser, John Denker, Donnie Henderson, Richard Howard, Wayne Hubbard, and Lawrence Jackel.
\newblock Handwritten digit recognition with a back-propagation network.
\newblock {\em Advances in Neural Information Processing Systems}, 2, 1989.

\bibitem{krizhevsky2017imagenet}
Alex Krizhevsky, Ilya Sutskever, and Geoffrey~E Hinton.
\newblock Imagenet classification with deep convolutional neural networks.
\newblock {\em Communications of the ACM}, 60(6):84--90, 2017.

\bibitem{hinton2012deep}
Geoffrey Hinton, Li~Deng, Dong Yu, George~E Dahl, Abdel-rahman Mohamed, Navdeep Jaitly, Andrew Senior, Vincent Vanhoucke, Patrick Nguyen, Tara~N Sainath, et~al.
\newblock Deep neural networks for acoustic modeling in speech recognition: The shared views of four research groups.
\newblock {\em IEEE Signal Processing Magazine}, 29(6):82--97, 2012.

\bibitem{sutskever2014sequence}
Ilya Sutskever, Oriol Vinyals, and Quoc~V Le.
\newblock Sequence to sequence learning with neural networks.
\newblock {\em Advances in Neural Information Processing Systems}, 27, 2014.

\bibitem{LJS}
Leonid Berlyand, Pierre-Emmanuel Jabin, and Alex Safsten.
\newblock Stability for the training of deep neural networks and other classifiers.
\newblock {\em Mathematical Models and Methods in Applied Sciences}, 31(11):2345--2390, 2021.

\bibitem{goodfellow2016deep}
Ian Goodfellow, Yoshua Bengio, and Aaron Courville.
\newblock {\em Deep learning}.
\newblock MIT Press, 2016.

\bibitem{soudry2018implicit}
Daniel Soudry, Elad Hoffer, Mor~Shpigel Nacson, Suriya Gunasekar, and Nathan Srebro.
\newblock The implicit bias of gradient descent on separable data.
\newblock {\em The Journal of Machine Learning Research}, 19(1):2822--2878, 2018.

\bibitem{shalev2007pegasos}
Shai Shalev-Shwartz, Yoram Singer, and Nathan Srebro.
\newblock Pegasos: Primal estimated sub-gradient solver for svm.
\newblock In {\em Proceedings of the 24th International Conference on Machine Learning}, pages 807--814, 2007.

\bibitem{zhang2021understanding}
Chiyuan Zhang, Samy Bengio, Moritz Hardt, Benjamin Recht, and Oriol Vinyals.
\newblock Understanding deep learning (still) requires rethinking generalization.
\newblock {\em Communications of the ACM}, 64(3):107--115, 2021.

\bibitem{ma2018power}
Siyuan Ma, Raef Bassily, and Mikhail Belkin.
\newblock The power of interpolation: Understanding the effectiveness of sgd in modern over-parametrized learning.
\newblock In {\em International Conference on Machine Learning}, pages 3325--3334. PMLR, 2018.

\bibitem{kawaguchi2017generalization}
Kenji Kawaguchi, Leslie~Pack Kaelbling, and Yoshua Bengio.
\newblock Generalization in deep learning.
\newblock {\em arXiv preprint arXiv:1710.05468}, 2017.

\bibitem{cohen2021learning}
Omry Cohen, Or~Malka, and Zohar Ringel.
\newblock Learning curves for overparametrized deep neural networks: A field theory perspective.
\newblock {\em Physical Review Research}, 3(2):023034, 2021.

\bibitem{xu2019trained}
Yuhui Xu, Yuxi Li, Shuai Zhang, Wei Wen, Botao Wang, Wenrui Dai, Yingyong Qi, Yiran Chen, Weiyao Lin, and Hongkai Xiong.
\newblock Trained rank pruning for efficient deep neural networks.
\newblock In {\em 2019 Fifth Workshop on Energy Efficient Machine Learning and Cognitive Computing-NeurIPS Edition (EMC2-NIPS)}, pages 14--17. IEEE, 2019.

\bibitem{yang2020learning}
Huanrui Yang, Minxue Tang, Wei Wen, Feng Yan, Daniel Hu, Ang Li, Hai Li, and Yiran Chen.
\newblock Learning low-rank deep neural networks via singular vector orthogonality regularization and singular value sparsification.
\newblock In {\em Proceedings of the IEEE/CVF Conference on Computer Vision and Pattern Recognition Workshops}, pages 678--679, 2020.

\bibitem{xue2013restructuring}
Jian Xue, Jinyu Li, and Yifan Gong.
\newblock Restructuring of deep neural network acoustic models with singular value decomposition.
\newblock In {\em Interspeech}, pages 2365--2369, 2013.

\bibitem{cai2014fast}
Chenghao Cai, Dengfeng Ke, Yanyan Xu, and Kaile Su.
\newblock Fast learning of deep neural networks via singular value decomposition.
\newblock In {\em Pacific Rim International Conference on Artificial Intelligence}, pages 820--826. Springer, 2014.

\bibitem{anhao2016svd}
Xing Anhao, Zhang Pengyuan, Pan Jielin, and Yan Yonghong.
\newblock Svd-based dnn pruning and retraining.
\newblock {\em Journal of Tsinghua University (Science and Technology)}, 56(7):772--776, 2016.

\bibitem{berlyand2023enhancing}
Leonid Berlyand, Etienne Sandier, Yitzchak Shmalo, and Lei Zhang.
\newblock Enhancing accuracy in deep learning using random matrix theory.
\newblock {\em arXiv preprint arXiv:2310.03165}, 2023.

\bibitem{shmalo2023deep}
Yitzchak Shmalo, Jonathan Jenkins, and Oleksii Krupchytskyi.
\newblock Deep learning weight pruning with rmt-svd: Increasing accuracy and reducing overfitting.
\newblock {\em arXiv preprint arXiv:2303.08986}, 2023.

\bibitem{staats2023boundary}
Max Staats, Matthias Thamm, and Bernd Rosenow.
\newblock Boundary between noise and information applied to filtering neural network weight matrices.
\newblock {\em Phys. Rev. E}, 108:L022302, Aug 2023.

\bibitem{donoho2013optimal}
Matan Gavish and David~L. Donoho.
\newblock The optimal hard threshold for singular values is $4/\sqrt {3}$.
\newblock {\em IEEE Transactions on Information Theory}, 60(8):5040--5053, 2014.

\bibitem{goodfellow2014explaining}
Ian~J. Goodfellow, Jonathon Shlens, and Christian Szegedy.
\newblock Explaining and harnessing adversarial examples.
\newblock {\em CoRR}, abs/1412.6572, 2014.

\bibitem{szegedy2013intriguing}
Christian Szegedy, Wojciech Zaremba, Ilya Sutskever, Joan Bruna, Dumitru Erhan, Ian Goodfellow, and Rob Fergus.
\newblock Intriguing properties of neural networks.
\newblock {\em arXiv preprint arXiv:1312.6199}, 2013.

\bibitem{zheng2016improving}
Stephan Zheng, Yang Song, Thomas Leung, and Ian Goodfellow.
\newblock Improving the robustness of deep neural networks via stability training.
\newblock In {\em Proceedings of the Ieee Conference on Computer Vision and Pattern Recognition}, pages 4480--4488, 2016.

\bibitem{thulasidasan2019mixup}
Sunil Thulasidasan, Gopinath Chennupati, Jeff~A Bilmes, Tanmoy Bhattacharya, and Sarah Michalak.
\newblock On mixup training: Improved calibration and predictive uncertainty for deep neural networks.
\newblock {\em Advances in Neural Information Processing Systems}, 32, 2019.

\end{thebibliography}

\end{document}